\documentclass{article}

    \PassOptionsToPackage{numbers, compress}{natbib}

 \usepackage[main, final]{neurips_2026}

\usepackage[utf8]{inputenc} 
\usepackage[T1]{fontenc}    
\usepackage{hyperref}       
\usepackage{url}            
\usepackage{booktabs}       
\usepackage{amsfonts}       
\usepackage{nicefrac}       
\usepackage{microtype}      
\usepackage{xcolor}         
\usepackage{wrapfig}
\usepackage{amssymb}
\usepackage{amsmath}
\usepackage[dvipsnames,svgnames]{xcolor}
\definecolor{darkgreen}{RGB}{0,100,0}
\usepackage{natbib}
\usepackage{amsmath,amsthm}
\usepackage{graphicx}
\usepackage{multicol}
\usepackage{algorithm}
\usepackage{algorithmic}
\usepackage{booktabs}
\usepackage{adjustbox}
\usepackage{graphicx}  
\usepackage{float}  
\usepackage{subfigure}  
\usepackage{algorithm}
\usepackage{algorithmic}
\usepackage{multirow}
\usepackage{enumerate}
\usepackage{amsthm}

\newtheorem{theorem}{Theorem}
\newtheorem{definition}{Definition}
\newtheorem{proposition}{Proposition}

\usepackage{svg}
\usepackage{wrapfig} 
\usepackage[utf8]{inputenc}
\usepackage{tikz}
\usetikzlibrary{arrows.meta, positioning, calc, shapes.geometric}

\usepackage{amsmath}
\usepackage{graphicx}

\title{Value-Rectified Distillation for Flow-based Offline Reinforcement Learning}

\author{%
  Ke Jiang$^{1,2,3}$, 
  Wen Jiang$^{1,2}$, 
  Yoshinobu Kawahara$^{3,4}$, 
  Xiaoyang Tan$^{1,2}$ \\
  $^1$College of Computer Science and Technology, Nanjing University of Aeronautics and Astronautics \\
  $^2$MIIT Key Laboratory of Pattern Analysis and Machine Intelligence \\
  $^3$Graduate School of Information Science and Technology, The University of Osaka \\
  $^4$Center for Advanced Intelligence Project, RIKEN \\
  \texttt{\{ke\_jiang, darren.jum\}@nuaa.edu.cn}, \\
  \texttt{kawahara@ist.osaka-u.ac.jp}, x.tan@nuaa.edu.cn
}

\begin{document}

\maketitle

\begin{abstract}

  We aim to address the problem of \textit{Out-of-Distribution (OOD) Mode Averaging}, a critical problem that emerges when distilling Flow Matching (FM) based behavior policies in offline Reinforcement Learning (RL): When learning over multi-modal datasets, standard distillation inadvertently forces the policy to interpolate across different behavioral modes, causing the generated actions to collapse into invalid, OOD regions. To resolve this, we propose \textbf{Value-Rectified Distillation (VRD)}, a novel framework that reformulates offline policy distillation as a value-guided generative trajectory alignment problem. By employing a dynamic coupling mechanism, VRD actively rectifies flow dynamics to explicitly decouple distinct behavioral modes, cleanly isolating high-value actions. Specifically, we instantiate VRD through two complementary algorithms: \textbf{VRD-Q}, which utilizes value-quantile behavior distillation to statistically filter out low-value behaviors, and \textbf{VRD-R}, which reorganizes the latent space to geometrically isolate high-value generative trajectories from the FM-based behavior policy. Extensive empirical evaluations demonstrate that VRD improves existing generative baselines across diverse continuous control and planning tasks on the D4RL, OGBench, and Minari benchmarks, especially over those multi-modal, non-expert datasets.

\end{abstract}

\section{Introduction}\label{sec:intro}

Offline Reinforcement Learning (RL) aims to derive a high-capacity policy from a fixed, pre-collected dataset without further environment interaction~\citep{kumar2020conservative,kumar2019stabilizing}, which is particularly vital for the applications where online exploration is costly or hazardous. Recently, generative models, particularly Flow Matching (FM)~\citep{lipman2022flow}, have emerged as a powerful alternative to traditional Diffusion-based methods~\citep{hansen2023idql,chen2023score,ding2023consistency}, where FM-based policies, such as Flow Q-Learning (FQL)~\citep{park2025flow}, offer significant advantages in modeling complex, multi-modal behavior distributions while circumventing the training instabilities typically associated with backpropagating through ODE solvers.

\begin{figure}[h]
    \centering
    \includegraphics[width=1\linewidth]{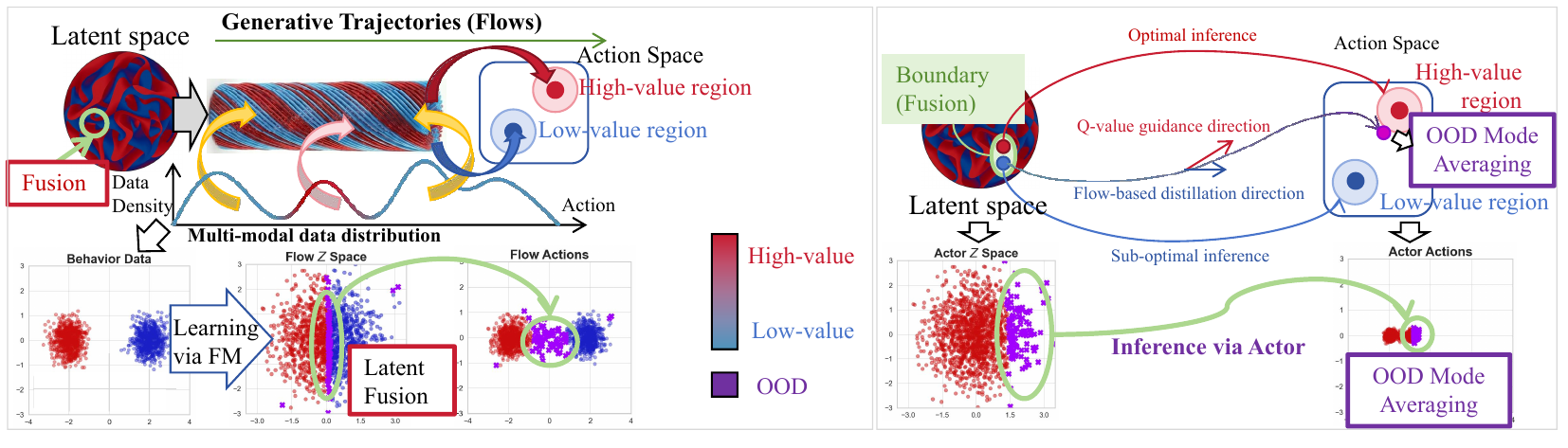}
    \caption{Problem of \textit{Latent Fusion} (left) and the OOD Mode Averaging problem caused by FQL-style behavior distillation (right).}
    \label{fig:problem}
\end{figure}

Despite their empirical promise, we identify a fundamental pathology in how current flow-based RL frameworks distill high-value actions from multi-modal behavior policies. Specifically, when learning from multi-modal data, our analysis reveals that Flow Matching networks inherently suffer from \textit{Latent Fusion}: the topological entanglement between the generative trajectories of optimal and sub-optimal behavioral modes, caused by the vector field's training dynamics, as is shown in Figure \ref{fig:problem} (left). In previous works, like FQL~\citep{park2025flow}, it first pre-trains a behavior flow policy to capture the full multi-modal behavior distribution~\citep{chi2025diffusion} from the dataset, and subsequently applies Q-value guidance to distill the optimal policy. However, because such distillation operates on a continuous vector field where trajectories are already topologically entangled by the MLE pre-training, the Q-guided objective, as is shown in Figure \ref{fig:problem} (right), inadvertently forces the sampled actions to linearly interpolate across distinct modes during inference, which triggers the \textit{OOD Mode Averaging} disaster, where it is risky for the generated actions to collapse into low-density regions of the state-action manifold, ultimately causing performance degradation. Although recent approaches like Guided Flow Policy (GFP)~\citep{tiofack2025guided} incorporate reward information to amplify optimal behavior modes in FM-based behavior policy, their FQL-style distillation mechanism still suffers from \textit{OOD Mode Averaging}.

To address this issue, our core insight is that during the distillation phase, policy refinement must go beyond simply cloning from FM-based behavior policy; rather, it inherently requires rectifying the flow dynamics to explicitly \textbf{decouple distinct behavioral modes} along the data manifold. Therefore, we introduce \textbf{Value-Rectified Distillation (VRD)}, a novel framework that reformulates offline policy distillation as a value-guided trajectory alignment problem, of which the core lies in a dynamic coupling mechanism that precisely modulates the flow dynamics to explicitly decouple and exclude sub-optimal behavioral modes, ensuring that the generated actions are confined to the optimal action manifold during inference.

To instantiate this unified framework, we derive two distinct algorithms from complementary mathematical perspectives: \textbf{VRD-Q (Statistical Rectification)}: this variant performs in-place vector field alignment by employing asymmetric weight modulation based on value quantiles, which statistically filters sub-optimal regions without altering the underlying flow topology. \textbf{VRD-R (Geometric Rectification)}: this variant actively reshapes the flow's topology, which constructs a value-gravity field that transports the generative trajectories toward a highly compact optimal support set, geometrically isolating them from sub-optimal modes.




In summary, our primary contributions include: \textbf{1) Tracing OOD Mode Averaging:} We theoretically trace the root cause of the \textit{OOD Mode Averaging} problem in FQL-style distillation mechanism, revealing how its distillation objective interacts with \textit{latent fusion} to generate invalid OOD actions on multi-modal data. \textbf{2) The VRD Framework:} We propose Value-Rectified Distillation (VRD) with two implementations VRD-Q and VRD-R, reformulating offline policy distillation as a value-guided trajectory alignment problem to safely extract high-value policies without mode averaging. \textbf{3) Empirical Validation:} Extensive evaluations demonstrate that VRD-Q and VRD-R effectively mitigate this OOD issue, outperforming or performing competitively with existing generative baselines across various continuous control and planning tasks.

\section{Preliminaries and motivation}\label{sec:preliminaries}

\subsection{Offline RL and Flow Q-Learning}

Suppose a Markov Decision Process $\mathcal{M} = (\mathcal{S}, \mathcal{A}, r, \rho, p)$, then the goal of offline RL is to learn a policy $\pi$ that maximizes the expected cumulative discounted return using a static dataset $\mathcal{D}$ collected by a behavioral policy. Previous work adopted the FQL framework~\citep{park2025flow} to avoid the instability of backpropagating through ODE solvers. FQL consists of three components: 

\textbf{Critic:} A value function $Q_\phi(s, a)$ trained to minimize the Bellman error. 

\textbf{BC Flow Policy:} A flow-matching policy $\mu_\theta(s, z)$ defined by a vector field $v_\theta$, trained purely on Behavioral Cloning (BC) to capture the dataset distribution:
\begin{align}
    \mathcal{L}_{\text{Flow}}(\theta) = \mathbb{E}_{s, a \sim \mathcal{D}, t \sim \mathcal{U}(0,1), x^0 \sim \mathcal{N}(0, I_d)} [ \|v_\theta(t, s, x^t) - (a - x^0)\|_2^2 ] \label{eq:flow_loss}
\end{align}
where $x^t = (1-t)x^0 + ta$.
    
\textbf{One-Step Actor Policy:} An expressive feed-forward network $\pi(s, z): \mathcal{S} \times \mathbb{R}^d \to \mathcal{A}$ that maps latent noise $z$ directly to actions.

In standard FQL, the actor $\pi$ is trained to maximize $Q$-values while being regularized via distillation from the frozen flow policy $\mu_\theta$. The standard objective uses an unstructured Gaussian prior $z \sim \mathcal{N}(0, I_d)$ for this coupling:
\begin{align}
    \mathcal{L}_{\pi} = \mathbb{E}_{s \sim \mathcal{D}, z \sim \mathcal{N}(0, I_d)} [ -\lambda \cdot Q_\phi(s, \pi(s, z)) + \|\pi(s, z) - \mu_\theta(s, z)\|_2^2 ]\label{eq:fql_actor}
\end{align}
, while the Q-networks $Q_\phi$ are learned by minimizing a standard double Temporal Difference (TD) objective:
\begin{align}
    \mathcal{L}(\phi_i) = \mathbb{E}_{(s,a,r,s') \sim \mathcal{D}} \left[ \left( Q_{\phi_i}(s,a) - r - \gamma \min_{j=1,2} Q_{\bar{\phi}, j}(s', \pi(s',z)) \right)^2 \right], \ z\sim\mathcal{N}(0,I)
\end{align}

This implies that the latent space of the actor is isotropic; the position of $z$ carries no semantic information regarding the quality of the action.

\subsection{Latent fusion and OOD mode-averaging problems of FQL in multi-modal data}

According to Flow Q-Learning, we can fit the action distribution of the behavior policy $\pi_\beta(a|s)$, by the flow model $\mu_\theta(s,z)$, where $z\sim\mathcal{N}(0,I)$. Then we define the support action set given $s$, as $\mathcal{A}(s)=\{a|\pi_\beta(a|s)>0\}$, and OOD action set $\bar{\mathcal{A}}=\{a|\pi_\beta(a|s)=0\}$. Furthermore, given a Q-function $Q(s,a)$ as the evaluator, we can define the optimal and suboptimal action sets as,
\begin{align}
    \mathcal{A}^*_Q(s)&=\{a|a\in\mathcal{A}(s), Q(s,a)\geq\max_{a'}Q(s,a') - \epsilon_{opt}\}\label{eq:optimalaction}\\
    \mathcal{A}^-_Q(s)&=\mathcal{A}(s)/\mathcal{A}^*_Q(s)
\end{align}

In our work, we focus on the setting of \textbf{multi-modal data}, which has raised attention in the field~\cite{chen2024diffusion, hong2023beyond} recently. Specifically, in our work, we focus on the \textbf{isolated} multi-modal data as,
\begin{align}
    \inf_{a^*\in\mathcal{A}^*_Q(s),a^-\in\mathcal{A}^-_Q(s)}\|a^* - a^-\|_2 = D>0\label{eq:dist_D}
\end{align}
which means that there exists a gap between the optimal and sub-optimal data modes. Under such kind of data, the first problem we encounter in flow-based offline RL is the {\it  Latent Fusion}  for the flow policy $\mu_\theta$ learned by $\mathcal{L}_{\text{Flow}}(\theta)$ in Eq.(\ref{eq:flow_loss}):



\begin{proposition}[Latent Entanglement of OOD Generation]
\label{prop:latentfusion}
Suppose the latent space $\mathcal{Z}$ (support of $z \sim \mathcal{N}(0, I)$) is connected, and the flow mapping $\mu_\theta(s, z)$ is $L$-Lipschitz continuous w.r.t. $z$. Let $\mathcal{Z}^*_Q(s)$, $\mathcal{Z}^-_Q(s)$, and $\mathcal{Z}_{OOD}(s)$ denote the latent preimages of $\mathcal{A}^*_Q(s)$, $\mathcal{A}^-_Q(s)$, and $\bar{\mathcal{A}}(s)$, respectively. If the model expresses both optimal and sub-optimal modes ($\mathcal{Z}^*_Q(s) \neq \emptyset, \mathcal{Z}^-_Q(s) \neq \emptyset$), the following holds:
\textbf{(1)} The optimal latent boundary is non-empty ($\partial \mathcal{Z}^*_Q(s) \neq \emptyset$);
\textbf{(2)} The OOD latent set is strictly non-empty ($\mathcal{Z}_{OOD}(s) \neq \emptyset$);
\textbf{(3)} Any $\epsilon$-neighborhood of a boundary point $z_{bd} \in \partial \mathcal{Z}^*_Q(s)$ intersects non-optimal regions: $\forall \epsilon \!>\! 0, \mathbb{B}_\epsilon(z_{bd}) \cap \big(\mathcal{Z}^-_Q(s) \cup \mathcal{Z}_{OOD}(s)\big) \neq \emptyset$.
\end{proposition}

 \textit{Proof is shown in Appendix} \ref{appendix:proofofpropositionlatentfusion}. Proposition \ref{prop:latentfusion} formalizes the \textit{Latent Fusion} dilemma. Because $\mu_\theta$ is continuous and $\mathcal{Z}$ is connected, the model cannot safely isolate optimal latents. Property (2) proves that to bridge the physical gap $D>0$, any continuous curve connecting the modes must structurally intersect the OOD space (a topological necessity, though without quantifying probability mass). Furthermore, Property (3) shows $\mathcal{Z}^*_Q(s)$ is inextricably entangled with sub-optimal/OOD regions at its boundary.

While Proposition \ref{prop:latentfusion} leverages the potent multi-modal modeling capabilities of Flow matching for targeted policy extraction, the extreme topological entanglement within the latent space, as characterized by the {\it  Latent Fusion} in Proposition \ref{prop:latentfusion}, renders such distillation and mapping via continuous, smooth networks susceptible to mode-averaging, as shown in Theorem \ref{thm:ood_avg}. This phenomenon forces the policy to generate interpolated, OOD actions ($\overline{a} \notin \mathcal{A}(s)$) that collapse into off-manifold regions.

\begin{theorem}[OOD Mode Averaging in FQL]
\label{thm:ood_avg} Suppose the Q-function satisfies $\mu$-strong concavity in the region bridging the modes and is $L$-Lipschitz smooth.
Consider the FQL-style distillation that minimizes the regularized objective in Eq.(\ref{eq:fql_actor}). 
Under multi-modal data with gap $D > 0$ as Eq.(\ref{eq:dist_D}) and any non-zero value regularization penalty $\lambda>0$,
FQL generates OOD actions with a minimum expected error bounded by $D^2/4$.
\end{theorem}

\textit{Proof is shown in Appendix} \ref{appendix:proofofthmoodavg}. Theorem \ref{thm:ood_avg} reveals a fundamental limitation of FQL: under the latent mode entanglement established in Proposition \ref{prop:latentfusion}, the explicit Euclidean regularization (e.g., $-\lambda Q + \|a - \mu_\theta\|_2^2$) blindly forces policy updates to perform linear interpolation across the action space. Because multi-modal data manifolds are topologically separated by a physical gap $D > 0$, this shortest Euclidean path traverses low-density OOD valleys, which triggers the Mode Averaging disaster with a strict expected error bounded by $D^2/4$. Consequently, we conclude that action-space Euclidean constraints are structurally incompatible with multi-modal offline RL. Besides, assuming local $\mu$-strong concavity in the valley between modes serves as an analytical abstraction to capture the optimization tension in FQL.

In other words, in probability density estimation, a smoothly connected set of feasible actions (like slightly different stable walking postures) topologically constitutes a single, broad, and continuous mode. Flow Matching naturally and gracefully captures this continuous variance. However, we suggest that arises specifically when the behavioral distribution contains structurally distinct peaks separated by low-density valleys (OOD regions).

\textbf{Motivation.} The observation that standard Euclidean distillation is vulnerable to mode averaging on multi-modal datasets motivates our Value-Rectified Distillation (VRD) framework. As analyzed in this section, directly penalizing the action difference can inadvertently pull the actor into OOD regions. To effectively mitigate this issue, policy extraction must move beyond simple action-space regularization. Instead of linearly interpolating actions across empty spaces, VRD rectifies the continuous latent mapping by dynamically aligning the flow dynamics with value guidance, hence encouraging the policy to safely extract high-value actions along the data manifold.

\section{Methodology}

\subsection{Unified framework: kernel-coupled value rectification}\label{sec:unified}

In Flow Q-Learning, while the initial behavior cloning phase effectively captures the multi-modal data distribution via the flow model $\mu_\theta$, it absorbs suboptimal behavioral modes (as discussed in Section \ref{sec:preliminaries}). To distill $\mu_\theta$ into an optimal policy $\pi$ that exclusively outputs actions within the optimal set $\mathcal{A}^*_Q$, a principled rectification mechanism is essential. Rather than treating this process as isolated heuristic regularizations, we cast the value rectification as a latent distribution alignment problem.

To this end, we propose a generalized Value-Rectified Functional Space. We define the rectification objective $\Psi(\pi; \mu_\theta)$ as an expected kernel-coupled divergence over the joint latent space $\mathcal{Z} \times \mathcal{Z}$:
\begin{equation}
\label{eq:vrd_functional}
\Psi(\pi; \mu_\theta) = \int_{\mathcal{Z}} \mathcal{P}_{\mathcal{Z}}(z) \left[ \int_{\mathcal{Z}} \kappa_Q(z, z' \mid s, \mu_\theta, \pi) \cdot \|\mu_\theta(s, z) - \pi(s, z')\|_2^2 \, dz' \right] dz
\end{equation}
where $\mathcal{P}_{\mathcal{Z}}(z)$ is the prior distribution of the pre-trained latent space, and $\kappa_Q(z, z' \mid \cdot): \mathcal{Z} \times \mathcal{Z} \rightarrow \mathbb{R}^+$ is defined as the Value-Guided Alignment Kernel (VGAK). Specifically, in our work, we focus on a \textit{Non-merging transport kernel} $\kappa_Q(z, z' \mid \cdot) = w(z)\delta(z' - T(z))$ for a measurable weight $w \ge 0$ and a measurable map $T: \mathcal{Z} \rightarrow \mathcal{Z}$ satisfying mode separation $T(\mathcal{Z}_\mu^*) \cap T(\mathcal{Z}_\mu^-) = \emptyset$ and fiber-wise optimal preservation (i.e., for any $z' \in T(\mathcal{Z}_\mu^*)$, the induced conditional barycenter over the pre-image fiber $T^{-1}(\{z'\})$ remains within $\mathcal{A}_Q^*(s)$). The core insight of this unified functional is that the transition from a behavior-cloned flow to an optimal policy is essentially a latent-to-latent coupling process. The kernel $\kappa_Q$ mathematically governs both the coupling strength (how much weight a sample carries) and the coordinate mapping (where the latent variable is transported) between the source latent variable $z$ and the target latent variable $z'$.

Then, by plugging Eq.~(\ref{eq:vrd_functional}) into Eq.~(\ref{eq:fql_actor}), the rectification-based policy optimization objective is:
\begin{equation}
\label{eq:vrd_objective}
\mathcal{L}_\pi = \mathbb{E}_{s \sim \mathcal{D}} \left[ \Psi(\pi; \mu_\theta \mid s) - \lambda \cdot \mathbb{E}_{z' \sim \mathcal{P}_{test}}[Q(s, \pi(s, z'))] \right]
\end{equation}
where $\lambda$ is the balance coefficient, and $\mathcal{P}_{test}$ is a prior distribution of the latent variable $z'$ used at test.

To establish a rigorous theoretical foundation for how this functional guarantees optimal policy extraction, we first define the partitioning of the original latent space.
\begin{definition}[Optimal Latent Subset]
\label{def:optimal_lat_sub}
For a given state $s$ and a pre-trained flow model $\mu_\theta$, the subset of latent variables that map to the optimal action set $\mathcal{A}_Q^*(s)$ in Eq.~(\ref{eq:optimalaction}) is defined as $\mathcal{Z}_\mu^*(s) = \{z \in \mathcal{Z} \mid \mu_\theta(s, z) \in \mathcal{A}_Q^*(s)\}$, and the corresponding suboptimal latent subset is $\mathcal{Z}_\mu^-(s) = \mathcal{Z} \setminus \mathcal{Z}_\mu^*(s)$.
\end{definition}

\begin{theorem}[Optimal Support]\label{thm:optimalsupport}
Suppose the critic $Q$ is conservative on out-of-distribution regions such that $\sup_{a \in \overline{\mathcal{A}}(s)} Q(s, a) < \max_{a' \in \mathcal{A}(s)} Q(s, a') - \epsilon_{opt}$. Let $\hat{\pi}^* = \arg\min_\pi [\Psi(\pi; \mu_\theta) - \lambda \cdot \mathbb{E}_{z' \sim \mathcal{P}_{test}}[Q(s, \pi(s, z'))]]$, as in Eq.~(\ref{eq:vrd_objective}), be the learned optimal policy assuming universal approximation. Over the test support $\mathcal{Z}_{test} = \mathrm{supp}(\mathcal{P}_{test})$, if the non-merging alignment kernel $\kappa_Q$ satisfies:
(i) \textbf{$\epsilon$-Bounded Suboptimal Leakage:} $\int_{\mathcal{Z}_{test}}\int_{\mathcal{Z}_\mu^-} \mathcal{P}_{\mathcal{Z}}(z)\kappa_Q(z, z' \mid \cdot) \, dz dz' \le \epsilon$, and 
(ii) \textbf{$C$-Dominant Optimal Coupling:} $\int_{\mathcal{Z}_{test}}\int_{\mathcal{Z}_\mu^*} \mathcal{P}_{\mathcal{Z}}(z)\kappa_Q(z, z' \mid \cdot) \, dz dz' \ge C > 0$, 
then under the kernel-induced marginal distribution over $\mathcal{Z}_{\text{test}}$, the policy guarantees to output optimal actions with probability $\mathbb{P}(\hat{\pi}^*(s, z') \in \mathcal{A}_Q^*(s)) \ge \frac{C}{C+\epsilon}$.
\end{theorem}

\textit{Proof see Appendix \ref{appendix:proofofoptimalsupport}.} Theorem \ref{thm:optimalsupport} resolves this issue by entirely discarding the Euclidean penalty in the action space. Instead, it reformulates the policy distillation process as a non-linear latent distribution alignment problem, introducing a Value-Guided Alignment Kernel, denoted as $\kappa_Q$. As long as $\kappa_Q$ satisfies two fundamental conditions: (i) $\epsilon$-Bounded Suboptimal Leakage, and (ii) $C$-Dominant Optimal Coupling, Theorem \ref{thm:optimalsupport} guarantees a positive lower bound ($\ge \frac{C}{C+\epsilon}$) on the probability that the distilled policy outputs actions within the optimal support set $\mathcal{A}_Q^*(s)$. Mathematically, this non-linear latent alignment allows the generative trajectories to conform to the underlying data manifold, safely bypassing OOD regions and eliminating the mode averaging.

In the following subsections (Section \ref{sec:quantile} and \ref{sec:ranking}), we demonstrate how the single framework, VRD, derives two practical rectification methodologies, as shown in Figure \ref{fig:framework}, one via quantile modulation and the other via latent reshaping, by simply instantiating different forms of Dirac kernels $\kappa_Q$.

\begin{figure}[t]
    \centering
    \includegraphics[width=1\linewidth]{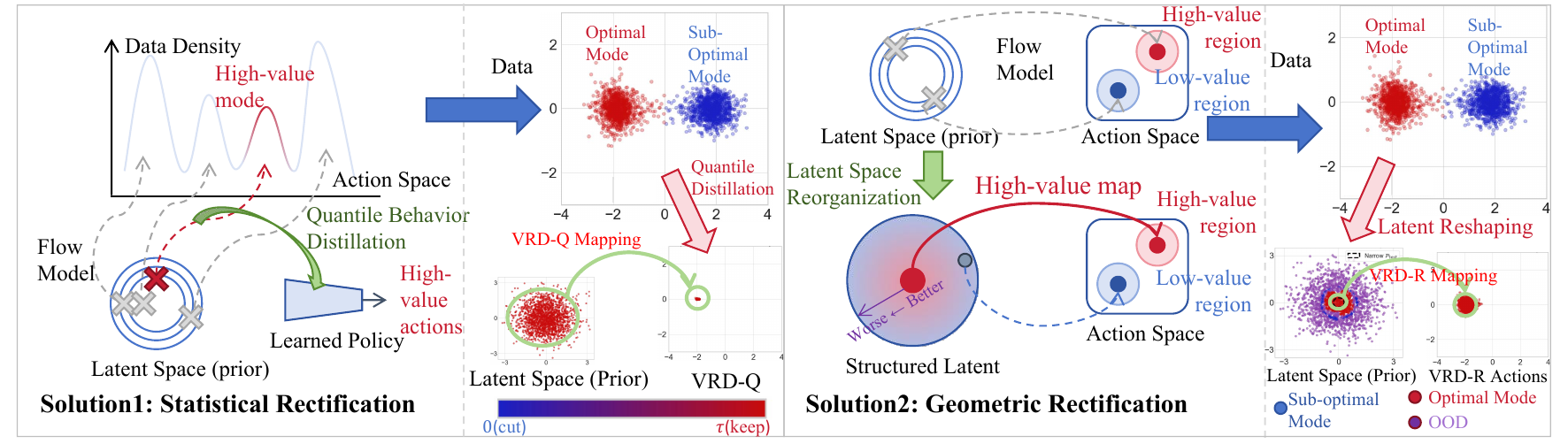}
    \caption{Implementations for \textbf{Value-Rectified Distillation (VRD)}: Statistical rectification via quantile behavior distillation (left), Geometric rectification via reshaping behavior distillation (right).}
    \label{fig:framework}
\end{figure}

\subsection{Statistical rectification: quantile behavior distillation}\label{sec:quantile}

In our first instantiation of the generalized framework, we restrict the target latent space from undergoing topological deformation. From an Eulerian perspective, we perform an \textbf{in-place alignment} where the source latent $z$ must strictly align with the target latent $z'$. 

Mathematically, this implies that the Value-Guided Alignment Kernel (VGAK) degenerates into a diagonal Dirac matrix modulated by a value-aware weight. We formulate this specific kernel as:
\begin{align}
\kappa_{Q}(z, z^{\prime} \mid s, \mu_{\theta}, \pi)
= \underbrace{\Big|\tau - \mathbb{I}\big(Q(s, \mu_{\theta}(s, z)) < Q(s, \pi(s, z)) - \epsilon_{\text{opt}}\big)\Big|}_{\text{Quantile-aware Weight } \omega_{\tau}(z)}
\cdot \underbrace{\delta(z^{\prime} - z)}_{\text{In-place Alignment}}
\label{eq:kernel_quantile}
\end{align}
where $\tau \in (0.5, 1)$ is a quantile hyperparameter, $\mathbb{I}(\cdot)$ is the indicator function, and $\delta(\cdot)$ is the Dirac delta function. Note that the weight $\omega_\tau(z)$ is strictly a function of the source latent $z$, ensuring stable gradient propagation during variational optimization. In practice, $\epsilon_{\text{opt}}$ is often set to a very small value, e.g., $10^{-6}$.

By substituting this diagonal kernel into the unified functional $\Psi(\pi; \mu_\theta)$ (defined in Section \ref{sec:unified}), the inner integral over $z'$ collapses due to the sifting property of the Dirac $\delta$ function. The functional seamlessly derives into our first practical objective, the \textbf{Quantile-Rectified Loss}:
\begin{align}
\Psi_{Q}(\pi; \mu_\theta) = \mathbb{E}_{s \sim \mathcal{D}, z \sim \mathcal{Z}} \left[ \Big|\tau - \mathbb{I}\Big(Q\big(s,\mu_\theta(s,z)\big) < Q\big(s,\pi(s,z)\big) - \epsilon_{\text{opt}}\Big)\Big|\cdot\Big\|\mu_\theta(s,z) - \pi(s,z)\Big\|_2^2 \right] \label{eq:loss_quantile}
\end{align}

\begin{proposition}[Connecting VRD-Q with Theorem \ref{thm:optimalsupport}]\label{prop:vrdq}
Let $P^{*}$ and $P^{-}$ be the prior masses of $\mathcal{Z}_{\mu}^{*}$ and $\mathcal{Z}_{\mu}^{-}$, respectively. 
For given constants $\epsilon > 0$ and $C > 0$, if the quantile threshold $\tau \in (0.5, 1)$ satisfies 
$\tau P^{*} \ge C$ and $(1-\tau)P^{-} + (2\tau-1)P^{-}_{\text{band}}(\tau) \le \epsilon$ 
(where $P^{-}_{\text{band}}(\tau)$ denotes the prior mass of the boundary band near $\mathcal{A}_{Q}^{*}$), then: (a) The VRD-Q kernel satisfies the coupling conditions of Theorem~\ref{thm:optimalsupport}, guaranteeing 
$\mathbb{P}_{\kappa_{Q}}(\hat{\pi}^{*}(s,z) \in \mathcal{A}_{Q}^{*}(s)) \ge \frac{C}{C+\epsilon}$ 
under the kernel-induced marginal distribution; (b) Under the standard prior $z \sim \mathcal{N}(0, I_{d})$ at inference, 
$\mathbb{P}_{z}(\hat{\pi}^{*}(s,z) \in \mathcal{A}_{Q}^{*}(s)) \ge P^{*} + P^{-}_{\text{rect}}(\tau)$, 
where $P^{-}_{\text{rect}}(\tau)$ is the rectified prior mass.
\end{proposition}

\textit{Proof see Appendix~\ref{appendix:proofofvrdq}.} Proposition~\ref{prop:vrdq} has two parts. Part (a) analyzes the kernel-induced marginal distribution: by the margin in Eq.~(\ref{eq:kernel_quantile}), optimal latents receive weight $\tau$ independent of $\pi$, while suboptimal ones outside a band near $\mathcal{A}_{Q}^{*}(s)$ receive $1-\tau$. Hence, Theorem~\ref{thm:optimalsupport}'s coupling conditions hold if $\tau P^{*} \ge C$ and $(1-\tau)P^{-} + (2\tau-1)P^{-}_{\text{band}}(\tau) \le \epsilon$. Part (b) covers inference under the standard prior without latent reweighting: optimal latents still map to $\mathcal{A}_{Q}^{*}(s)$, and a weakened behavioral anchor rectifies some suboptimal latents into $\mathcal{A}_{Q}^{*}(s)$, yielding the lower bound $P^{*} + P^{-}_{\text{rect}}(\tau)$. Notably, this bound is distinct and does not follow from (a)'s $\frac{C}{C+\epsilon}$ bound.

Practically, tuning $\tau$ trades off suboptimal filtering (larger $\tau$ heavily attenuates suboptimal weights) against critic-error robustness (moderate $1-\tau$ anchors the policy against noisy $Q$-gradients); see Appendix~\ref{appendix:ablation_study} for empirical $\tau$ analysis.

\subsection{Geometric rectification: reshaped behavior distillation}\label{sec:ranking}

In our second instantiation, instead of modulating the fitting weights, we assign equal importance to all samples but permit topological deformation of the target latent space. This achieves rectification through \textbf{spatial transportation}, where the source latent $z$ is dynamically routed to a new coordinate $z'$ based on its evaluated Q-value.

Mathematically, the Value-Guided Alignment Kernel (VGAK) degenerates into a non-diagonal, deterministic transport map. We formulate this transport kernel as a shifted Dirac delta function:
\begin{equation}
\label{eq:vrdr_kernel}
\kappa_Q(z, z' \mid s, \mu_\theta, \pi) = 1 \cdot \delta(z' - z_Q(s, z, \mu_\theta))
\end{equation}
where $z_Q(s, z, \mu_\theta)$ is a value-guided reshaping function. To construct an optimal cluster that is highly accessible during inference, we design $z_Q$ as a gravitational field towards high-value actions:
\begin{equation}
\label{eq:vrdr_reshaping}
z_Q(s, z, \mu_\theta) = \Phi_Q(s, z) \cdot \frac{z}{\|z\|_2}, \quad \text{where} \quad \Phi_Q(s, z) = \frac{Q_{\max}(s) - Q(s, \mu_\theta(s, z))}{Q_{\max}(s) - Q_{\min}(s) + \epsilon_m}
\end{equation}
where $Q_{\max}(s) = \max_{a \in \mathcal{A}(s)} Q(s, a)$ and $Q_{\min}(s) = \min_{a \in \mathcal{A}(s)} Q(s, a)$ denote the state-wise value boundaries (which can be practically approximated by the empirical extremes or the MDP reward bounds $\frac{R_{\max}}{1-\gamma}$ and $\frac{R_{\min}}{1-\gamma}$). Here, $\Phi_Q(s, z) \in [0, 1]$ acts as a normalized ranking scalar. By substituting this transport kernel into the unified functional $\Psi(\pi; \mu_\theta)$ in Eq.~(\ref{eq:vrd_functional}), the integral domain collapses precisely at $z' = z_Q(s, z, \mu_\theta)$. This derivation seamlessly yields our second practical objective, the Reshaping-Rectified Loss:
\begin{equation}
\label{eq:vrdr_loss}
\Psi_R(\pi; \mu_\theta) = \mathbb{E}_{s \sim \mathcal{D}, z \sim \mathcal{Z}} \left[ \|\mu_\theta(s, z) - \pi(s, z_Q(s, z, \mu_\theta))\|_2^2 \right]
\end{equation}

\begin{proposition}[Connecting VRD-R with Theorem~\ref{thm:optimalsupport}]
\label{prop:vrdr}
Suppose the critic $Q$ satisfies the conservatism condition in Theorem \ref{thm:optimalsupport}, and the optimal action set $\mathcal{A}_{Q}^{*}(s)$ forms a convex dominant mode over which $Q(s, \cdot)$ is concave. During inference, let $z^{\prime} \sim \mathcal{N}(0, \sigma^{2} I)$ be truncated to $\mathcal{Z}_{\text{test}} = \{z^{\prime} \in \mathcal{Z} \mid \|z^{\prime}\|_{2} \le \sigma\}$. If the truncation radius satisfies $\sigma < \epsilon_{\text{opt}} / (Q_{\max}(s) - Q_{\min}(s) + \epsilon_{m})$, where $\epsilon_{\text{opt}}$ is the optimality relaxation factor in Eq.~(\ref{eq:optimalaction}), then $\kappa_{Q}$ in Eq.~(\ref{eq:vrdr_kernel}) fulfills all conditions of Theorem \ref{thm:optimalsupport} with zero suboptimal leakage ($I^{-} = 0$), guaranteeing $\mathbb{P}(\hat{\pi}^{*}(s, z^{\prime}) \in \mathcal{A}_{Q}^{*}(s)) = 1$ under the kernel-induced marginal distribution.
\end{proposition}

\textit{Proof see Appendix \ref{appendix:proofofvrdr}.} Unlike VRD-Q, whose threshold $\tau$ must be carefully balanced against probability masses to satisfy the $\epsilon$ and $C$ bounds (Proposition \ref{prop:vrdq}), VRD-R translates the distribution alignment problem into a strict geometric isolation. By setting the inference truncation radius $\sigma < \epsilon_{\text{opt}} / (Q_{\max}(s) - Q_{\min}(s) + \epsilon_m)$, under the assumption of concave dominant optimal mode, VRD-R achieves the ideal extreme of Theorem \ref{thm:optimalsupport} where the suboptimal leakage is approximate zero ($\epsilon \approx 0$). Consequently, the topological boundary $\epsilon_{\text{opt}} / (Q_{\max}(s) - Q_{\min}(s) + \epsilon_m)$ is determined exclusively by the Q-value landscape.

Proposition \ref{prop:vrdr} bridges the geometric mapping of VRD-R with the requirements of Theorem \ref{thm:optimalsupport}, providing the precise analytical threshold for the truncation radius $\sigma$ necessary to eliminate suboptimal leakage. By satisfying these conditions, it ensures that the inference process is not merely a heuristic reshaping but a value-guided optimal sampler that yields actions within $\mathcal{A}_Q^*(s)$ with high probability.

\textbf{In summary,} both VRD-Q and VRD-R prevent catastrophic out-of-distribution (OOD) mode-averaging by using dynamic Q-guidance to distill optimal actions along the pre-trained behavior manifold. However, they differ in mechanism and trade-offs: \textbf{VRD-Q (In-place Dynamical Intervention):} Truncates ODE gradient flows toward sub-optimal modes in-place via quantile modulation. It preserves the spatial geometry and multi-modal diversity of expert behaviors, which makes VRD-Q ideal for complex, multi-modal tasks requiring diverse path preservation and trajectory stitching (e.g., OGBench navigation). \textbf{VRD-R (Geometric Reshaping):} Rescales high-value latent vectors into a compact region around the origin. It acts as a symmetry-breaker to isolate a single, highly stable peak mode, which enable VRD-R for continuous control tasks where single-mode precision and stability are prioritized (e.g., standard D4RL locomotion). 
\textit{Algorithm procedures of \textbf{VRD-Q} and \textbf{VRD-R} are illustrated in Appendix \ref{appendix:algs}. Hyperparameter analysis for both components see Appendix \ref{appendix:ablation_study}.}



\section{Related works}

\paragraph{Gaussian Policies.} Traditional offline RL methods often rely on Gaussian distributions to parameterize the policy. This category includes standard Behavior Cloning (\textbf{BC}~\citep{kumar2021should}) as well as prominent value-based algorithms like Implicit Q-Learning (\textbf{IQL}~\citep{kostrikov2021offline}) and \textbf{ReBRAC}~\citep{tarasov2023revisiting}. While Gaussian policies are computationally efficient, in complex offline datasets where behavior is often multi-modal or skewed, Gaussian-based actors suffer from limited expressivity~\citep{park2025flow}, often leading to mode-averaging and suboptimal performance when the data distribution is non-Gaussian.

\paragraph{Diffusion Policies.} To overcome the expressivity bottleneck, recent works such as \textbf{IDQL}~\citep{hansen2023idql}, \textbf{SRPO}~\citep{chen2023score}, and \textbf{CAC}~\citep{ding2023consistency} leverage the power of score-based diffusion modeling to capture highly complex and multi-modal action distributions. However,
directly optimizing an iterative diffusion process to maximize a Q-function remains challenging, often requiring unstable backpropagation through the sampling steps or complex reward-guided sampling techniques~\citep{janner2022planning}.

\paragraph{Flow Policies.} Flow matching~\citep{lipman2022flow} has emerged as a compelling alternative to diffusion, which is simpler to train and faster to infer. To address the difficulty of directly optimizing iterative flow models, Flow Q-Learning (\textbf{FQL})~\citep{park2025flow} proposed distilling an expressive one-step policy from a flow-matching model guided by RL objectives, along with two variants: \textbf{FAWAC}~\citep{ashvin2020accelerating} and \textbf{FBRAC}~\citep{wu2019behavior}, which bridges the gap between expressivity and efficiency.  More recently, Guided Flow Policy (\textbf{GFP})~\citep{tiofack2025guided} and Decision Flow (\textbf{DF})~\citep{hu2025decision} develops tailored optimization techniques, focusing on specifically extracting high-value actions within the flow framework.


\section{Experiments}
\label{sec:experiments}

We aim to systematically assess the efficacy of the proposed VRD\footnote{Our code is available at: \url{https://github.com/Jack10843/VRD-master}} by answering the core research questions (RQs): \textbf{RQ1 (Toy Example):} Can VRD-Q statistically filter sub-optimal modes and VRD-R geometrically isolate optimal trajectories?
\textbf{RQ2 (Benchmarks):} Does the overall VRD framework outperform existing baselines across standard continuous control (D4RL) and complex long-horizon planning (OGBench) tasks?
\textbf{RQ3 (Multi-Modal):} How effectively does VRD mitigate OOD Mode Averaging when explicitly evaluated on different scales of multi-modal datasets?
\textbf{RQ4 (Hyperparameters):} How robust is the VRD to key hyperparameter variations?

\subsection{Visualizing latent rectification on toy examples (RQ1)}
\label{subsec:toy_example}

To empirically evaluate how our methods manipulate the latent space to extract optimal policies, we construct a 2D bi-modal toy environment. The dataset comprises actions sampled from an optimal mode at $[-2.0, 0.0]$ and a sub-optimal mode at $[2.0, 0.0]$, with isotropic Gaussian noise, where the corresponding value landscape is simulated using a mock Q-function strictly centered at $[-2.0, 0.0]$.
\begin{figure}[h]
    \centering
    \includegraphics[width=1\linewidth]{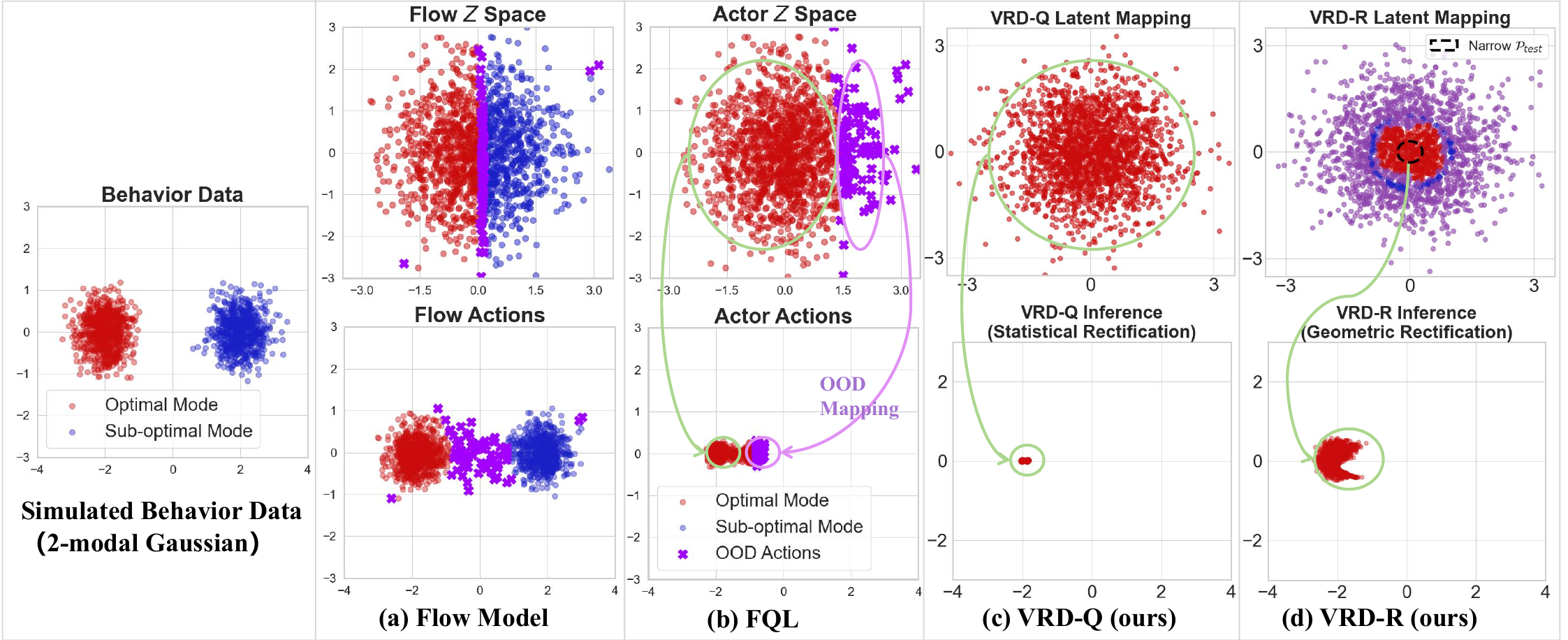}
    \caption{Visualizing results on toy examples with 2D multi-modal data.}
    \label{fig:vis_toyexp}
\end{figure}
Although the baseline flow model perfectly replicates the offline dataset by indiscriminately absorbing both modes (Figure \ref{fig:vis_toyexp}a), it suffers from \textit{latent fusion} at the boundary between the optimal (red) and sub-optimal (blue) latent variables ($Z$), where the entanglement induces OOD action interpolations, hindering clear optimal mode distillation and exacerbating the risk of OOD mode averaging in FQL-style behavior distillation (Figure \ref{fig:vis_toyexp}b).
In contrast, the proposed VRD successfully extract a strictly optimal policy (Figure \ref{fig:vis_toyexp}c, d). Specifically, VRD-Q performs statistical filtering by dynamically truncating distillation weights in low-value latent regions, which severs the gradient flow for sub-optimal actions, forcing the standard Gaussian prior to map exclusively to the optimal mode. Conversely, VRD-R geometrically reshapes the latent topology via value-guided spatial transportation. It attracts optimal latent variables toward the origin while expelling sub-optimal ones to the periphery. By sampling from a low-temperature Gaussian prior during zero-shot inference, VRD-R isolates optimal trajectories, verifying the efficacy of both rectification strategies.


\subsection{Continuous control (D4RL) and complex planning (OGBench) tasks (RQ2)}
\label{subsec:d4rl_results}

\textbf{D4RL}~\citep{fu2020d4rl} features \textbf{MuJoCo} locomotion (HalfCheetah, Hopper, Walker2d) across 'medium' (collected by a partially trained SAC~\citep{haarnoja2018soft}), 'medium-replay' (training history of SAC), and 'medium-expert' (mixed of expert data and 'medium') levels. It also includes \textbf{AntMaze} navigation across umaze, medium, and large layouts, using 'play' (undirected) and 'diverse' (goal-reaching) data to test trajectory stitching; and the 24-DOF \textbf{Adroit} Pen task with 'human' (demos), 'expert' (RL-trained), and 'cloned' (BC-mixed) datasets to evaluate high-dimensional control under narrow distributions.

\begin{table}[ht]
\centering
\caption{Results on D4RL MuJoCo, AntMaze and Adroit Benchmarks. The highest results of each row are bolded. Each result of our method is calculated over 8 seeds. Part of the results were reported in~\citep{park2025flow}.}
\label{tab:antmaze_results_extra_col}
\resizebox{1\linewidth}{!}{%
\begin{tabular}{l|rrr|rrr|rrrrr|rr} 
\toprule
& \multicolumn{3}{c}{\textbf{Gaussian Policies}} & \multicolumn{3}{c}{\textbf{Diffusion Policies}} & \multicolumn{7}{c}{\textbf{Flow Policies}} \\ 
\cmidrule(lr){2-4} \cmidrule(lr){5-7} \cmidrule(lr){8-14} 
\textbf{Task} & BC & IQL & ReBRAC & IDQL & SRPO & CAC & FAWAC & FBRAC & FQL & GFP& DF & \textbf{VRD-Q} & \textbf{VRD-R} \\ 
\midrule
halfcheetah-med-expert & 55.2 & 86.7 & \textbf{101.1} & 95.9 & 92.2 & 84.3 & 44.5 & 54.4 & 86.1 & 92.3 & 87.6 & 93.7 $\pm$ 1.2 & \textbf{101.3} $\pm$ 1.3\\
hopper-med-expert & 52.5 & 91.5 & 107.0 & 108.6 & 100.1 & 100.4 & 88.4 & 76.5 & 105.4 & 110.1 & 112.6 & 96.4 $\pm$ 13.5 & \textbf{113.5} $\pm$ 0.4\\
walker2d-med-expert & 107.5 & 109.6 & 111.6 & 112.7 & 114.0 & 110.4 & 57.3 & 53.5 & 102.7 & 106.4 & 111.3 & 112.1 $\pm$ 2.3 & \textbf{113.7} $\pm$ 0.2\\
halfcheetah-medium & 42.6 & 47.4 & \textbf{65.6} & 51.0 & 60.4 & 69.1 & 48.3 & 53.1 & 60.0 & 62.3 & 49.3 & \textbf{65.7} $\pm$ 1.2 & 61.7 $\pm$ 4.5\\
hopper-medium & 52.9 & 66.3 & \textbf{102.0} & 65.4 & 95.5 & 80.7 & 81.2 & 52.1 & 67.3 & 84.7 & 97.1 & 65.8 $\pm$ 15.8 & 76.3 $\pm$ 12.1 \\
walker2d-medium & 75.3 & 78.3 & 82.5 & 82.5 & 84.4 & 83.1 & 42.1 & 81.4 & 77.2 & 83.1 & 87.2 & 82.3 $\pm$ 4.9 & \textbf{89.7} $\pm$ 6.4\\
halfcheetah-med-replay & 36.6 & 44.2 & 51.0 & 45.9 & 51.4 & \textbf{58.7} & 42.6 & 49.6 & 53.1 & 54.7 & 44.2 & 54.7 $\pm$ 2.7 & 56.6 $\pm$ 2.1\\
hopper-med-replay & 18.1 & 94.7 & 98.1 & 92.1 & 101.2 & 99.7 & 73.4 & 50.8 & 78.4 & 82.1 & \textbf{102.3} & 71.7$\pm$ 15.1 & 96.3 $\pm$ 4.7\\
walker2d-med-replay & 26.0 & 73.9 & 77.3 & 85.1 & 84.6 & 79.5 & 37.5 & 32.2 & 72.9 & 80.1 & 86.6 & 84.4 $\pm$ 7.1 & \textbf{87.6} $\pm$ 5.3\\
\midrule
\textbf{Average (MuJoCo)} & 51.9 & 77.0 & \textbf{88.5} & 82.1 & 87.1 & 85.1 & 57.3 & 56.0  & 78.1 & 84.0 & 86.5 & 80.8 & \textbf{88.5} \\
\midrule
\midrule
antmaze-umaze & 55.0 & 77.0 & 98.0 & 94.0 & 97.0 & 66.0 & 90.0 & 94.0 & 96.0 & 96.8 & 93.7 & 96.7 $\pm$ 0.3 & \textbf{99.3} $\pm$ 0.7\\
antmaze-umaze-diverse & 47.0 & 62.2 & 84.0 & 80.0 & 82.0 & 66.0 & 55.0 & 62.0 & 89.0 & 91.9 & 83.2 & 83.2 $\pm$ 3.1 & \textbf{94.2} $\pm$ 2.1\\
antmaze-medium-play & 0.0 & 71.2 & \textbf{90.0} & 84.0 & 81.0 & 49.0 & 52.0 & 77.0 & 78.0 & 81.9 & 43.8 & 76.7 $\pm$ 9.5 & 87.5 $\pm$ 4.1\\
antmaze-medium-diverse & 1.0 & 70.0 & 84.0 & 85.0 & 75.0 & 0.0 & 44.0 & 60.0 & 71.0 & 61.6 & 40.1 & \textbf{87.1} $\pm$ 3.1 & 83.2 $\pm$ 0.7\\
antmaze-large-play & 0.0 & 39.6 & 52.0 & 64.0 & 54.0 & 0.0 & 10.0 & 32.0 & 84.0 & 82.6 & 47.3 & \textbf{84.4} $\pm$ 1.8 & 78.1 $\pm$ 8.3\\
antmaze-large-diverse & 0.0 & 47.5 & 64.0 & 68.0 & 54.0 & 0.0 & 16.0 & 20.0 & {83.0} & 84.1 & 52.1 & \textbf{85.7} $\pm$ 2.4 & {83.7} $\pm$ 5.7\\
\midrule
\textbf{Average (AntMaze)} & 17.2 & 61.3 & 78.7 & 79.2 & 73.8 & 30.2 & 44.5 & 57.5 & 83.5 & 83.2 & 60.0 & 85.6 & \textbf{87.7} \\
\midrule
\midrule
pen-human & 71 & 78 & 103 & 76 & 69 & 64 & 67 & 77 & 53 & 64.6 & 102.8 & 110.4$\pm$ 2.7 & \textbf{111.3} $\pm$ 4.3 \\
pen-cloned & 52 & 83 & 103 & 64 & 61 & 56 & 62 & 67 & 74 & 77.1 & 104.3 & \textbf{108.1} $\pm$ 3.6 & 106.7 $\pm$ 6.7 \\
pen-expert & 110 & 128 & \textbf{152} & 140 & 134 & 103 & 118 & 119 & 142 & 140.4 & 139.1 & 138.4 $\pm$ 7.3  & 129.4 $\pm$ 12.9 \\
\midrule
\textbf{Average (Adroit)} & 77.7 & 96.3 & \textbf{119.3} & 93.3 & 88.0 & 74.3 & 82.3 & 87.7 & 89.7 & 94.0 & 115.4 & 119.0  & 115.8 \\ \midrule\midrule
\textbf{Average (ALL)} & 44.6 & 75.0 & 90.3 & 83.0 & 82.8 & 65.0 & 57.2 & 61.8 & 81.8 & 85.4 & 82.5 & 88.8 & \textbf{92.8} \\ 
\bottomrule
\end{tabular}
}
\end{table}

As shown in Table \ref{tab:antmaze_results_extra_col}, VRD-R outperforms the competitive baselines over most benchmarks, achieving the highest average score of \textbf{92.8}, while VRD-Q achieves \textbf{88.8}. 
On MuJoCo's hopper-med-expert (multi-modal), VRD-R scores \textbf{113.5} compared to FQL's 105.4, suggesting geometric reshaping effectively isolates expert modes to mitigate mode averaging. For sparse-reward navigation in antmaze-large-diverse, VRD-Q reaches \textbf{85.7} where diffusion policy methods like CAC fail (0.0), implying statistical filtering successfully truncates low-value actions for robust trajectory stitching. Finally, VRD-R achieves \textbf{111.3} on Adroit's pen-human (vs. GFP's 64.6). Overall, value-rectification allows generative policies to stably extract high-value manifolds even from mixed data distributions.

\textbf{OGBench:} To further evaluate our method's scalability to long-horizon and high-dimensional goal-conditioned tasks, we benchmark VRD-Q on 50 distinct tasks from \textbf{OGBench}~\citep{park2024ogbench}. These environments span exceptionally challenging navigation scenarios, including large and giant \textbf{AntMazes}, medium and large \textbf{HumanoidMazes}, and \textbf{AntSoccer}, as well as complex manipulation and spatial reasoning tasks like \textbf{Cube} (single/double), \textbf{Scene}, and sliding \textbf{Puzzles} (3x3/4x4).

\begin{table}[htbp]
\centering
\caption{Results on OGBench Tasks. The highest results of each row are bolded. Each result of our method is calculated over 8 seeds. The full results are shown in Table \ref{tab:ogbench_full} in Appendix, where part of the results were reported in~\citep{park2025flow}.}
\label{tab:ogbench_summary}
\resizebox{\textwidth}{!}{
\begin{tabular}{l| ccc| ccc| ccccc|cc}
\toprule
 & \multicolumn{3}{c}{\textbf{Gaussian Policies}} & \multicolumn{3}{c}{\textbf{Diffusion Policies}} & \multicolumn{7}{c}{\textbf{Flow Policies}} \\
\cmidrule(lr){2-4} \cmidrule(lr){5-7} \cmidrule(lr){8-14}
\textbf{Task} & BC & IQL & ReBRAC & IDQL & SRPO & CAC & FAWAC & FBRAC & IFQL & FQL & GFP & \textbf{VRD-Q} & \textbf{VRD-R} \\
\midrule
{antmaze-large} & 11.0 & 53.0 & 81.0 & 21.0 & 11.0 & 33.0 & 6.0 & 60.0 & 28.0 & 79.0 & 93.8 & \textbf{95.5} $\pm$ 1.2 & \textbf{94.9} $\pm$ 1.9\\
{antmaze-giant} & 0.0 & 4.0 & 26.0 & 0.0 & 0.0 & 0.0 & 0.0 & 4.0 & 3.0 & 9.0 & 27.9 & \textbf{49.8}$\pm$ 5.3 & 38.3 $\pm$ 7.2\\
{humanoidmaze-medium} & 2.0 & 33.0 & 22.0 & 1.0 & 1.0 & 53.0 & 19.0 & 38.0 & 60.0 & 58.0 & 72.0 & \textbf{75.3}$\pm$ 2.9 & 70.3$\pm$ 7.2\\
{humanoidmaze-large} & 1.0 & 2.0 & 2.0 & 1.0 & 0.0 & 0.0 & 0.0 & 2.0 & 11.0 & 4.0 & 17.8 & \textbf{19.7} $ \pm$ 2.9 & 16.8 $\pm$5.9 \\
{antsoccer-arena} & 1.0 & 8.0 & 0.0 & 12.0 & 1.0 & 2.0 & 12.0 & 16.0 & 33.0 & 60.0 & 57.9 & \textbf{63.8}$\pm$ 4.7 & 47.9 $\pm$ 9.7\\
{cube-single} & 5.0 & 83.0 & 91.0 & 95.0 & 80.0 & 85.0 & 81.0 & 79.0 & 79.0 & 96.0 & \textbf{98.8} & \textbf{99.5} $\pm$ 0.4 & 97.7 $\pm$ 0.8\\
{cube-double} & 2.0 & 7.0 & 12.0 & 15.0 & 2.0 & 6.0 & 5.0 & 15.0 & 14.0 & 29.0 & \textbf{47.2} & 45.6 $\pm$ 6.4 & 32.2 $\pm$ 8.4\\
{scene} & 5.0 & 28.0 & 41.0 & 46.0 & 20.0 & 40.0 & 30.0 & 45.0 & 30.0 & \textbf{56.0} & 53.5 & \textbf{55.4} $\pm$ 3.2 & \textbf{55.8} $\pm$ 2.5\\
{puzzle-3x3} & 2.0 & 9.0 & 21.0 & 10.0 & 18.0 & 19.0 & 6.0 & 14.0 & 19.0 & 30.0 & 23.1 & 28.6 $\pm$ 3.8 & \textbf{34.7} $\pm$ 11.6\\
{puzzle-4x4} & 0.0 & 7.0 & 14.0 & 29.0 & 10.0 & 15.0 & 1.0 & 13.0 & 25.0 & 17.0 & 26.1 & \textbf{30.6} $\pm$ 3.5 & 19.8 $\pm$ 6.3 \\
\midrule
\textbf{Average (50 tasks)} & 2.8 & 23.4 & 31.0 & 22.8 & 14.3 & 25.3 & 16.1 & 28.6 & 30.0 & 43.5 & 51.8 & \textbf{56.4} & 50.8 \\
\bottomrule
\end{tabular}
}
\end{table}

As shown in Table \ref{tab:ogbench_summary}, VRD-Q achieves the highest overall average score of \textbf{56.4} across all 50 tasks, demonstrating superior robustness over existing baselines. Notably, in antmaze-giant, VRD-Q (49.8) and VRD-R (38.3) outperform the best Flow baseline, GFP (27.9), while Diffusion-based policies, like CAC, fail (0.0). VRD-Q also performs well over high-dimensional locomotion (humanoidmaze, antsoccer-arena), whereas VRD-R (50.8 average) excels in spatial reasoning tasks like puzzle-3x3 (34.7) and scene (55.8). Ultimately, both VRD-Q's statistical filtering and VRD-R's geometric reshaping effectively mitigate compounding OOD errors, scaling robustly to large-scale challenges where standard unconditional policy distillation methods collapse.



\subsection{Deep dive into multi-modal datasets (RQ3)}
\label{subsec:multimodal_analysis}

To evaluate robustness against multi-modal distributions, we construct mixed datasets combining expert (optimal) and random (sub-optimal) trajectories, as discussed in~\citep{jiang2023recovering,zhang2022state}. Maintaining a 1,000,000-sample buffer, we vary the ratio of random samples from $50\%$ to $90\%$, progressively diminishing the expert mode's presence, which serves as a rigorous stress test to determine if our mechanisms can successfully isolate the vanishing optimal mode from multi-modal data.

\begin{figure}[h]
    \centering
    \includegraphics[width=1\linewidth]{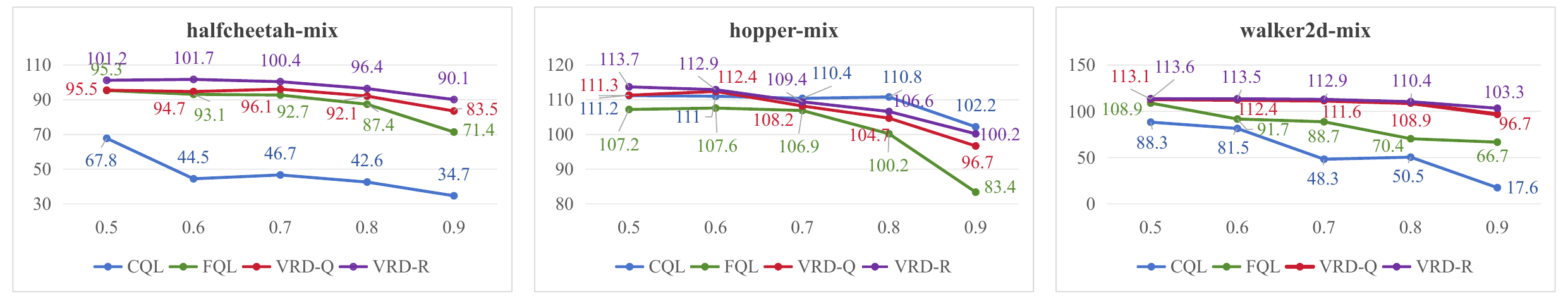}
    \caption{Results over different scales of multi-modal datasets.}
    \label{fig:exp_mixdata}
\end{figure}

From Figure \ref{fig:exp_mixdata}, we observe that both the statistical filtering of VRD-Q and the geometric reshaping of VRD-R exhibit a more gradual performance degradation as the proportion of sub-optimal modes (i.e., the random sample ratio) increases. Specifically, on the 'halfcheetah-mix' and 'walker2d-mix' datasets, both VRD variants significantly outperform the CQL and FQL baselines. This improved robustness is likely because agents in these two environments are more susceptible to overfitting to sub-optimal noisy data, a vulnerability that the proposed VRD methods appear to effectively mitigate.

\subsection{Hyperparameter analysis (RQ4)}
\label{subsec:ablations}

\begin{wraptable}{r}{0.55\textwidth}
    \centering
    \caption{Sensitivity analysis for main hyperparameters $\tau$ of VRD-Q and $\sigma$ of VRD-R ($\lambda=0.05$, full results see Appendix \ref{appendix:ablation_study}).}
    \label{tab:wrap_results}
    \setlength{\tabcolsep}{3pt} 
    \resizebox{\linewidth}{!}{
    \begin{tabular}{l|ccccc|ccccc}
        \toprule
        & \multicolumn{5}{c}{\textbf{VRD-Q} ($\tau$)} & \multicolumn{5}{c}{\textbf{VRD-R} ($\sigma$)} \\
        \cmidrule(lr){2-6} \cmidrule(lr){7-11}
        Values & 0.5 & 0.6 & 0.7 & 0.8 & 0.9 & 1 & 1e-2 & 1e-3 & 1e-4 & 1e-5 \\
        \midrule
        halfcheetah-m & 60.4 & 63.4 & \textbf{65.7} & 65.3 & 58.1 & 45.7 & 60.6 & 60.4 & \textbf{61.7} & 61.2 \\
        hopper-m & 64.3 & 65.1 & \textbf{65.8} & 65.1 & 56.4 & 50.2 & 73.9 & 75.5 & \textbf{76.3} & 74.2 \\
        walker2d-m & 76.5 & 81.7 & \textbf{82.3} & 81.1 & 77.4 & 53.9 & 89.1 & 88.6 & \textbf{89.7} & 88.3 \\
        \bottomrule
    \end{tabular}}
    \vspace{-1em} 
\end{wraptable}

We evaluate the sensitivity of the quantile parameter $\tau$ in VRD-Q and the scale parameter $\sigma$ of latent variables in VRD-R's inference. From Table \ref{tab:wrap_results}, VRD-Q performance follows an inverted U-shape, peaking at $\tau=0.7$ to balance noise exclusion with distribution coverage. VRD-R achieves its optimum at $\sigma=10^{-4}$, where moderate transformation effectively decouples modes without distorting latent topology. We recommend $\tau=0.7$ and $\sigma=10^{-4}$ as default settings. Specifically, $\tau$ should be based on sub-optimal data proportions, while $\sigma$ should be tailored to the specific characteristics of environments.

\textit{Other experimental details see Appendix \ref{appendix:experiments}, e.g., Minari results (Appendix \ref{appendix:minari}), full hyperparameter analysis (Appendix \ref{appendix:ablation_study}) and validation study for core components of VRD (Appendix \ref{appendix:validation}).}

\section{Conclusion}
This paper proposes the Value-Rectified Distillation (VRD) framework, reformulating policy distillation as a value-guided trajectory alignment process to solve the \textit{latent fusion} problem in offline RL over multi-modal dataset. VRD-Q and VRD-R mitigate OOD Mode Averaging problem via statistical filtering and geometric reshaping respectively. Experiments show VRD successfully decouples the high-value behavior mode from the multi-modal data, which enables VRD outperform baselines on various benchmarks and datasets, demonstrating exceptional robustness for sub-optimal data noise.

\paragraph{Limitations (see Appendix \ref{appendix:discussion}).} The proposed VRD relies on the underlying value network's accuracy, where severe Q-value overestimation may mislead the alignment process. Future work could explore more robust value estimation or extend this framework to discrete action spaces. 

\section*{Acknowledgements}
This work was partially supported by National Natural Science Foundation of China (62476128). This work was partially supported by JSPS KAKENHI Grant Number JP26H02498, and JST Moonshot R\&D Program JPMJMS263K.

\newpage
\bibliographystyle{named}
\bibliography{ref}


\newpage
\appendix

\section{Proofs}
\subsection{Proof of Proposition \ref{prop:latentfusion}}\label{appendix:proofofpropositionlatentfusion}

\noindent\textbf{Proposition \ref{prop:latentfusion}.} \textit{Suppose the latent space $\mathcal{Z}$ (support of $z \sim \mathcal{N}(0, I)$) is connected, and the flow mapping $\mu_\theta(s, z)$ is $L$-Lipschitz continuous w.r.t. $z$. Let $\mathcal{Z}^*_Q(s)$, $\mathcal{Z}^-_Q(s)$, and $\mathcal{Z}_{OOD}(s)$ denote the latent preimages of $\mathcal{A}^*_Q(s)$, $\mathcal{A}^-_Q(s)$, and $\bar{\mathcal{A}}(s)$, respectively. If the model expresses both optimal and sub-optimal modes ($\mathcal{Z}^*_Q(s) \neq \emptyset, \mathcal{Z}^-_Q(s) \neq \emptyset$), the following holds:
\textbf{(1)} The optimal latent boundary is non-empty ($\partial \mathcal{Z}^*_Q(s) \neq \emptyset$);
\textbf{(2)} The OOD latent set is strictly non-empty ($\mathcal{Z}_{OOD}(s) \neq \emptyset$);
\textbf{(3)} Any $\epsilon$-neighborhood of a boundary point $z_{bd} \in \partial \mathcal{Z}^*_Q(s)$ intersects non-optimal regions: $\forall \epsilon \!>\! 0, \mathbb{B}_\epsilon(z_{bd}) \cap \big(\mathcal{Z}^-_Q(s) \cup \mathcal{Z}_{OOD}(s)\big) \neq \emptyset$.}

\begin{proof}






Due to the prior latent space $\mathcal{Z} = \mathbb{R}^d$ (often a standard Gaussian $\mathcal{N}(0,I)$) is path-connected, there exists a continuous path $\gamma: [0, 1] \to \mathcal{Z}$ connecting the pre-images:
\begin{equation} \label{eq:path_z}
    \gamma(0) = z^* \in \mu_\theta^{-1}(\mathcal{A}_Q^*), \quad \gamma(1) = z^- \in \mu_\theta^{-1}(\mathcal{A}_Q^-)
\end{equation}
Let the mapped trajectory be $\Lambda = \gamma([0, 1])$. Because $\Lambda$ is a connected set that intersects both the interior of $\mathcal{Z}^*_Q(s)$ and its exterior complement, it cannot be partitioned into two disjoint open sets. By the fundamental property of connected spaces \citep{munkres2014topology} (Section 23, Exercise 6), $\Lambda$ must intersect the topological boundary of $\mathcal{Z}^*_Q(s)$. Therefore, there exists at least one boundary point $z_{bd} \in \partial \mathcal{Z}^*_Q(s)$ along the path, strictly proving \textbf{claim (1)}: $\partial \mathcal{Z}^*_Q(s) \neq \emptyset$.

By composition of continuous functions $\mu_\theta(s,z)$, the mapped trajectory $\Gamma$ in the action space is a connected set~\citep{rudin2021principles} (Theorem 4.22):
\begin{equation} \label{eq:gamma_a}
    \Gamma = \{ \mu_\theta(s, \gamma(t)) \mid t \in [0, 1] \}
\end{equation}

Since $\Gamma$ is connected and connects $\mathcal{A}_Q^*$ and $\mathcal{A}_Q^-$ across a geometric gap $D$, $\Gamma$ cannot be entirely contained within the disconnected support $\mathcal{A}=\mathcal{A}_Q^*\cup\mathcal{A}_Q^-$~\citep{munkres2014topology} (Lemma 23.2). Define the OOD latent set:
\begin{equation} \label{eq:z_ood}
    \mathcal{Z}_{OOD} = \{ z \in \mathcal{Z} \mid \mu_\theta(s, z) \in \bar{\mathcal{A}} \}
\end{equation}
By the Intermediate Value Theorem~\citep{rudin2021principles} (Theorem 4.23) on connected spaces, $\Gamma \cap \bar{\mathcal{A}} \neq \emptyset \implies \mathcal{Z}_{OOD} \neq \emptyset$, hence proving \textbf{claim (2)}.

Then, observe that the action space is partitioned into three mutually disjoint sets: the optimal mode $\mathcal{A}^*_Q(s)$, the sub-optimal mode $\mathcal{A}^-_Q(s)$, and the OOD region $\bar{\mathcal{A}}(s)$. Because the flow mapping $\mu_\theta(s, \cdot)$ is a well-defined function assigning every latent $z$ to exactly one action, its pre-images preserve this disjointness and form a mutually exhaustive partition of the entire latent space $\mathcal{Z}$:
\begin{equation} \label{eq:partition}
    \mathcal{Z} = \mathcal{Z}^*_Q(s) \cup \mathcal{Z}^-_Q(s) \cup \mathcal{Z}_{OOD}(s)
\end{equation}
Consequently, the topological complement of the optimal pre-image in $\mathcal{Z}$ is exactly the union of the remaining regions: $\mathcal{Z} \setminus \mathcal{Z}^*_Q(s) = \mathcal{Z}^-_Q(s) \cup \mathcal{Z}_{OOD}(s)$.

According to the standard neighborhood characterization of a topological boundary \citep{munkres2014topology} (Section 17, Exercise 19(b)), a point $z_{bd}$ belongs to the boundary $\partial \mathcal{Z}^*_Q(s)$ if and only if every open set containing $z_{bd}$ intersects both the set itself and its complement. Therefore, for any arbitrary radius $\epsilon > 0$, the open $\epsilon$-ball $\mathbb{B}_\epsilon(z_{bd})$ must strictly intersect the complement $\mathcal{Z} \setminus \mathcal{Z}^*_Q(s)$~\citep{munkres2014topology} (Theorem 17.5 (a)). Substituting the exact complement derived in Eq.(~\ref{eq:partition}) directly proves \textbf{claim (3)}:
\begin{equation} \label{eq:claim2}
    \forall \epsilon > 0, \quad \mathbb{B}_\epsilon(z_{bd}) \cap \big(\mathcal{Z}^-_Q(s) \cup \mathcal{Z}_{OOD}(s)\big) \neq \emptyset
\end{equation}

Completing the proof.

    
\end{proof}

\subsection{Proof of Theorem \ref{thm:ood_avg}}\label{appendix:proofofthmoodavg}

\noindent\textbf{Theorem \ref{thm:ood_avg}.} \textit{Suppose the Q-function satisfies $\mu$-strong concavity in the region bridging the modes and is $L$-Lipschitz smooth.
Consider the FQL policy update $\pi(s, z)$ which minimizes the regularized objective in Eq.(\ref{eq:fql_actor}). 
Under multi-modal data with gap $D > 0$ as Eq.(\ref{eq:dist_D}),
FQL fails to recover the optimal mode and instead generates OOD actions with a minimum expected error bounded by $\frac{D^2}{4}$.}

\begin{proof}
    We reformulate the objective in Eq.(\ref{eq:fql_actor}) as:
    \begin{align}
        \mathcal{L}_{\pi} = \mathbb{E}_{s \sim \mathcal{D}, z \sim \mathcal{N}(0, I_d)} [ -Q_\phi(s, \pi(s, z)) + \alpha \cdot \|\pi(s, z) - \mu_\theta(s, z)\|_2^2 ]
    \end{align}
    which switches the coefficient $\lambda$ over value to $\alpha$ onto the behavior regularization. 
    Based on the latent entanglement established in Proposition \ref{prop:latentfusion}, consider a latent variable $z \in \mathcal{Z}_Q^-(s)$ whose pre-trained behavior action falls in the sub-optimal mode, i.e., $a_{bc}^- = \mu_\theta(s, z) \in \mathcal{A}^-_Q(s)$. The objective of FQL given state $s$ aims to solve:
    \begin{align}
        &\min_{a} \left( -Q(s, a) + \alpha \|a - \mu_\theta(s, z)\|_2^2 \right)\\
        \Leftrightarrow &\max_{a} \left( Q(s, a) - \alpha \|a - \mu_\theta(s, z)\|_2^2 \right) \label{eq:objmax}
    \end{align}
    Denote the objective function as $\mathcal{J}(a) = Q(s, a) - \alpha \|a - a_{bc}^-\|_2^2$. Assuming $Q$ is twice differentiable, the optimal action generated by FQL, denoted as $a_{FQL}$, must satisfy the first-order optimality condition $\nabla_a\mathcal{J}(a_{FQL})=0$:
    \begin{align}
        \nabla_a Q(s, a_{FQL}) - 2\alpha (a_{FQL} - a_{bc}^-) = 0 \label{eq:exact_opt}
    \end{align}
    Let $a^* \in \mathcal{A}^*_Q(s)$ be the optimal action that maximizes $Q$, meaning $\nabla_aQ(s,a^*)=0$. By the Mean Value Theorem (or assuming locally quadratic $Q$), we can express the gradient using the Hessian matrix $H = \nabla^2_a Q(s, a)$:
    \begin{align}
        \nabla_a Q(s, a_{FQL}) = \nabla_a Q(s, a_{FQL}) - \nabla_a Q(s, a^*) = H (a_{FQL} - a^*) \label{eq:hessian}
    \end{align}
    Plugging Eq.(\ref{eq:hessian}) into Eq.(\ref{eq:exact_opt}), we obtain:
    \begin{align}
        H(a_{FQL} - a^*) - 2\alpha(a_{FQL} - a_{bc}^-) = 0
    \end{align}
    Assuming the Q-function satisfies $\mu$-strong concavity in the region bridging the modes and is $L$-Lipschitz smooth, there exists a symmetric positive-definite matrix $P = -H$ such that $\mu I \preceq P \preceq L I$. Substituting $P$:
    \begin{align}
        &-P(a_{FQL} - a^*) - 2\alpha (a_{FQL} - a_{bc}^-) = \mathbf{0} \nonumber \\
        \Rightarrow \quad & (2\alpha I + P)a_{FQL} = P a^* + 2\alpha a_{bc}^-
    \end{align}
    Because $P \succ 0$ and $2\alpha > 0$, the matrix $(2\alpha I + P)$ is strictly positive-definite and thus invertible. Therefore, we can explicitly solve for $a_{FQL}$:
    \begin{align}
        a_{FQL} = (2\alpha I + P)^{-1}P a^* + (2\alpha I + P)^{-1}(2\alpha I) a_{bc}^-\label{eq:alpha_fql}
    \end{align}
    Let the weight matrix be $M(\alpha) = (2\alpha I + P)^{-1}P$. Recognizing that $M(\alpha) + (2\alpha I + P)^{-1}(2\alpha I) = I$, the FQL action is precisely a matrix-weighted convex combination (interpolation) of the two modes:
    \begin{align}
        a_{FQL} = M(\alpha) a^* + (I - M(\alpha)) a_{bc}^-
    \end{align}
    
    Consider the trajectory of $a_{FQL}$ as the Euclidean regularization coefficient $\alpha \in (0, +\infty)$ varies, forming a continuous parametric curve $\gamma(\alpha) \subset \mathbb{R}^{\text{dim}(\mathcal{A})}$:
    \begin{itemize}
        \item When $\alpha \to 0$ (no behavior regularization), $M(\alpha) \to I$, yielding $\lim_{\alpha \to 0} \gamma(\alpha) = a^* \in \mathcal{A}^*_Q(s)$.
        \item When $\alpha \to +\infty$ (strict behavior cloning), $M(\alpha) \to \mathbf{0}$, yielding $\lim_{\alpha \to \infty} \gamma(\alpha) = a_{bc}^- \in \mathcal{A}^-_Q(s)$.
    \end{itemize}
    
    As established in Proposition \ref{prop:latentfusion}, the multi-modal valid support sets $\mathcal{A}^*_Q(s)$ and $\mathcal{A}^-_Q(s)$ are strictly separated by an Euclidean distance $D > 0$. By the connectedness theorem in topology, the continuous spatial curve $\gamma(\alpha)$ connecting these two separated modes must inherently traverse the low-density OOD vacuum region $\bar{\mathcal{A}}(s) = \{a \mid \pi_\beta(a|s)=0\}$.
    
    Specifically, there exists a regularization strength $\alpha$ where the interpolation is geometrically centered between the modes (e.g., when the eigenvalues of $M(\alpha)$ are roughly $0.5$). At this intersection, the generated action $a_{FQL}$ is positioned in the middle of the physical gap $D$. Consequently, its minimum Euclidean distance to \textit{any} valid action $a \in \mathcal{A}(s)$ is bounded below by $\frac{D}{2}$:
    \begin{align}
        \min_{a \in \mathcal{A}(s)} \|a_{FQL} - a\|_2 \geq \frac{\|a^* - a_{bc}^-\|_2}{2} \geq \frac{D}{2}
    \end{align}
    This strictly implies that the expected minimum squared prediction error incurred by the FQL mode averaging is bounded by:
    \begin{align}
        \min_{a \in \mathcal{A}(s)} \|a_{FQL} - a\|_2^2 \geq \left(\frac{D}{2}\right)^2 = \frac{D^2}{4} > 0
    \end{align}
    This proves that the action-space Euclidean penalty structurally forces FQL to output invalid OOD actions with a strict error lower bound, completing the proof.
\end{proof}

\subsection{Proof of Theorem \ref{thm:optimalsupport}}\label{appendix:proofofoptimalsupport}

\paragraph{Theorem \ref{thm:optimalsupport}.} 
Suppose the critic $Q$ is conservative on out-of-distribution regions such that $\sup_{a \in \overline{\mathcal{A}}(s)} Q(s, a) < \max_{a' \in \mathcal{A}(s)} Q(s, a') - \epsilon_{opt}$. Let $\hat{\pi}^* = \arg\min_\pi [\Psi(\pi; \mu_\theta) - \lambda \cdot \mathbb{E}_{z' \sim \mathcal{P}_{test}}[Q(s, \pi(s, z'))]]$, as in Eq.~(\ref{eq:vrd_objective}), be the learned optimal policy assuming universal approximation. Over the test support $\mathcal{Z}_{test} = \mathrm{supp}(\mathcal{P}_{test})$, if the non-merging alignment kernel $\kappa_Q$ satisfies (i) $\epsilon$-Bounded Suboptimal Leakage: $\int_{\mathcal{Z}_{test}}\int_{\mathcal{Z}_\mu^-} \mathcal{P}_{\mathcal{Z}}(z)\kappa_Q(z, z' \mid \cdot) \, dz dz' \le \epsilon$; and (ii) $C$-Dominant Optimal Coupling: $\int_{\mathcal{Z}_{test}}\int_{\mathcal{Z}_\mu^*} \mathcal{P}_{\mathcal{Z}}(z)\kappa_Q(z, z' \mid \cdot) \, dz dz' \ge C > 0$, then under the kernel-induced marginal distribution over $\mathcal{Z}_{\text{test}}$, the policy guarantees to output optimal actions with probability $\mathbb{P}(\hat{\pi}^*(s, z') \in \mathcal{A}_Q^*(s)) \ge \frac{C}{C+\epsilon}$.

\begin{proof}
To analyze the exact properties of the learned optimal policy $\hat{\pi}^*$ without relying on global unimodal convexity assumptions, we perform variational calculus on the regularized global objective in the function space. The proof proceeds in the following rigorous mathematical steps:

First, we expand the objective functional in Eq.~(\ref{eq:vrd_objective}) into integral forms over the latent priors $\mathcal{P}_{\mathcal{Z}}(z)$ and $\mathcal{P}_{test}(z')$. By applying Fubini-Tonelli's Theorem (~\cite{folland1999real}, Theorem 2.37) to swap the integration order, we isolate the inner pointwise objective for any given $z' \in \mathrm{supp}(\mathcal{P}_{test})$:
\begin{equation}
\label{eq:pointwise_loss}
\mathcal{L}(\pi(s, z')) = \int_{\mathcal{Z}} \mathcal{P}_{\mathcal{Z}}(z)\kappa_Q(z, z' \mid \cdot) \|\mu_\theta(s, z) - \pi(s, z')\|_2^2 \, dz - \lambda \mathcal{P}_{test}(z') Q(s, \pi(s, z'))
\end{equation}
Let $Z_{norm}(z') = \int_{\mathcal{Z}} \mathcal{P}_{\mathcal{Z}}(z)\kappa_Q(z, z' \mid \cdot) \, dz$ denote the marginal coupling density at $z'$, and let $\mathcal{W}(z \mid z') = \mathcal{P}_{\mathcal{Z}}(z)\kappa_Q(z, z' \mid \cdot) / Z_{norm}(z')$ be the induced conditional probability distribution over the source latent space $\mathcal{Z}$. Treating the alignment kernel $\kappa_Q$ as a fixed guidance operator during the variational evaluation step (consistent with the stop-gradient formulation in Section~3.1), dividing Eq.~(\ref{eq:pointwise_loss}) by $Z_{norm}(z') > 0$ shows that the optimal action $\hat{\pi}^*(s, z')$ maximizes the conditionally weighted functional:
\begin{equation}
\label{eq:variational_max}
\hat{\pi}^*(s, z') = \arg\max_{a} \left\{ Q(s, a) - \frac{1}{2\alpha(z')} \int_{\mathcal{Z}} \mathcal{W}(z \mid z') \|\mu_\theta(s, z) - a\|_2^2 \, dz \right\}, \quad \alpha(z') = \frac{\lambda \mathcal{P}_{test}(z')}{2 Z_{norm}(z')}
\end{equation}

Then we evaluate the optimality of $\hat{\pi}^*(s, z')$ strictly within the local fiber supported by $\mathcal{W}(\cdot \mid z')$. Specifically, for any candidate action $a_{cand}$ constructed from the support of $\mathcal{W}(\cdot \mid z')$, comparing the objective value in Eq.~(\ref{eq:variational_max}) at the global maximizer $a = \hat{\pi}^*(s, z')$ with its value at $a = a_{cand}$ yields:
\begin{align}
Q(s, \hat{\pi}^*(s, z')) &- \frac{1}{2\alpha(z')} \int_{\mathcal{Z}} \mathcal{W}(\tilde{z} \mid z') \|\mu_\theta(s, \tilde{z}) - \hat{\pi}^*(s, z')\|_2^2 \, d\tilde{z} \nonumber \\
&\ge Q(s, a_{cand}) - \frac{1}{2\alpha(z')} \int_{\mathcal{Z}} \mathcal{W}(\tilde{z} \mid z') \|\mu_\theta(s, \tilde{z}) - a_{cand}\|_2^2 \, d\tilde{z}
\label{eq:candidate_comparison}
\end{align}

Next, consider drawing $(z, z')$ from the effective joint coupling distribution over $\mathcal{Z} \times \mathrm{supp}(\mathcal{P}_{test})$ with density proportional to $\mathcal{P}_{\mathcal{Z}}(z)\kappa_Q(z, z' \mid \cdot)$. Let $I^* = \int_{\mathrm{supp}(\mathcal{P}_{test})}\int_{\mathcal{Z}_\mu^*} \mathcal{P}_{\mathcal{Z}}(z)\kappa_Q(z, z' \mid \cdot) \, dz dz'$ and $I^- = \int_{\mathrm{supp}(\mathcal{P}_{test})}\int_{\mathcal{Z}_\mu^-} \mathcal{P}_{\mathcal{Z}}(z)\kappa_Q(z, z' \mid \cdot) \, dz dz'$ denote the effective optimal and suboptimal coupling masses over the test support, respectively. By conditions (i)--(ii) ($I^- \le \epsilon$ and $I^* \ge C > 0$) and the monotonicity of $g(x, y) = \frac{x}{x+y}$:
\begin{equation}
\label{eq:prob_z_optimal}
\mathbb{P}(z \in \mathcal{Z}_\mu^*) = \frac{I^*}{I^* + I^-} \ge \frac{C}{C + \epsilon}
\end{equation}

It remains to show that $z \in \mathcal{Z}_\mu^*$ implies $\hat{\pi}^*(s, z') \in \mathcal{A}_Q^*(s)$. Recall from Section~\ref{sec:unified} that the alignment kernel is formulated as a non-merging transport kernel $\kappa_Q(z, z' \mid \cdot) = w(z)\delta(z' - T(z))$ with disjoint target images $T(\mathcal{Z}_\mu^*) \cap T(\mathcal{Z}_\mu^-) = \emptyset$. For any fixed target latent $z' \in \mathrm{supp}(\mathcal{P}_{test})$, the sifting property of the Dirac delta distribution restricts the effective support of $\kappa_Q(\cdot, z' \mid \cdot)$ over the source space $\mathcal{Z}$ strictly to the pre-image fiber:
\begin{equation}
\label{eq:preimage_fiber}
\mathrm{supp}_z(\kappa_Q(\cdot, z' \mid \cdot)) \subseteq T^{-1}(\{z'\}) = \{z \in \mathcal{Z} \mid T(z) = z'\}
\end{equation}
Since $T(\mathcal{Z}_\mu^*) \cap T(\mathcal{Z}_\mu^-) = \emptyset$, the pre-image fiber $T^{-1}(\{z'\})$ in Eq.~(\ref{eq:preimage_fiber}) can intersect at most one of the two subsets:
\begin{itemize}
    \item If $z' \in T(\mathcal{Z}_\mu^*)$, then $z' \notin T(\mathcal{Z}_\mu^-) \Rightarrow T^{-1}(\{z'\}) \cap \mathcal{Z}_\mu^- = \emptyset$, which implies that $\kappa_Q(\cdot, z' \mid \cdot)$ vanishes identically on $\mathcal{Z}_\mu^-$.
    \item If $z' \in T(\mathcal{Z}_\mu^-)$, then $z' \notin T(\mathcal{Z}_\mu^*) \Rightarrow T^{-1}(\{z'\}) \cap \mathcal{Z}_\mu^* = \emptyset$, which implies that $\kappa_Q(\cdot, z' \mid \cdot)$ vanishes identically on $\mathcal{Z}_\mu^*$.
    \item If $z' \notin T(\mathcal{Z}_\mu^*) \cup T(\mathcal{Z}_\mu^-)$, $\kappa_Q(\cdot, z' \mid \cdot)$ vanishes on both subsets.
\end{itemize}
Consequently, for every $z' \in \mathrm{supp}(\mathcal{P}_{test})$, $\kappa_Q(\cdot, z' \mid \cdot)$ vanishes on $\mathcal{Z}_\mu^-$ or on $\mathcal{Z}_\mu^*$, preventing any simultaneous non-zero coupling between optimal and suboptimal modes. Thus, up to a null event, $z \in \mathcal{Z}_\mu^*$ occurs if and only if $z' \in T(\mathcal{Z}_\mu^*)$, meaning $\mathrm{supp}(\mathcal{W}(\cdot \mid z')) \subseteq \mathcal{Z}_\mu^*$.

Crucially, for any $z' \in T(\mathcal{Z}_\mu^*)$, every source latent $\tilde{z} \in \mathrm{supp}(\mathcal{W}(\cdot \mid z'))$ belongs to $\mathcal{Z}_\mu^*$ and satisfies $\mu_\theta(s, \tilde{z}) \in \mathcal{A}_Q^*(s)$, i.e., $Q(s, \mu_\theta(s, \tilde{z})) \ge \max_{a'} Q(s, a') - \epsilon_{opt}$. Let $\bar{\mu}(s, z') = \int_{\mathcal{Z}} \mathcal{W}(\tilde{z} \mid z')\mu_\theta(s, \tilde{z}) \, d\tilde{z}$ denote the conditional optimal barycenter over the fiber $T^{-1}(\{z'\})$. By applying the bias-variance decomposition:
\begin{equation}
\label{eq:bias_variance}
\int_{\mathcal{Z}} \mathcal{W}(\tilde{z} \mid z') \|\mu_\theta(s, \tilde{z}) - a\|_2^2 \, d\tilde{z} = \|a - \bar{\mu}(s, z')\|_2^2 + \mathrm{Var}_{\mathcal{W}}(\mu_\theta)
\end{equation}
to Eq.~(\ref{eq:candidate_comparison}) and evaluating the objective against the candidate action $a_{cand} = \bar{\mu}(s, z')$, the constant variance term $\mathrm{Var}_{\mathcal{W}}(\mu_\theta)$ cancels on both sides. We analyze the two resulting alignment regimes:
\begin{enumerate}
    \item \textbf{Bijective / In-place Alignment (as in VRD-Q, $T(z) = z$):} The pre-image fiber in Eq.~(\ref{eq:preimage_fiber}) is a singleton $T^{-1}(\{z'\}) = \{z'\}$, so $\mathcal{W}(\tilde{z} \mid z') = \delta(\tilde{z} - z')$ is a degenerate point mass at a single optimal action $\bar{\mu}(s, z') = \mu_\theta(s, z') \in \mathcal{A}_Q^*(s)$ with $\mathrm{Var}_{\mathcal{W}}(\mu_\theta) = 0$. This preserves the exact multi-modal geometry of $\mathcal{A}_Q^*(s)$ point-by-point without any inter-mode averaging or convexity assumptions, reducing Eq.~(\ref{eq:candidate_comparison}) directly to:
    \begin{align}
    Q(s, \hat{\pi}^*(s, z')) &\ge Q(s, \mu_\theta(s, z')) + \frac{1}{2\alpha(z')} \|\hat{\pi}^*(s, z') - \mu_\theta(s, z')\|_2^2 \nonumber \\
    &\ge Q(s, \mu_\theta(s, z')) \ge \max_{a'} Q(s, a') - \epsilon_{opt}
    \label{eq:regime_bijective}
    \end{align}
    \item \textbf{Contractive / Reshaping Alignment (as in VRD-R):} When $T$ compresses multiple optimal latents into the same target coordinate $z'$, by the fiber-wise optimal preservation property of the non-merging kernel $\kappa_Q$ (which we explicitly verify for VRD-R under a locally concave optimal mode in Proposition~\ref{prop:vrdr}), the conditional barycenter satisfies $\bar{\mu}(s, z') \in \mathcal{A}_Q^*(s)$. Substituting Eq.~(\ref{eq:bias_variance}) with $a_{cand} = \bar{\mu}(s, z')$ into Eq.~(\ref{eq:candidate_comparison}) yields:
    \begin{align}
    Q(s, \hat{\pi}^*(s, z')) &\ge Q(s, \bar{\mu}(s, z')) + \frac{1}{2\alpha(z')} \|\hat{\pi}^*(s, z') - \bar{\mu}(s, z')\|_2^2 \nonumber \\
    &\ge Q(s, \bar{\mu}(s, z')) \ge \max_{a'} Q(s, a') - \epsilon_{opt}
    \label{eq:regime_contractive}
    \end{align}
\end{enumerate}

In either regime, combining Eq.~(\ref{eq:regime_bijective}) and Eq.~(\ref{eq:regime_contractive}), the event $z \in \mathcal{Z}_\mu^*$ (equivalently $z' \in T(\mathcal{Z}_\mu^*)$) directly guarantees $Q(s, \hat{\pi}^*(s, z')) \ge \max_{a' \in \mathcal{A}(s)} Q(s, a') - \epsilon_{opt}$. Furthermore, by the conservatism of the critic $Q$ on out-of-distribution regions ($\sup_{a \in \overline{\mathcal{A}}(s)} Q(s, a) < \max_{a' \in \mathcal{A}(s)} Q(s, a') - \epsilon_{opt}$), this value lower bound strictly precludes $\hat{\pi}^*(s, z')$ from falling into $\overline{\mathcal{A}}(s)$, thereby guaranteeing $\hat{\pi}^*(s, z') \in \mathcal{A}(s)$ and hence $\hat{\pi}^*(s, z') \in \mathcal{A}_Q^*(s)$. Combined with Eq.~(\ref{eq:prob_z_optimal}), we obtain:
\begin{equation}
\label{eq:final_prob_bound}
\mathbb{P}(\hat{\pi}^*(s, z') \in \mathcal{A}_Q^*(s)) \ge \mathbb{P}(z \in \mathcal{Z}_\mu^*) \ge \frac{C}{C + \epsilon}
\end{equation}
Completing the proof.
\end{proof}

\subsection{Proof of Proposition \ref{prop:vrdq}}\label{appendix:proofofvrdq}

\paragraph{Proposition \ref{prop:vrdq}.} \textit{Let $P^{*}$ and $P^{-}$ be the prior masses of $\mathcal{Z}_{\mu}^{*}$ and $\mathcal{Z}_{\mu}^{-}$, respectively. 
For given constants $\epsilon > 0$ and $C > 0$, if the quantile threshold $\tau \in (0.5, 1)$ satisfies 
$\tau P^{*} \ge C$ and $(1-\tau)P^{-} + (2\tau-1)P^{-}_{\text{band}}(\tau) \le \epsilon$ 
(where $P^{-}_{\text{band}}(\tau)$ denotes the prior mass of the boundary band near $\mathcal{A}_{Q}^{*}$): (a) The VRD-Q kernel satisfies the coupling conditions of Theorem~\ref{thm:optimalsupport}, guaranteeing 
$\mathbb{P}_{\kappa_{Q}}(\hat{\pi}^{*}(s,z) \in \mathcal{A}_{Q}^{*}(s)) \ge \frac{C}{C+\epsilon}$ 
under the kernel-induced marginal distribution; (b) Under the standard prior $z \sim \mathcal{N}(0, I_{d})$ at inference, 
$\mathbb{P}_{z}(\hat{\pi}^{*}(s,z) \in \mathcal{A}_{Q}^{*}(s)) \ge P^{*} + P^{-}_{\text{rect}}(\tau)$, 
where $P^{-}_{\text{rect}}(\tau)$ is the rectified prior mass.}

\begin{proof}
Recall the VRD-Q kernel with optimality margin:
\begin{equation}
\kappa_{Q}(z, z^{\prime} \mid s, \mu_{\theta}, \pi)
= \underbrace{\Big|\tau - \mathbb{I}\big(Q(s, \mu_{\theta}(s, z)) < Q(s, \pi(s, z)) - \epsilon_{\text{opt}}\big)\Big|}_{\omega_{\tau}(z)}
\cdot \delta(z^{\prime} - z).
\label{eq:prop2_kernel}
\end{equation}
Since the weight depends on $\pi$, we use the stop-gradient fixed point of Theorem~\ref{thm:optimalsupport}: $\hat{\pi}^{*}$ is a variational
optimum if, for every $z$, $\hat{\pi}^{*}(s,z)$ maximizes the pointwise objective of Eq.~(\ref{eq:variational_max}) with $\omega_{\tau}(z)$
evaluated at $\pi = \hat{\pi}^{*}$. Because $T(z) = z$ and $\mathcal{P}_{\text{test}}(z) = \mathcal{P}_{\mathcal{Z}}(z)$, the penalty coefficient $\frac{1}{2\alpha(z)}$ in Eq.~(\ref{eq:variational_max}) equals $\alpha_{\text{reg}}(z) := \omega_{\tau}(z)/\lambda$, reducing Eq.~(\ref{eq:variational_max}) to
\begin{equation}
\hat{\pi}^{*}(s,z) = \arg\max_{a}\Big[Q(s,a) - \alpha_{\text{reg}}(z)\,\|a - \mu_{\theta}(s,z)\|_{2}^{2}\Big],
\qquad \alpha_{\text{reg}}(z) = \omega_{\tau}(z)/\lambda .
\label{eq:prop2_pointwise}
\end{equation}

\textbf{Step 1: Closed form and improvement under the quadratic model.}
Write $v = \mu_{\theta}(s,z) - a^{*}$. For any regularization strength $\alpha > 0$, the first-order condition of Eq.~\eqref{eq:prop2_pointwise},
$-P(a-a^{*}) - 2\alpha(a-\mu_{\theta}(s,z)) = 0$, gives (cf.\ Eq.~(\ref{eq:alpha_fql}))
\begin{equation}
\hat{\pi}^{*}(s,z) - a^{*} = (I - M(\alpha))\,v, \qquad I - M(\alpha) = 2\alpha(2\alpha I + P)^{-1}.
\label{eq:prop2_closed}
\end{equation}
Let $p_{i} \in [\mu_{Q}, L]$ be the eigenvalues of $P$, $v_{i}$ the coordinates of $v$ in its eigenbasis, and
$r_{i}(\alpha) = \frac{2\alpha}{2\alpha + p_{i}} \in (0,1)$. The value improvement of the actor over the behavior action is
\begin{equation}
\Delta(\alpha) := Q(s,\hat{\pi}^{*}(s,z)) - Q(s,\mu_{\theta}(s,z)) = \tfrac{1}{2}\sum_{i} p_{i}\big(1 - r_{i}(\alpha)^{2}\big)v_{i}^{2}.
\label{eq:prop2_delta}
\end{equation}
Each $r_{i}(\alpha)$ increases with $\alpha$, so $\Delta(\alpha)$ is strictly decreasing in $\alpha$ whenever $v \neq 0$. Moreover,
$r_{i}(\alpha) \le \frac{2\alpha}{2\alpha + \mu_{Q}}$, hence $1 - r_{i}(\alpha)^{2} \ge c(\alpha)$ and
\begin{equation}
\Delta(\alpha) \ge c(\alpha)\cdot\tfrac{1}{2}v^{\top}P v = c(\alpha)\big(Q_{\max}(s) - Q(s,\mu_{\theta}(s,z))\big).
\label{eq:prop2_delta_lb}
\end{equation}
Note that the indicator in Eq.~\eqref{eq:prop2_kernel} equals $1$ if and only if $\Delta(\alpha_{\text{reg}}(z)) > \epsilon_{\text{opt}}$.

\textbf{Step 2: Weights on $\mathcal{Z}_{\mu}^{*}$.}
By critic conservatism, $\sup_{a \in \overline{\mathcal{A}}(s)} Q(s,a) < Q_{\max}(s) - \epsilon_{\text{opt}}$, and
$Q(s,a) \le Q_{\max}(s)$ on $\mathcal{A}(s)$; hence $\sup_{a} Q(s,a) = Q_{\max}(s)$. For any $z \in \mathcal{Z}_{\mu}^{*}$
and any policy $\pi$,
\begin{equation}
Q(s,\pi(s,z)) - \epsilon_{\text{opt}} \le Q_{\max}(s) - \epsilon_{\text{opt}} \le Q(s,\mu_{\theta}(s,z)),
\end{equation}
so the indicator vanishes and $\omega_{\tau}(z) = \tau$ for every $z \in \mathcal{Z}_{\mu}^{*}$, independently of $\pi$.
Since $\omega_{\tau}(z) > 0$, Eq.~(\ref{eq:regime_bijective}) of Theorem~\ref{thm:optimalsupport} gives
$Q(s,\hat{\pi}^{*}(s,z)) \ge Q(s,\mu_{\theta}(s,z)) \ge Q_{\max}(s) - \epsilon_{\text{opt}}$; by conservatism
$\hat{\pi}^{*}(s,z) \notin \overline{\mathcal{A}}(s)$, hence $\hat{\pi}^{*}(s,z) \in \mathcal{A}_{Q}^{*}(s)$ for every
$z \in \mathcal{Z}_{\mu}^{*}$.

\textbf{Step 3: Weights on $\mathcal{Z}_{\mu}^{-}$.}
Let $z \in \mathcal{Z}_{\mu}^{-} \setminus \mathcal{Z}^{-}_{\text{band}}(\tau)$, i.e.,
$Q_{\max}(s) - Q(s,\mu_{\theta}(s,z)) > \epsilon_{\text{opt}}/c(\alpha_{\tau})$. By Eq.~\eqref{eq:prop2_delta_lb},
$\Delta(\alpha_{\tau}) > \epsilon_{\text{opt}}$. We check both candidate weights:
\begin{itemize}
\item If $\omega_{\tau}(z) = \tau$, then $\alpha_{\text{reg}}(z) = \alpha_{\tau}$ and $\Delta(\alpha_{\tau}) > \epsilon_{\text{opt}}$, so the
indicator equals $1$ and $\omega_{\tau}(z) = 1-\tau$, a contradiction.
\item If $\omega_{\tau}(z) = 1-\tau$, then $\alpha_{\text{reg}}(z) = (1-\tau)/\lambda < \alpha_{\tau}$ since $\tau > 0.5$. By monotonicity,
$\Delta((1-\tau)/\lambda) \ge \Delta(\alpha_{\tau}) > \epsilon_{\text{opt}}$, so the indicator equals $1$, which is consistent.
\end{itemize}
Hence $\omega_{\tau}(z) = 1-\tau$ is the unique consistent weight on
$\mathcal{Z}_{\mu}^{-} \setminus \mathcal{Z}^{-}_{\text{band}}(\tau)$. On the band $\mathcal{Z}^{-}_{\text{band}}(\tau)$ the
improvement may not exceed the margin, and we only use the bound $\omega_{\tau}(z) \le \tau$.

\textbf{Step 4: Coupling conditions of Theorem 2.}
Since $\mathcal{Z}_{\text{test}} = \mathcal{Z}$, $\int_{\mathcal{Z}_{\text{test}}} \delta(z^{\prime} - z)\,dz^{\prime} = 1$, and
\begin{align}
I^{*} &= \int_{\mathcal{Z}_{\mu}^{*}} \mathcal{P}_{\mathcal{Z}}(z)\,\omega_{\tau}(z)\,dz = \tau P^{*} \ge C, \\
I^{-} &= \int_{\mathcal{Z}_{\mu}^{-}} \mathcal{P}_{\mathcal{Z}}(z)\,\omega_{\tau}(z)\,dz
\le (1-\tau)\big(P^{-} - P^{-}_{\text{band}}(\tau)\big) + \tau P^{-}_{\text{band}}(\tau)\nonumber\\
&= (1-\tau)P^{-} + (2\tau-1)P^{-}_{\text{band}}(\tau) \le \epsilon .
\end{align}
Since $T(z) = z$, $T(\mathcal{Z}_{\mu}^{*}) \cap T(\mathcal{Z}_{\mu}^{-}) = \emptyset$ and every fiber is a singleton, so the
kernel is non-merging and fiber-wise optimal preservation holds trivially (regime 1 in the proof of Theorem~\ref{thm:optimalsupport}).
Theorem~\ref{thm:optimalsupport} then yields
$\mathbb{P}_{\kappa_{Q}}(\hat{\pi}^{*}(s,z) \in \mathcal{A}_{Q}^{*}(s)) \ge I^{*}/(I^{*}+I^{-}) \ge \frac{C}{C+\epsilon}$,
which proves (a).

\textbf{Step 5: Inference under the standard prior.}
At inference, $z \sim \mathcal{N}(0,I_{d})$ is not reweighted by $\omega_{\tau}$.
By Step~2, every $z \in \mathcal{Z}_{\mu}^{*}$ yields an action in $\mathcal{A}_{Q}^{*}(s)$, contributing mass $P^{*}$.
For $z \in \mathcal{Z}_{\mu}^{-} \setminus \mathcal{Z}^{-}_{\text{band}}(\tau)$, Step~3 gives $\omega_{\tau}(z) = 1-\tau$, and
Eq.~\eqref{eq:prop2_closed} with $\alpha^{\prime}_{\tau} = (1-\tau)/\lambda$ gives
$\hat{\pi}^{*}(s,z) = M(\alpha^{\prime}_{\tau})a^{*} + (I - M(\alpha^{\prime}_{\tau}))\mu_{\theta}(s,z)$. Define the rectified set
\begin{align}
\mathcal{Z}^{-}_{\text{rect}}(\tau) &= \Big\{z \in \mathcal{Z}_{\mu}^{-} \setminus \mathcal{Z}^{-}_{\text{band}}(\tau) :
\big\|(I - M(\alpha^{\prime}_{\tau}))(\mu_{\theta}(s,z) - a^{*})\big\|_{2} \le \sqrt{2\epsilon_{\text{opt}}/L}\Big\},\nonumber\\
P^{-}_{\text{rect}}(\tau) &= \int_{\mathcal{Z}^{-}_{\text{rect}}(\tau)} \mathcal{P}_{\mathcal{Z}}(z)\,dz .
\label{eq:prop2_rect}
\end{align}
For $z \in \mathcal{Z}^{-}_{\text{rect}}(\tau)$, $P \preceq LI$ gives
$Q_{\max}(s) - Q(s,\hat{\pi}^{*}(s,z)) = \frac{1}{2}(\hat{\pi}^{*}-a^{*})^{\top}P(\hat{\pi}^{*}-a^{*})
\le \frac{L}{2}\|\hat{\pi}^{*}-a^{*}\|_{2}^{2} \le \epsilon_{\text{opt}}$, and conservatism gives
$\hat{\pi}^{*}(s,z) \in \mathcal{A}_{Q}^{*}(s)$. Latents in the band and the remaining latents of $\mathcal{Z}_{\mu}^{-}$
are conservatively counted as failures. Therefore
\begin{equation}
\mathbb{P}_{z \sim \mathcal{N}(0,I_{d})}\big(\hat{\pi}^{*}(s,z) \in \mathcal{A}_{Q}^{*}(s)\big) \ge P^{*} + P^{-}_{\text{rect}}(\tau),
\end{equation}
which proves (b). Completing the proof.
\end{proof}

\subsection{Proof of Proposition \ref{prop:vrdr}}\label{appendix:proofofvrdr}

\textbf{Proposition~\ref{prop:vrdr}.} \textit{Suppose the critic $Q$ satisfies the conservatism condition in Theorem \ref{thm:optimalsupport}, and the optimal action set $\mathcal{A}_{Q}^{*}(s)$ forms a convex dominant mode over which $Q(s, \cdot)$ is concave. During inference, let $z^{\prime} \sim \mathcal{N}(0, \sigma^{2} I)$ be truncated to $\mathcal{Z}_{\text{test}} = \{z^{\prime} \in \mathcal{Z} \mid \|z^{\prime}\|_{2} \le \sigma\}$. If the truncation radius satisfies $\sigma < \epsilon_{\text{opt}} / (Q_{\max}(s) - Q_{\min}(s) + \epsilon_{m})$, where $\epsilon_{\text{opt}}$ is the optimality relaxation factor in Eq.~(\ref{eq:optimalaction}), then $\kappa_{Q}$ in Eq.~(\ref{eq:vrdr_kernel}) fulfills all conditions of Theorem \ref{thm:optimalsupport} with zero suboptimal leakage ($I^{-} = 0$), guaranteeing $\mathbb{P}(\hat{\pi}^{*}(s, z^{\prime}) \in \mathcal{A}_{Q}^{*}(s)) = 1$ under the kernel-induced marginal distribution.}

\begin{proof}
Recall the definition of the transport mapping for VRD-R in Eq.~(\ref{eq:vrdr_kernel}) and Eq.~(\ref{eq:vrdr_reshaping}). For any $z \in \mathcal{Z} \setminus \{0\}$ (where $\{0\}$ has Lebesgue and Gaussian measure zero), the source latent variable $z$ is mapped to the target coordinate $z_{Q}(s, z)$:
\begin{equation}
z_{Q}(s, z) = \Phi_{Q}(s, z) \cdot \frac{z}{\|z\|_{2}}, \quad \text{where} \quad \Phi_{Q}(s, z) = \frac{Q_{\max}(s) - Q(s, \mu_{\theta}(s, z))}{Q_{\max}(s) - Q_{\min}(s) + \epsilon_{m}}.
\label{eq:proof_prop3_map}
\end{equation}
Since $Q(s, \mu_{\theta}(s, z)) \le Q_{\max}(s)$ holds on $\mathcal{A}(s)$ by definition and on $\bar{\mathcal{A}}(s)$ by the critic conservatism assumption ($\sup_{a \in \bar{\mathcal{A}}(s)} Q(s, a) < Q_{\max}(s) - \epsilon_{\text{opt}}$), we have $\Phi_{Q}(s, z) \ge 0$ everywhere. Thus, the Euclidean norm of the mapped vector is strictly $\|z_{Q}(s, z)\|_{2} = \Phi_{Q}(s, z)$, and the alignment kernel is given by $\kappa_{Q}(z, z^{\prime} \mid \cdot) = \delta(z^{\prime} - z_{Q}(s, z))$.

According to Eq.~(\ref{eq:optimalaction}), Definition \ref{def:optimal_lat_sub}, and the critic conservatism assumption in Theorem \ref{thm:optimalsupport}, for any latent variable $z \in \mathcal{Z}_{\mu}^{-} = \mathcal{Z} \setminus \mathcal{Z}_{\mu}^{*}$ (which includes both suboptimal in-distribution latents with $\mu_{\theta}(s, z) \in \mathcal{A}_{Q}^{-}(s)$ and OOD latents with $\mu_{\theta}(s, z) \in \bar{\mathcal{A}}(s)$), the generated behavior action strictly fails the optimality threshold:
\begin{equation}
Q(s, \mu_{\theta}(s, z)) < Q_{\max}(s) - \epsilon_{\text{opt}}, \quad \forall z \in \mathcal{Z}_{\mu}^{-}.
\label{eq:proof_prop3_subopt_q}
\end{equation}
Substituting Eq.~\eqref{eq:proof_prop3_subopt_q} into the numerator of $\Phi_{Q}(s, z)$ in Eq.~\eqref{eq:proof_prop3_map}, we obtain a strict analytical lower bound for the mapping radius of all non-optimal latents:
\begin{equation}
\|z_{Q}(s, z)\|_{2} = \Phi_{Q}(s, z) > \frac{Q_{\max}(s) - (Q_{\max}(s) - \epsilon_{\text{opt}})}{Q_{\max}(s) - Q_{\min}(s) + \epsilon_{m}} = \frac{\epsilon_{\text{opt}}}{Q_{\max}(s) - Q_{\min}(s) + \epsilon_{m}} \triangleq \delta.
\label{eq:proof_prop3_delta}
\end{equation}
This implies that under the geometric reshaping of VRD-R, all suboptimal and OOD behaviors in $\mathcal{Z}_{\mu}^{-}$ are strictly repelled to the exterior region with radius $\|z_{Q}(s, z)\|_{2} > \delta$.

\textbf{Verification of Theorem \ref{thm:optimalsupport} Condition (i): $\epsilon$-Bounded Suboptimal Leakage.} During inference, the test prior $\mathcal{P}_{\text{test}}$ is supported on the closed ball $\mathcal{Z}_{\text{test}} = \{z^{\prime} \in \mathcal{Z} \mid \|z^{\prime}\|_{2} \le \sigma\}$. Given the condition $\sigma < \delta$, any $z^{\prime} \in \mathcal{Z}_{\text{test}}$ satisfies:
\begin{equation}
\|z^{\prime}\|_{2} \le \sigma < \delta.
\end{equation}
However, by Eq.~\eqref{eq:proof_prop3_delta}, every $z \in \mathcal{Z}_{\mu}^{-}$ satisfies $\|z_{Q}(s, z)\|_{2} > \delta > \sigma$. Consequently, for all $z \in \mathcal{Z}_{\mu}^{-}$, the target image $z_{Q}(s, z)$ falls strictly outside $\mathcal{Z}_{\text{test}}$, so the Dirac delta measure vanishes identically over $\mathcal{Z}_{\text{test}}$:
\begin{equation}
I^{-} = \int_{\mathcal{Z}_{\text{test}}} \int_{\mathcal{Z}_{\mu}^{-}} \mathcal{P}_{\mathcal{Z}}(z) \kappa_{Q}(z, z^{\prime} \mid \cdot) \, dz \, dz^{\prime} = \int_{\mathcal{Z}_{\mu}^{-}} \mathcal{P}_{\mathcal{Z}}(z) \left( \int_{\|z^{\prime}\|_{2} \le \sigma} \delta(z^{\prime} - z_{Q}(s, z)) \, dz^{\prime} \right) dz = 0.
\label{eq:proof_prop3_i_minus}
\end{equation}
Thus, Condition (i) of Theorem \ref{thm:optimalsupport} ($I^{-} \le \epsilon$) holds trivially for any $\epsilon \ge 0$, achieving exact zero suboptimal leakage ($I^{-} = 0$).

\textbf{Verification of Theorem \ref{thm:optimalsupport} Condition (ii): $C$-Dominant Optimal Coupling.} Conversely, for any optimal latent $z \in \mathcal{Z}_{\mu}^{*}$, Definition \ref{def:optimal_lat_sub} ensures $Q(s, \mu_{\theta}(s, z)) \ge Q_{\max}(s) - \epsilon_{\text{opt}}$. Substituting this into Eq.~\eqref{eq:proof_prop3_map} yields the upper bound on its mapped radius:
\begin{equation}
\|z_{Q}(s, z)\|_{2} = \Phi_{Q}(s, z) \le \delta, \quad \forall z \in \mathcal{Z}_{\mu}^{*}.
\label{eq:proof_prop3_opt_radius}
\end{equation}
That is, all optimal behaviors are compressed into the closed ball of radius $\delta$ centered at the origin, with peak-value actions ($Q(s, \mu_{\theta}(s, z)) = Q_{\max}(s)$) mapping directly to the origin ($\|z_{Q}(s, z)\|_{2} = 0$). Because the flow policy $\mu_{\theta}(s, z)$ and critic $Q(s, \cdot)$ are continuous and $\mu_{\theta}$ captures the optimal mode where $\max_{z \in \mathcal{Z}} Q(s, \mu_{\theta}(s, z)) = Q_{\max}(s)$, the pre-image of the test ball $\mathcal{Z}_{\text{test}} = \{z^{\prime} \mid \|z^{\prime}\|_{2} \le \sigma\}$ for any $\sigma > 0$ contains a non-empty open neighborhood in $\mathcal{Z}_{\mu}^{*}$. Since the Gaussian prior $\mathcal{P}_{\mathcal{Z}}$ has strictly positive density everywhere on $\mathcal{Z}$, integrating over $\mathcal{Z}_{\text{test}}$ yields a strictly positive optimal coupling mass:
\begin{equation}
I^{*} = \int_{\mathcal{Z}_{\text{test}}} \int_{\mathcal{Z}_{\mu}^{*}} \mathcal{P}_{\mathcal{Z}}(z) \delta(z^{\prime} - z_{Q}(s, z)) \, dz \, dz^{\prime} = \int_{\{z \in \mathcal{Z}_{\mu}^{*} \mid \|z_{Q}(s, z)\|_{2} \le \sigma\}} \mathcal{P}_{\mathcal{Z}}(z) \, dz \triangleq C > 0.
\label{eq:proof_prop3_i_star}
\end{equation}

\textbf{Verification of Mode Separation and Fiber-wise Optimal Preservation.} For the transport map $T(z) = z_{Q}(s, z)$, Eq.~\eqref{eq:proof_prop3_delta} shows $T(\mathcal{Z}_{\mu}^{-}) \subseteq \{z^{\prime} \in \mathcal{Z} \mid \|z^{\prime}\|_{2} > \delta\}$, whereas Eq.~\eqref{eq:proof_prop3_opt_radius} shows $T(\mathcal{Z}_{\mu}^{*}) \subseteq \{z^{\prime} \in \mathcal{Z} \mid \|z^{\prime}\|_{2} \le \delta\}$. Therefore, their target images are strictly disjoint:
\begin{equation}
T(\mathcal{Z}_{\mu}^{*}) \cap T(\mathcal{Z}_{\mu}^{-}) = \emptyset,
\end{equation}
which verifies the mode separation condition of Theorem \ref{thm:optimalsupport}.

Furthermore, for any $z^{\prime} \in T(\mathcal{Z}_{\mu}^{*}) \cap \mathcal{Z}_{\text{test}}$, its pre-image fiber $T^{-1}(\{z^{\prime}\}) = \{z \in \mathcal{Z} \mid z_{Q}(s, z) = z^{\prime}\}$ is contained entirely within $\mathcal{Z}_{\mu}^{*}$. Thus, the induced conditional distribution $\mathcal{W}(\tilde{z} \mid z^{\prime})$ is supported exclusively on $\mathcal{Z}_{\mu}^{*}$, meaning $\mu_{\theta}(s, \tilde{z}) \in \mathcal{A}_{Q}^{*}(s)$ for all $\tilde{z} \in \text{supp}(\mathcal{W}(\cdot \mid z^{\prime}))$. Since $\mathcal{A}_{Q}^{*}(s) \subseteq \mathcal{A}(s)$ is a convex set by assumption, the conditional barycenter $\bar{\mu}(s, z^{\prime}) = \int_{\mathcal{Z}} \mathcal{W}(\tilde{z} \mid z^{\prime}) \mu_{\theta}(s, \tilde{z}) \, d\tilde{z}$ is a convex combination of points in $\mathcal{A}_{Q}^{*}(s)$ and therefore remains strictly within the valid in-distribution set: $\bar{\mu}(s, z^{\prime}) \in \mathcal{A}_{Q}^{*}(s) \subseteq \mathcal{A}(s)$. Moreover, by the concavity of $Q(s, \cdot)$ on $\mathcal{A}_{Q}^{*}(s)$ and Jensen's inequality:
\begin{equation}
Q(s, \bar{\mu}(s, z^{\prime})) \ge \int_{\mathcal{Z}} \mathcal{W}(\tilde{z} \mid z^{\prime}) Q(s, \mu_{\theta}(s, \tilde{z})) \, d\tilde{z} \ge Q_{\max}(s) - \epsilon_{\text{opt}},
\label{eq:proof_prop3_jensen}
\end{equation}
confirming $\bar{\mu}(s, z^{\prime}) \in \mathcal{A}_{Q}^{*}(s)$ and fulfilling the fiber-wise optimal preservation condition required by Eq.~(\ref{eq:regime_contractive}) in Theorem \ref{thm:optimalsupport}.

\textbf{Conclusion.} Combining $I^{-} = 0$ (Eq.~\eqref{eq:proof_prop3_i_minus}) and $I^{*} = C > 0$ (Eq.~\eqref{eq:proof_prop3_i_star}) with Eq.~(\ref{eq:prob_z_optimal}) and Eq.~(\ref{eq:final_prob_bound}) of Theorem \ref{thm:optimalsupport}, we obtain:
\begin{equation}
\mathbb{P}(\hat{\pi}^{*}(s, z^{\prime}) \in \mathcal{A}_{Q}^{*}(s)) \ge \frac{I^{*}}{I^{*} + I^{-}} = \frac{C}{C + 0} = 1,
\end{equation}
which guarantees that under the kernel-induced marginal distribution over $\mathcal{Z}_{\text{test}}$, the VRD-R policy outputs optimal in-distribution actions within $\mathcal{A}_{Q}^{*}(s)$ with probability $1$, theoretically eliminating OOD Mode Averaging.
\end{proof}

\section{Algorithms}\label{appendix:algs}

\begin{algorithm}
\caption{Quantile-based Value-Rectified Distillation (VRD-Q) Algorithm}
\label{alg:VRDq_summary}
\begin{algorithmic}[1]
\REQUIRE Pre-trained flow policy $\mu_\theta$, Q-function evaluator $Q$, static dataset $\mathcal{D}$, base prior $\mathcal{Z}$, target actor policy $\pi$
\FOR{each batch $s \sim \mathcal{D}$, $z \sim \mathcal{Z}$}
    \STATE \textbf{1.} Calculate quantile-aware weight: \\
    \quad $\omega_\tau(z) = |\tau - \mathbb{I}(Q(s, \mu_\theta(s, z)) < Q(s, \pi(s, z)) - \epsilon_{\text{opt}})|$
    \STATE \textbf{2.} Construct in-place alignment kernel: \\
    \quad $\kappa_Q(z, z'|s, \mu_\theta, \pi) = \omega_\tau(z) \cdot \delta(z' - z)$
    \STATE \textbf{3.} Minimize Quantile-Rectified Loss to update $\pi$: \\
    \quad $\mathcal{L}_\pi = \mathbb{E} \left[\omega_\tau(z) \cdot ||\mu_\theta(s, z) - \pi(s, z)||_2^2  -  \lambda \cdot\mathbb{E}_{z' \sim \mathcal{P}_{test}}[Q(s, \pi(s, z'))]\right]$
\ENDFOR
\STATE \textbf{Zero-Shot Inference (VRD-Q):} Sample $z \sim \mathcal{N}(0, I_d)$, output $a = \pi^*(s, z)$
\end{algorithmic}
\end{algorithm}

\begin{algorithm}
\caption{Reshaping-based Value-Rectified Distillation (VRD-R) Algorithm}
\label{alg:VRDr_summary}
\begin{algorithmic}[1]
\REQUIRE Pre-trained flow policy $\mu_\theta$, Q-function evaluator $Q$, static dataset $\mathcal{D}$, base prior $\mathcal{Z}$, target actor policy $\pi$
\FOR{each batch $s \sim \mathcal{D}$, $z \sim \mathcal{Z}$}
    \STATE \textbf{1.} Calculate normalized ranking scalar: \\
    \quad $\Phi_Q(s, z) = \frac{Q_{max}(s) - Q(s, \mu_\theta(s, z))}{Q_{max}(s) - Q_{min}(s) + \epsilon_m}$
    \STATE \textbf{2.} Apply value-guided reshaping function (gravitational field towards origin): \\
    \quad $z_Q(s, z, \mu_\theta) = \Phi_Q(s, z) \cdot \frac{z}{||z||_2}$
    \STATE \textbf{3.} Construct transport alignment kernel: \\
    \quad $\kappa_Q(z, z'|s, \mu_\theta, \pi) = 1 \cdot \delta(z' - z_Q(s, z, \mu_\theta))$
    \STATE \textbf{4.} Minimize Reshaping-Rectified Loss to update $\pi$: \\
    \quad $\mathcal{L}_\pi = \mathbb{E} \left[||\mu_\theta(s, z) - \pi(s, z_Q(s, z, \mu_\theta))||_2^2  -  \lambda \cdot \mathbb{E}_{z' \sim \mathcal{P}_{test}}[Q(s, \pi(s, z'))]\right]$
\ENDFOR
\STATE \textbf{Zero-Shot Inference (VRD-R):} Sample $z' \sim \tilde{\mathcal{N}}(0, \sigma^2 I_d)$ (where $\sigma \ll 1$), output $a = \pi^*(s, z')$
\end{algorithmic}
\end{algorithm}

\section{External experimental details}\label{appendix:experiments}

\subsection{Full results over OGBenchmarks (50 tasks)}

\textbf{Dataset Overview.} OGBench (Offline Goal-Conditioned RL Benchmark) is a large-scale evaluation suite comprising 50 tasks across 10 diverse domains, including complex navigation (e.g., \textit{antmaze-giant}, \textit{humanoidmaze}), robotic manipulation (\textit{cube-play}, \textit{scene-play}), and logical puzzles (\textit{puzzle-3x3/4x4}). These tasks are specifically designed to test an agent's ability to reach diverse goals using multi-modal offline data, presenting significant challenges in mode separation and long-horizon planning.

\begin{table}[h]
\centering
\caption{Detailed Offline RL Results on OGBench Tasks: Performance comparison across different offline RLalgorithms. All prior work results are taken from~\citep{park2025flow} to compare a wide range of previous methods against VRD.}
\label{tab:ogbench_full}
\resizebox{\linewidth}{!}{
\begin{tabular}{l cccccccccc  c || cc}
\toprule
\textbf{Task} & BC & IQL & ReBRAC & IDQL & SRPO & CAC & FAWAC & FBRAC & IFQL & FQL & GFP actor $\pi_\omega$ & \textbf{VRD-Q (Ours)} & \textbf{VRD-R (Ours)} \\
\midrule

\multicolumn{14}{l}{\textit{antmaze-large-navigate-singletask}} \\
\midrule
task1-v0 & $0 \pm 0$ & $48 \pm 9$ & $91 \pm 10$ & $0 \pm 0$ & $0 \pm 0$ & $42 \pm 7$ & $1 \pm 1$ & $70 \pm 20$ & $24 \pm 17$ & $80 \pm 8$ & $ 95.4 \pm 0.8$ & $\mathbf{96.7} \pm 0.4$ & $\mathbf{97.2} \pm 0.6$ \\
task2-v0 & $6 \pm 3$ & $42 \pm 6$ & $88 \pm 4$ & $14 \pm 8$ & $4 \pm 4$ & $1 \pm 1$ & $0 \pm 1$ & $35 \pm 12$ & $8 \pm 3$ & $57 \pm 10$ & $ 92.2 \pm 3.0$ & $\mathbf{96.2} \pm 0.1$ & $93.1 \pm 3.3$ \\
task3-v0 & $29 \pm 5$ & $72 \pm 7$ & $51 \pm 18$ & $26 \pm 8$ & $3 \pm 2$ & $49 \pm 10$ & $12 \pm 4$ & $83 \pm 15$ & $52 \pm 17$ & $93 \pm 3$ & $ 95.6 \pm 2.7$ & $\mathbf{98.4} \pm 0.7$ & $96.4 \pm 1.7$ \\
task4-v0 & $8 \pm 3$ & $51 \pm 9$ & $84 \pm 7$ & $62 \pm 25$ & $45 \pm 19$ & $17 \pm 6$ & $10 \pm 3$ & $37 \pm 18$ & $18 \pm 8$ & $80 \pm 4$ & $\mathbf{90.6} \pm 2.6$ & $\mathbf{90.7} \pm 3.8$ & $\mathbf{90.3} \pm 2.9$ \\
task5-v0 & $10 \pm 3$ & $54 \pm 22$ & $90 \pm 2$ & $2 \pm 2$ & $1 \pm 1$ & $55 \pm 6$ & $9 \pm 5$ & $76 \pm 8$ & $38 \pm 18$ & $83 \pm 4$ & ${95.0} \pm 1.3$ & ${95.4} \pm 0.9$ & $\mathbf{97.5} \pm 0.8$\\
\midrule

\multicolumn{14}{l}{\textit{antmaze-giant-navigate-singletask-task}} \\
\midrule
task1-v0 & $0 \pm 0$ & $0 \pm 0$ & $27 \pm 22$ & $0 \pm 0$ & $0 \pm 0$ & $0 \pm 0$ & $0 \pm 0$ & $0 \pm 1$ & $0 \pm 0$ & $4 \pm 5$ & $12.6 \pm 15.1$ & $\mathbf{62.4} \pm 1.7$& $6.4 \pm 2.6$ \\
task2-v0 & $0 \pm 0$ & $1 \pm 1$ & $16 \pm 17$ & $0 \pm 0$ & $0 \pm 0$ & $0 \pm 0$ & $0 \pm 0$ & $4 \pm 7$ & $0 \pm 0$ & $9 \pm 7$ & $52.2 \pm 26.5$ & $\mathbf{94.2} \pm 0.4$ & $87.0 \pm 2.6$\\
task3-v0 & $0 \pm 0$ & $0 \pm 0$ & $\mathbf{34} \pm 22$ & $0 \pm 0$ & $0 \pm 0$ & $0 \pm 0$ & $0 \pm 0$ & $0 \pm 0$ & $0 \pm 0$ & $0 \pm 1$ & $13.7 \pm 10.2$ & $28.6 \pm 0.4$ & $3.5 \pm 2.2$\\
task4-v0 & $0 \pm 0$ & $0 \pm 0$ & $5 \pm 12$ & $0 \pm 0$ & $0 \pm 0$ & $0 \pm 0$ & $0 \pm 0$ & $9 \pm 4$ & $0 \pm 0$ & $14 \pm 23$ & $\mathbf{17.8} \pm 19.9$ & $6.4 \pm 8.5$ & $14.3 \pm 15.8$\\
task5-v0 & $1 \pm 1$ & $19 \pm 7$ & $ 49 \pm 22$ & $0 \pm 1$ & $0 \pm 0$ & $0 \pm 0$ & $0 \pm 0$ & $6 \pm 10$ & $13 \pm 9$ & $16 \pm 28$ & $43.2 \pm 35.7$ & $57.3 \pm 15.5$ & $\mathbf{80.3} \pm 12.7$\\
\midrule

\multicolumn{14}{l}{\textit{humanoidmaze-medium-navigate-singletask}} \\
\midrule
task1-v0 & $1 \pm 0$ & $32 \pm 7$ & $16 \pm 9$ & $1 \pm 1$ & $0 \pm 0$ & $38 \pm 19$ & $6 \pm 2$ & $25 \pm 8$ & $69 \pm 19$ & $19 \pm 12$ & $83.5 \pm 3.7$ & $\mathbf{87.2} \pm 2.8$ & $\mathbf{85.1}\pm 6.6$ \\
task2-v0 & $1 \pm 0$ & $41 \pm 9$ & $18 \pm 16$ & $1 \pm 1$ & $1 \pm 1$ & $47 \pm 35$ & $40 \pm 2$ & $76 \pm 10$ & $85 \pm 11$ & $\mathbf{94} \pm 3$ & $ 91.2 \pm 6.3$ & $\mathbf{95.8} \pm 1.6$ & ${91.7} \pm 0.8$\\
task3-v0 & $6 \pm 2$ & $25 \pm 5$ & $36 \pm 13$ & $0 \pm 1$ & $2 \pm 1$ & $ 83 \pm 18$ & $19 \pm 2$ & $27 \pm 11$ & $49 \pm 49$ & $74 \pm 18$ & $ 86.3 \pm 10.7$ & $\mathbf{92.7} \pm 4.3$ & $ 78.4 \pm 19.6$\\
task4-v0 & $0 \pm 0$ & $0 \pm 1$ & $15 \pm 16$ & $1 \pm 1$ & $1 \pm 1$ & $5 \pm 4$ & $1 \pm 1$ & $1 \pm 2$ & $1 \pm 1$ & $3 \pm 4$ & $3.0 \pm 6.6$ & $4.7 \pm 4.9$ & $2.1 \pm 2.5$\\
task5-v0 & $2 \pm 1$ & $66 \pm 4$ & $24 \pm 20$ & $1 \pm 1$ & $3 \pm 3$ & $91 \pm 5$ & $31 \pm 7$ & $63 \pm 9$ & $\mathbf{98} \pm 2$ & $\mathbf{97} \pm 2$ & $\mathbf{95.8} \pm 1.3$ & $\mathbf{96.3} \pm 0.7$ & $94.1 \pm 3.8$ \\
\midrule

\multicolumn{14}{l}{\textit{humanoidmaze-large-navigate-singletask}} \\
\midrule
task1-v0 & $0 \pm 0$ & $3 \pm 1$ & $2 \pm 1$ & $0 \pm 0$ & $0 \pm 0$ & $1 \pm 1$ & $0 \pm 0$ & $0 \pm 1$ & $6 \pm 2$ & $7 \pm 6$ & $57.2 \pm 23.7$ & $\mathbf{68.2}\pm 12.4$ & $39.1 \pm 18.4$\\
task2-v0 & $0 \pm 0$ & $0 \pm 0$ & $0 \pm 0$ & $0 \pm 0$ & $0 \pm 0$ & $0 \pm 0$ & $0 \pm 0$ & $0 \pm 0$ & $0 \pm 0$ & $0 \pm 0$ & $\mathbf{0.1} \pm 0.2$ & $\mathbf{0.1}\pm 0.2$ & $\mathbf{0.8}\pm1.1$\\
task3-v0 & $1 \pm 1$ & $7 \pm 3$ & $8 \pm 4$ & $3 \pm 1$ & $1 \pm 1$ & $2 \pm 3$ & $1 \pm 1$ & $10 \pm 2$ & $\mathbf{48} \pm 10$ & $11 \pm 7$ & $14.6 \pm 16.2$ & $12.7 \pm 15.1$ & $37.8\pm 6.1$\\
task4-v0 & $1 \pm 0$ & $1 \pm 0$ & $1 \pm 1$ & $0 \pm 0$ & $0 \pm 0$ & $0 \pm 1$ & $0 \pm 0$ & $0 \pm 0$ & $1 \pm 1$ & $2 \pm 3$ & $3.7 \pm 4.1$ & $\mathbf{4.7}\pm 4.4$ & $1.7 \pm 1.4$ \\
task5-v0 & $0 \pm 1$ & $1 \pm 1$ & $2 \pm 2$ & $0 \pm 0$ & $0 \pm 0$ & $0 \pm 0$ & $0 \pm 0$ & $1 \pm 1$ & $0 \pm 0$ & $1 \pm 3$ & $\mathbf{13.1} \pm 15.5$ & $\mathbf{12.9} \pm 14.2$ & $4.4 \pm 2.6$\\
\midrule

\multicolumn{14}{l}{\textit{antsoccer-arena-navigate-singletask}} \\
\midrule
task1-v0 & $2 \pm 1$ & $14 \pm 5$ & $0 \pm 0$ & $44 \pm 12$ & $2 \pm 1$ & $1 \pm 3$ & $22 \pm 2$ & $17 \pm 3$ & $61 \pm 25$ & $77 \pm 4$ & $77.0 \pm 1.7$ & $\mathbf{82.3} \pm 2.1$& $52.5 \pm 11.2$\\
task2-v0 & $2 \pm 2$ & $17 \pm 7$ & $0 \pm 1$ & $15 \pm 12$ & $3 \pm 1$ & $0 \pm 0$ & $8 \pm 1$ & $8 \pm 2$ & $75 \pm 3$ & $\mathbf{88} \pm 3$ & $\mathbf{91.2} \pm 2.2$ & $\mathbf{91.3}\pm 2.4$ & $78.7 \pm 2.8$\\
task3-v0 & $0 \pm 0$ & $6 \pm 4$ & $0 \pm 0$ & $0 \pm 0$ & $0 \pm 0$ & $8 \pm 19$ & $11 \pm 5$ & $16 \pm 3$ & $14 \pm 22$ & $\mathbf{61} \pm 6$ & $51.9 \pm 4.9$ & $\mathbf{59.8} \pm 5.8$ & $52.4 \pm 12.7$\\
task4-v0 & $1 \pm 0$ & $3 \pm 2$ & $0 \pm 0$ & $0 \pm 1$ & $0 \pm 0$ & $0 \pm 0$ & $12 \pm 3$ & $24 \pm 4$ & $16 \pm 9$ & $39 \pm 6$ & $40.2 \pm 4.2$ & $\mathbf{44.1}\pm 5.6$ & $22.7 \pm 14.0$\\
task5-v0 & $0 \pm 0$ & $2 \pm 2$ & $0 \pm 0$ & $0 \pm 0$ & $0 \pm 0$ & $0 \pm 0$ & $9 \pm 2$ & $15 \pm 4$ & $0 \pm 1$ & $ 36 \pm 9$ & $29.1 \pm 8.9$ & $\mathbf{41.7}\pm 7.4$ & $33.3 \pm 7.6$\\
\midrule

\multicolumn{14}{l}{\textit{cube-single-play-singletask}} \\
\midrule
task1-v0 & $10 \pm 5$ & $88 \pm 3$ & $89 \pm 5$ & $\mathbf{95} \pm 2$ & $89 \pm 7$ & $77 \pm 28$ & $81 \pm 9$ & $73 \pm 33$ & $79 \pm 4$ & $\mathbf{97} \pm 2$ & $\mathbf{99.1} \pm 0.4$ & $\mathbf{100}\pm 0$ & $\mathbf{100}\pm 0$\\
task2-v0 & $3 \pm 1$ & $85 \pm 8$ & $92 \pm 4$ & $\mathbf{96} \pm 2$ & $82 \pm 16$ & $80 \pm 30$ & $81 \pm 9$ & $83 \pm 13$ & $73 \pm 3$ & $\mathbf{97} \pm 2$ & $\mathbf{99.4} \pm 0.7$ & $\mathbf{100}\pm 0$ & $97.5\pm 1.2$\\
task3-v0 & $9 \pm 3$ & $91 \pm 5$ & $93 \pm 3$ & $\mathbf{99} \pm 1$ & $\mathbf{96} \pm 2$ & $\mathbf{98} \pm 1$ & $87 \pm 4$ & $82 \pm 12$ & $88 \pm 4$ & $\mathbf{98} \pm 2$ & $\mathbf{99.4} \pm 0.5$ & $\mathbf{100}\pm 0$ & $\mathbf{100}\pm 0$\\
task4-v0 & $2 \pm 1$ & $73 \pm 6$ & $92 \pm 3$ & $93 \pm 4$ & $70 \pm 18$ & $91 \pm 2$ & $79 \pm 6$ & $79 \pm 20$ & $79 \pm 6$ & $94 \pm 3$ & $\mathbf{99.1} \pm 0.7$ & $\mathbf{99.1}\pm 0.4$ & $92.7 \pm 2.4$\\
task5-v0 & $3 \pm 3$ & $78 \pm 9$ & $87 \pm 8$ & $90 \pm 6$ & $61 \pm 12$ & $80 \pm 20$ & $78 \pm 10$ & $76 \pm 33$ & $77 \pm 7$ & $\mathbf{93} \pm 3$ & $\mathbf{97.0} \pm 1.6$ & $\mathbf{98.2}\pm 2.4$ & $\mathbf{98.4}\pm 0.2$\\
\midrule

\multicolumn{14}{l}{\textit{cube-double-play-singletask}} \\
\midrule
task1-v0 & $8 \pm 3$ & $27 \pm 5$ & $45 \pm 6$ & $39 \pm 19$ & $7 \pm 6$ & $21 \pm 8$ & $21 \pm 7$ & $47 \pm 11$ & $35 \pm 9$ & $61 \pm 9$ & $\mathbf{76.1} \pm 4.6$ & $72.4\pm 3.1$ & $63.6 \pm 8.1$\\
task2-v0 & $0 \pm 0$ & $1 \pm 1$ & $7 \pm 3$ & $16 \pm 10$ & $0 \pm 0$ & $2 \pm 2$ & $2 \pm 1$ & $22 \pm 12$ & $9 \pm 5$ & $36 \pm 6$ & $\mathbf{53.3} \pm 8.9$ & $48.4 \pm 10.3$ & $17.4 \pm 7.1$ \\
task3-v0 & $0 \pm 0$ & $0 \pm 0$ & $4 \pm 1$ & $17 \pm 8$ & $0 \pm 1$ & $3 \pm 1$ & $1 \pm 1$ & $4 \pm 2$ & $8 \pm 5$ & $22 \pm 5$ & $\mathbf{43.3} \pm 8.9$ & $\mathbf{44.4}\pm 5.1$ & $24.7 \pm 14.4$\\
task4-v0 & $0 \pm 0$ & $0 \pm 0$ & $1 \pm 1$ & $0 \pm 1$ & $0 \pm 0$ & $0 \pm 1$ & $0 \pm 0$ & $0 \pm 1$ & $1 \pm 1$ & $5 \pm 2$ & $\mathbf{7.1} \pm 3.1$ & $\mathbf{7.0}\pm 2.7$ & $6.5 \pm 3.3$\\
task5-v0 & $0 \pm 0$ & $4 \pm 3$ & $4 \pm 2$ & $1 \pm 1$ & $0 \pm 0$ & $3 \pm 2$ & $2 \pm 1$ & $2 \pm 2$ & $17 \pm 6$ & $19 \pm 10$ & $\mathbf{56.3} \pm 11.3$ & $\mathbf{55.7}\pm 10.7$&  $48.9 \pm 9.3$\\
\midrule

\multicolumn{14}{l}{\textit{scene-play-singletask}} \\
\midrule
task1-v0 & $19 \pm 6$ & $94 \pm 3$ & $\mathbf{95} \pm 2$ & $\mathbf{100} \pm 0$ & $94 \pm 4$ & $\mathbf{100} \pm 1$ & $87 \pm 8$ & $\mathbf{96} \pm 8$ & $\mathbf{98} \pm 3$ & $\mathbf{100} \pm 0$ & $\mathbf{99.8} \pm 0.4$ & $\mathbf{100} \pm 0$ & $\mathbf{100}\pm 0$\\
task2-v0 & $1 \pm 1$ & $12 \pm 3$ & $50 \pm 13$ & $33 \pm 14$ & $2 \pm 2$ & $50 \pm 40$ & $18 \pm 8$ & $46 \pm 10$ & $0 \pm 0$ & $76 \pm 9$ & $89.0 \pm 4.1$ & $\mathbf{99.2} \pm 0.6$ & $\mathbf{97.3}\pm 1.7$\\
task3-v0 & $1 \pm 1$ & $32 \pm 7$ & $55 \pm 16$ & $\mathbf{94} \pm 4$ & $4 \pm 4$ & $49 \pm 16$ & $38 \pm 9$ & $78 \pm 14$ & $54 \pm 19$ & $\mathbf{98} \pm 1$ & $78.0 \pm 13.2$ & $75.4 \pm 14.5$ & $81.4\pm 10.4$\\
task4-v0 & $2 \pm 2$ & $0 \pm 1$ & $3 \pm 3$ & $4 \pm 3$ & $0 \pm 0$ & $0 \pm 0$ & $\mathbf{6} \pm 1$ & $4 \pm 4$ & $0 \pm 0$ & $5 \pm 1$ & $0.6 \pm 0.6$ & $2.4 \pm 1.1$ & $0.4\pm 0.3$\\
task5-v0 & $0 \pm 0$ & $0 \pm 0$ & $0 \pm 0$ & $0 \pm 0$ & $0 \pm 0$ & $0 \pm 0$ & $0 \pm 0$ & $0 \pm 0$ & $0 \pm 0$ & $0 \pm 0$ & $0.0 \pm 0.0$ & $0.0\pm 0.0$ & $0.0\pm 0.0$\\
\midrule

\multicolumn{14}{l}{\textit{puzzle-3x3-play-singletask}} \\
\midrule
task1-v0 & $5 \pm 2$ & $33 \pm 6$ & $97 \pm 4$ & $52 \pm 12$ & $89 \pm 5$ & $97 \pm 2$ & $25 \pm 9$ & $63 \pm 19$ & $94 \pm 3$ & $90 \pm 4$ & $94.8 \pm 4.4$ & $\mathbf{99.0} \pm 0.4$ & $\mathbf{97.6}\pm 2.1$\\
task2-v0 & $1 \pm 1$ & $4 \pm 3$ & $1 \pm 1$ & $0 \pm 1$ & $0 \pm 1$ & $0 \pm 0$ & $4 \pm 2$ & $2 \pm 2$ & $1 \pm 2$ & ${16} \pm 5$ & $0.3 \pm 0.3$ & ${7.7}\pm 5.3$ & $\mathbf{20.3}\pm 22.6$\\
task3-v0 & $1 \pm 1$ & $3 \pm 2$ & $3 \pm 1$ & $0 \pm 0$ & $0 \pm 0$ & $0 \pm 0$ & $1 \pm 0$ & $1 \pm 1$ & $0 \pm 0$ & $\mathbf{10} \pm 3$ & $0.9 \pm 0.6$ & $4.7 \pm 4.2$ & $1.0 \pm 0.4$\\
task4-v0 & $1 \pm 1$ & $2 \pm 1$ & $2 \pm 1$ & $0 \pm 0$ & $0 \pm 0$ & $0 \pm 0$ & $1 \pm 1$ & $2 \pm 2$ & $0 \pm 0$ & $ 16  \pm 5$ & $5.4 \pm 2.1$ & $16.2\pm 4.3$ & $\mathbf{34.1}\pm 30.4$\\
task5-v0 & $1 \pm 0$ & $3 \pm 2$ & $5 \pm 3$ & $0 \pm 0$ & $0 \pm 0$ & $0 \pm 0$ & $1 \pm 1$ & $2 \pm 2$ & $0 \pm 0$ & ${16} \pm 3$ & $14.1 \pm 8.4$ & $15.7 \pm 4.6$ & $\mathbf{20.7} \pm 2.6$\\
\midrule

\multicolumn{14}{l}{\textit{puzzle-4x4-play-singletask}} \\
\midrule
task1-v0 & $1 \pm 1$ & $12 \pm 2$ & $26 \pm 4$ & $48 \pm 5$ & $24 \pm 9$ & $44 \pm 10$ & $1 \pm 2$ & $32 \pm 9$ & $49 \pm 9$ & $34 \pm 8$ & $50.0 \pm 8.1$ & $\mathbf{58.7} \pm 5.6$ & $35.4 \pm 12.7$\\
task2-v0 & $0 \pm 0$ & $7 \pm 4$ & $12 \pm 4$ & $14 \pm 5$ & $0 \pm 1$ & $0 \pm 0$ & $0 \pm 1$ & $5 \pm 3$ & $4 \pm 4$ & $\mathbf{16} \pm 5$ & $9.9 \pm 2.5$ & $\mathbf{15.6} \pm 1.7$ & $6.7 \pm 3.4$\\
task3-v0 & $0 \pm 0$ & $9 \pm 3$ & $15 \pm 3$ & $34 \pm 5$ & $21 \pm 10$ & $29 \pm 12$ & $1 \pm 1$ & $20 \pm 10$ & $\mathbf{50} \pm 14$ & $18 \pm 5$ & $46.2 \pm 3.8$ & $\mathbf{49.8} \pm 2.9$ & $24.4 \pm 9.8$\\
task4-v0 & $0 \pm 0$ & $5 \pm 2$ & $10 \pm 3$ & $\mathbf{26} \pm 6$ & $7 \pm 4$ & $1 \pm 1$ & $0 \pm 0$ & $5 \pm 1$ & $21 \pm 11$ & $11 \pm 3$ & $17.2 \pm 2.5$ & $20.4 \pm 3.4$  & $\mathbf{25.7} \pm 2.1$ \\
task5-v0 & $0 \pm 0$ & $4 \pm 1$ & $7 \pm 3$ & $\mathbf{24} \pm 11$ & $1 \pm 1$ & $0 \pm 0$ & $0 \pm 1$ & $4 \pm 3$ & $2 \pm 2$ & $7 \pm 3$ & $7.3 \pm 3.6$ & $8.6 \pm 4.1$ & $6.7 \pm 3.4$ \\
\midrule

\textbf{Average} & 2.8 & 23.4 & 31.0 & 22.8 & 14.3 & 25.3 & 16.1 & 28.6 & 30.0 & 43.5 & 51.8 & \textbf{56.4} & 50.8 \\
\textbf{(50 tasks)} & & & & & & & & & & & & & \\
\bottomrule
\end{tabular}
}
\end{table}

\textbf{Results and Analysis.} As shown in Table~\ref{tab:ogbench_full}, our proposed VRD framework demonstrates superior performance across the vast majority of the 50 tasks. Specifically:
\begin{itemize}
    \item \textbf{Overall Superiority}: VRD-Q achieves the highest overall average score of \textbf{56.4}, significantly outperforming the strongest generative baseline, GFP actor (51.8), and the state-of-the-art FQL (43.5). Besides, VRD-R achieves competitive average score of 50.8.
    \item \textbf{Robustness in Challenging Domains}: In highly complex environments like \textit{antmaze-giant} and \textit{humanoidmaze}, where conventional baselines such as IQL, ReBRAC, and IDQL often struggle or yield near-zero scores, VRD-Q and VRD-R maintain high success rates (e.g., 94.2\% in \textit{antmaze-giant} task2). This highlights the effectiveness of our value-rectification mechanism in navigating sparse-reward, high-dimensional spaces.
    \item \textbf{High-Precision Manipulation}: In manipulation tasks like \textit{cube-single-play} and \textit{scene-play}, VRD-Q reaches a 100\% success rate over some tasks, demonstrating its capability to precisely align generative trajectories with optimal expert behaviors.
    \item \textbf{Variant Comparison}: While VRD-Q is generally more stable and leads in average performance, VRD-R shows competitive advantages in specific structured tasks, such as \textit{puzzle-3x3} and \textit{antsoccer}, suggesting that geometric latent reshaping is particularly effective for certain task dynamics.
\end{itemize}

In conclusion, these results validate that VRD effectively addresses the latent fusion problem, providing a more robust and high-performing solution for goal-conditioned offline reinforcement learning.

\subsection{Minari: External comparison study}\label{appendix:minari}

The MuJoCo task suite within the Minari dataset~\citep{minari} serves as a standardized benchmark widely adopted in offline reinforcement learning (RL) research. Developed by the Farama Foundation as a modern successor to the D4RL dataset, these tasks are built upon classic physics simulation environments, including \textit{HalfCheetah}, \textit{Hopper}, and \textit{Walker2d}. Minari provides a tiered classification system to evaluate algorithms across varying data qualities:

\begin{itemize}
    \item \textbf{Simple Datasets}: These are primarily generated by random or early-stage training policies. They contain a high volume of low-value, exploratory trajectories characterized by a divergent distribution and a low signal-to-noise ratio, testing an algorithm's ability to extract fundamental control logic from poor-quality data.
    \item \textbf{Medium Datasets}: Collected from policies at intermediate performance levels, these datasets are goal-oriented but frequently contain sub-optimal actions and failed attempts. They serve as a core benchmark for evaluating whether an algorithm can effectively isolate optimal modes and outperform the data average.
    \item \textbf{Expert Datasets}: Generated by fully converged, high-performance policies, these represent optimal behavior with a narrow, high-quality distribution. They are designed to verify if an algorithm can achieve high-precision policy alignment and maintain expert-level performance under limited data diversity.
\end{itemize}

By adhering to the Gymnasium API and providing comprehensive metadata, Minari serves as an ideal testbed for evaluating capacity for mode separation and distributional correction within complex, multi-modal, and non-expert datasets.

Flow Q-Learning (FQL)~\citep{park2025flow} and Guided Flow Policy (GFP)~\citep{tiofack2025guided} are selected as the main baselines for comparison study.

\begin{table}[h]
\centering
\caption{Comparisons on Minari MuJoCo benchmarks. The best scores are highlighted in bold.}
\label{tab:mujoco_results}
\resizebox{\textwidth}{!}{
\begin{tabular}{l|c|c|c||c|c}
\toprule
\multirow{2}{*}{\textbf{Task}} & \multicolumn{4}{c}{\textbf{Offline RL algorithms}} \\
\cmidrule{2-6}
 & FQL & \textbf{GFP} actor $\pi_\theta$ & \textbf{GFP} VaBC $\pi_\omega$ & \textbf{VRD-Q (Ours)} & \textbf{VRD-R (Ours)} \\
\midrule
Minari halfcheetah-simple-v0 & $59.2 \pm 0.3$ & $\textbf{72.5} \pm 0.5$ & $64.4 \pm 0.4$& $71.4 \pm 0.1$  & $70.3 \pm 0.2$ \\
Minari halfcheetah-medium-v0 & $100.2 \pm 6.7$ & $121.3 \pm 5.8$ & $108.6 \pm 5.3$& $ 117.4 \pm 3.1$  &  $\textbf{123.9}\pm 2.3$  \\
Minari halfcheetah-expert-v0 & $134.1 \pm 2.3$ & $133.4 \pm 1.3$ & $136.4 \pm 0.8$& $96.7 \pm 10.7$  & $\textbf{139.0}\pm 0.1$ \\
\midrule
Minari hopper-simple-v0 & $57.2 \pm 5.9$ & $91.6 \pm 4.3$ & $87.4 \pm 6.6$& $\textbf{99.1} \pm 0.1$  &  $\textbf{99.2} \pm 0.6$\\
Minari hopper-medium-v0 & $81.9 \pm 23.9$ & $79.6 \pm 14.5$ & $78.2 \pm 24.2$ & $\textbf{107.3} \pm 4.4$ & $104.3\pm 8.6$ \\
Minari hopper-expert-v0 & $99.6 \pm 10.9$ & $103.9 \pm 10.6$ & $\textbf{108.8} \pm 13.1$ & $107.3 \pm 9.1$ & $101.4\pm 10.2$ \\
\midrule
Minari walker2d-simple-v0 & $89.4 \pm 0.7$ & $90.0 \pm 0.9$ & $89.7 \pm 0.9$ & $\textbf{91.5} \pm 0.2$ & $\textbf{91.5}\pm 0.7$ \\
Minari walker2d-medium-v0 & $127.6 \pm 3.0$ & $133.7 \pm 1.1$ & $126.1 \pm 3.5$ & $\textbf{135.6} \pm 0.3$ & $134.0\pm 1.7$ \\
Minari walker2d-expert-v0 & $148.0 \pm 1.8$ & $149.9 \pm 2.0$ & $150.8 \pm 0.4$ & $\textbf{152.3} \pm 0.2$ & $151.4\pm 0.1$ \\
\midrule
\textbf{Average} & 99.7 & 108.4 & 105.6 & 108.7 & \textbf{112.8} \\
\bottomrule
\end{tabular}
}
\end{table}

\textbf{Results Overview.} Table~\ref{tab:mujoco_results} summarizes the performance of our proposed VRD framework against several state-of-the-art generative baselines across various task levels in the Minari benchmark. Overall, VRD-R achieves the highest average score of \textbf{112.8}, outperforming both the flow-based FQL (99.7) and the strongest GFP variants (108.4).

\textbf{Key Observations.}
\begin{itemize}
    \item \textbf{Consistent Superiority in Complex Environments}: In the \textit{Walker2d} and \textit{Hopper} domains, VRD-Q and VRD-R consistently outperform the baselines. Notably, in \textit{hopper-simple}, both VRD variants achieve a score of approximately \textbf{99.1}, nearly doubling the performance of FQL (57.2), which demonstrates the robustness of our framework in extracting optimal policies from low-quality data.
    \item \textbf{High-Value Mode Recovery}: VRD-R exhibits exceptional performance in \textit{halfcheetah-expert}, reaching a peak score of \textbf{139.0}. This suggests that the geometric latent reshaping in VRD-R is highly effective at aligning generative trajectories with expert behaviors in high-dimensional action spaces.
    \item \textbf{Mitigation of Mode Averaging}: The significant performance gains observed in \textit{Simple} and \textit{Medium} datasets—which typically contain a higher proportion of sub-optimal trajectories—validate that our value-rectification mechanism successfully mitigates the "latent fusion" problem. By decoupling optimal modes from sub-optimal noise, VRD maintains high performance where standard generative models often suffer from catastrophic mode averaging.
\end{itemize}

In conclusion, the results across Minari tasks confirm that the VRD framework is not only competitive in expert-level tasks but also exceptionally robust in handling non-expert, multi-modal data.

\subsection{Full sensitivity analysis for hyperparameters}\label{appendix:ablation_study}

To provide a comprehensive understanding of the stability and robustness of the Value-Rectified Distillation (VRD) framework, we conduct an extensive grid search on two primary hyperparameters: the quantile threshold $\tau$ for VRD-Q and the truncation radius $\sigma$ for VRD-R, alongside the balance coefficient $\lambda$. The results are visualized in Figure \ref{fig:vrd_q_full_sensi} and Figure \ref{fig:vrd_r_full_sensi}.

\begin{figure}[h]
    \centering
    \includegraphics[width=1\linewidth]{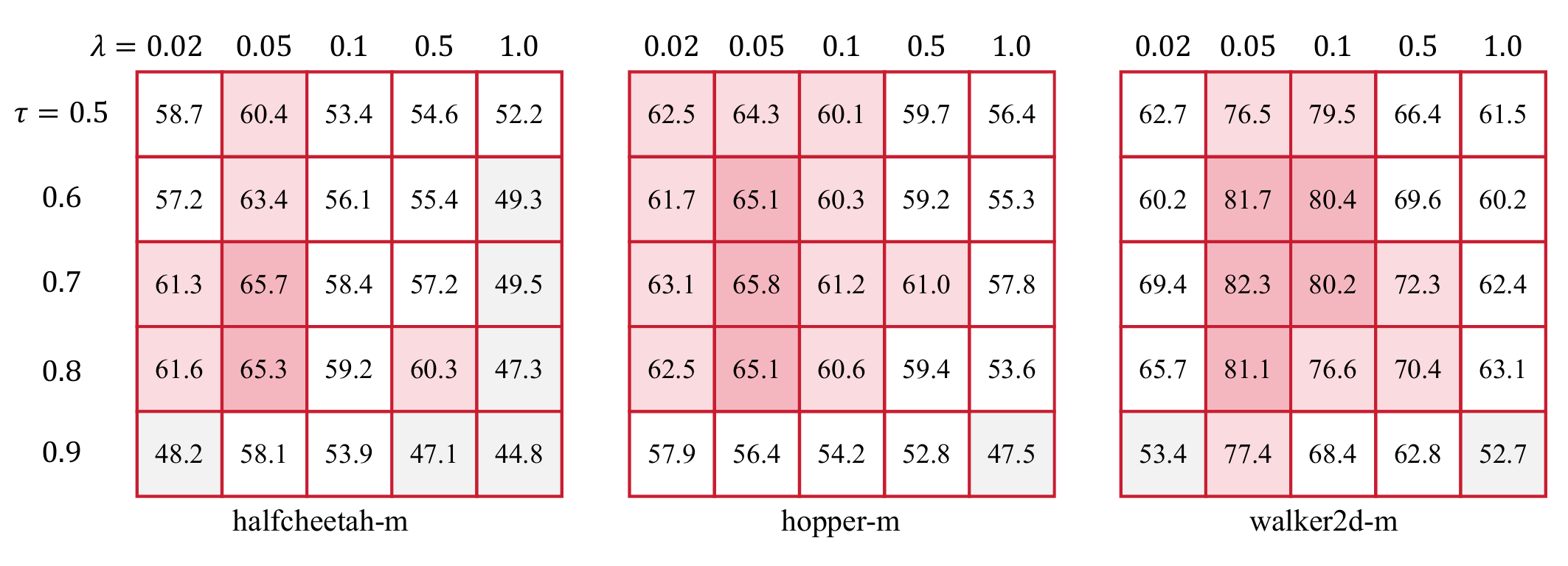}
    \caption{Sensitivity analysis for VRD-Q.}
    \label{fig:vrd_q_full_sensi}
\end{figure}

\subsubsection{Analysis of VRD-Q Sensitivity (Figure \ref{fig:vrd_q_full_sensi})}
The experimental results on D4RL `medium` datasets (HalfCheetah, Hopper, and Walker2d) reveal several key insights for VRD-Q:
\begin{itemize}
    \item \textbf{Inverted U-shape performance w.r.t. $\tau$:} The performance generally follows an inverted U-shaped curve across all environments. When $\tau$ is too small (e.g., $0.5$), the statistical filtering is insufficient to suppress sub-optimal behavioral modes. Conversely, an excessively large $\tau$ (e.g., $0.9$) overly restricts the distillation weight, leading to a loss of distribution coverage and potential vanishing gradients.
    \item \textbf{Optimal $\tau$ Region:} The peak performance consistently resides in the range of $\tau \in [0.7, 0.8]$. This range effectively balances the exclusion of low-value "OOD noise" with the preservation of high-value multi-modal trajectories.
    \item \textbf{Robustness to $\lambda$:} VRD-Q demonstrates relative stability across a wide range of $\lambda$ (from $0.02$ to $0.1$). However, very large values of $\lambda$ (e.g., $1.0$) tend to degrade performance as the rectification term begins to dominate the Q-value maximization objective.
\end{itemize}

\begin{figure}[h]
    \centering
    \includegraphics[width=1\linewidth]{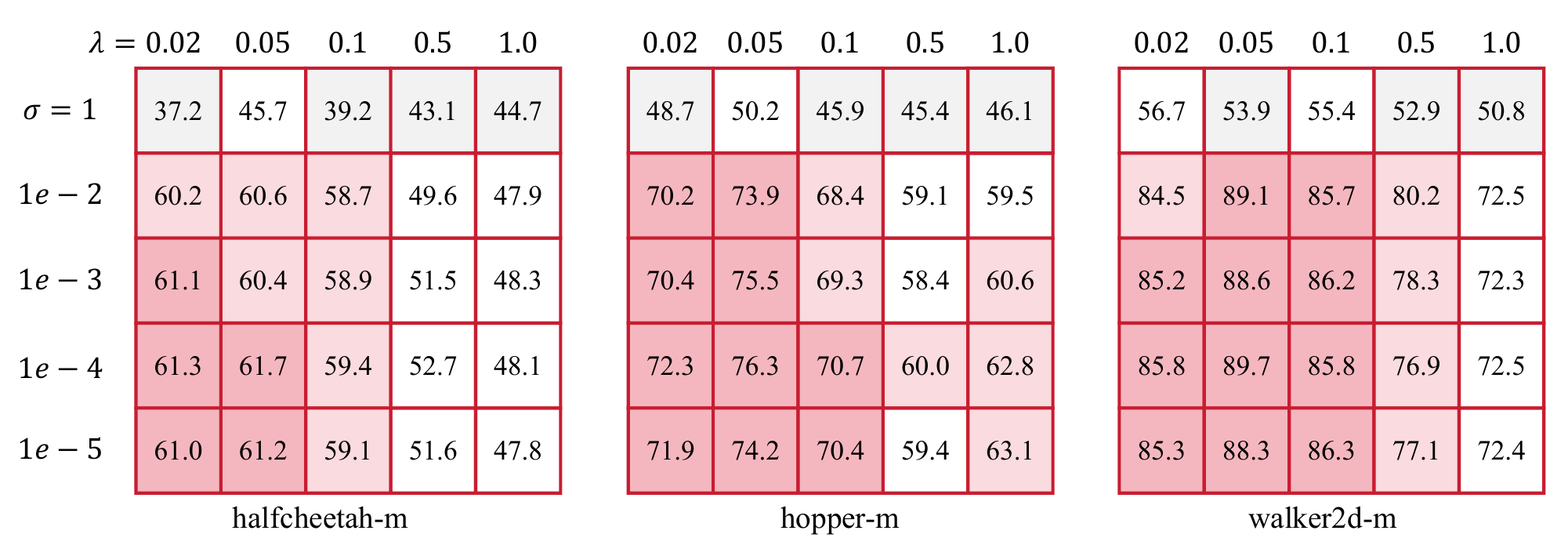}
    \caption{Sensitivity analysis for VRD-R.}
    \label{fig:vrd_r_full_sensi}
\end{figure}

\subsubsection{Analysis of VRD-R Sensitivity (Figure \ref{fig:vrd_r_full_sensi})}
For the geometric rectification variant VRD-R, the sensitivity analysis highlights:
\begin{itemize}
    \item \textbf{Criticality of Truncation Radius $\sigma$:} Performance is highly sensitive to the scale of the latent prior during inference. When $\sigma=1$ (equivalent to standard Gaussian sampling), the performance is sub-optimal because the sampler hits the peripheral sub-optimal regions. 
    \item \textbf{Optimal $\sigma$ Scale:} As $\sigma$ decreases towards $10^{-4}$, the performance significantly improves and stabilizes. This confirms that the value-guided gravity field successfully compresses high-value actions toward the latent origin, and a smaller $\sigma$ ensures the sampler stays within this "expert core."
    \item \textbf{Environment Consistency:} Similar to VRD-Q, the best configuration for VRD-R is relatively consistent across different locomotion tasks, with $\lambda=0.05$ and $\sigma=10^{-4}$ providing the most robust results.
\end{itemize}

\subsubsection{Practical Hyperparameter Tuning Suggestions}
Based on our empirical analysis, we offer the following guidelines for applying VRD to new offline RL tasks:

\begin{enumerate}
    \item \textbf{Tuning the Quantile Threshold $\tau$ (VRD-Q):}
    \begin{itemize}
        \item \textbf{Default:} Start with $\tau = 0.7$.
        \item \textbf{High-Noise / Exploratory Datasets:} While Theorem \ref{thm:optimalsupport} suggests a larger $\tau$ under an oracle critic, on noisy datasets (e.g., ``random'' or ``medium-replay'') the learned Q-network inherently suffers from higher estimation variance. Setting $\tau$ too close to $1.0$ makes the suboptimal anchor weight $1-\tau$ vanishingly small, causing the actor to exploit spurious OOD Q-peaks and destabilizing training. Thus, we recommend decreasing $\tau$ slightly (e.g., to $0.6 \sim 0.7$) to retain a necessary behavioral regularization margin ($1-\tau \in [0.3, 0.4]$) for training stability while still down-weighting suboptimal modes.
        \item \textbf{Expert / Structured Datasets:} For datasets with high-quality demonstrations where value estimation is accurate and stable, a larger $\tau$ (e.g., $0.8 \sim 0.9$) can be safely employed to aggressively concentrate the policy onto the peak optimal trajectories without risking OOD divergence.
    \end{itemize}
    
    \item \textbf{Tuning the Truncation Radius $\sigma$ (VRD-R):}
    \begin{itemize}
        \item \textbf{Default:} Set $\sigma = 10^{-4}$.
        \item \textbf{Geometric Adjustment:} If the agent exhibits "mode collapse" (too little variety in actions), increase $\sigma$ to $10^{-3}$. If the agent still chooses sub-optimal actions in multi-modal settings, decrease $\sigma$ to $10^{-5}$ to force more aggressive centering.
    \end{itemize}
    
    \item \textbf{Tuning the Balance Coefficient $\lambda$:}
    \begin{itemize}
        \item We recommend $\lambda \in [0.05, 0.1]$. A smaller $\lambda$ (e.g., $0.02$) is preferred when the Q-function is well-trained and accurate. If the environment is complex and requires stronger regularization to prevent OOD mode-averaging, increase $\lambda$ to $0.5$.
    \end{itemize}
\end{enumerate}

\subsection{Validation study}\label{appendix:validation}

\subsubsection{Quantile Distillation in VRD-Q}

To qualitatively investigate the efficacy of the proposed Value-Rectified Distillation (VRD-Q) in isolating high-value behaviors, we conduct a visualization analysis of the latent space. Specifically, we apply t-SNE to reduce the dimensionality of the standard prior noise vectors $z$, which are sampled during policy inference, and project them into a 2D plane. We evaluate three representative tasks from the D4RL benchmark: \texttt{halfcheetah-medium-v2}, \texttt{hopper-medium-v2}, and \texttt{walker2d-medium-v2}. In the visualization, each point represents a sampled latent code $z$, color-coded to indicate whether the action generated by the VRD-Q Actor or the baseline Flow Model policy yields a higher estimated Q-value.

\begin{figure}[H]
    \centering
    \includegraphics[width=1\linewidth]{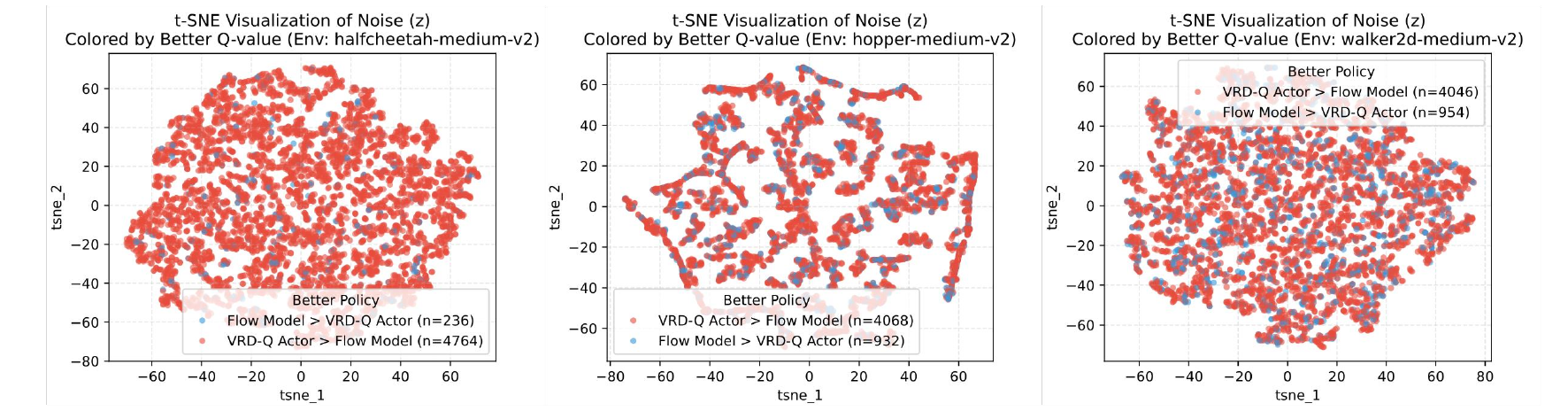}
    \caption{Validation study for Quantile Distillation in VRD-Q.}
    \label{fig:validationq}
\end{figure}

The experimental results, as illustrated in Figure \ref{fig:validationq}, reveal that the VRD-Q Actor consistently generates actions with superior Q-values across the vast majority of the latent space compared to the original Flow Model. The quantitative distribution of "better" samples is summarized as follows:
\begin{itemize}
    \item \textbf{Halfcheetah-medium-v2}: The VRD-Q Actor produces higher-value actions for $4,764$ samples, while the Flow Model only excels in $236$ samples.
    \item \textbf{Hopper-medium-v2}: The VRD-Q Actor outperforms the baseline in $4,068$ samples, significantly exceeding the $932$ samples of the Flow Model.
    \item \textbf{Walker2d-medium-v2}: Higher Q-values are achieved by the VRD-Q Actor in $4,046$ samples, compared to $954$ for the Flow Model.
\end{itemize}

This visualization provides strong empirical evidence for the rationality of the VRD-Q framework. As theoretically analyzed, standard Flow Matching policies often suffer from \textit{Latent Fusion} in multi-modal offline datasets, where optimal and sub-optimal generative trajectories become topologically entangled. Direct Q-guided distillation without rectification risks \textit{Out-of-Distribution (OOD) Mode Averaging}, leading to policy collapse into low-value regions. In contrast, VRD-Q utilizes a \textit{quantile-aware weight modulation} mechanism to statistically filter out sub-optimal behaviors while maintaining the structural integrity of the latent space. The fact that the majority of the prior noise distribution is mapped to higher-value actions demonstrates that VRD-Q effectively decouples behavioral modes and successfully rectifies the flow dynamics toward the high-value action manifold, thereby ensuring robust and high-performance policy execution.

\subsection{Reshaping Distillation in VRD-R}

To investigate the effectiveness of the proposed VRD-R in achieving geometric separation of behavioral modes, we conduct a sensitivity analysis of Q-values relative to the latent space scale. 

\textbf{Experimental Setup.} We evaluate three representative tasks from the D4RL benchmark: \texttt{halfcheetah-medium-v2}, \texttt{hopper-medium-v2}, and \texttt{walker2d-medium-v2}. The horizontal axis represents the \textit{Latent Radius} $\|z\|$, which denotes the magnitude of the standard prior noise sampled during the inference process, ranging from $0.0$ to $1.0$. The vertical axis measures the \textit{Normalized $\Delta Q$}, representing the normalized change in the estimated Q-values of actions generated at the corresponding latent radius. This setup allows us to track the topological distribution of action values across the latent manifold.

\begin{figure}[H]
    \centering
    \includegraphics[width=1\linewidth]{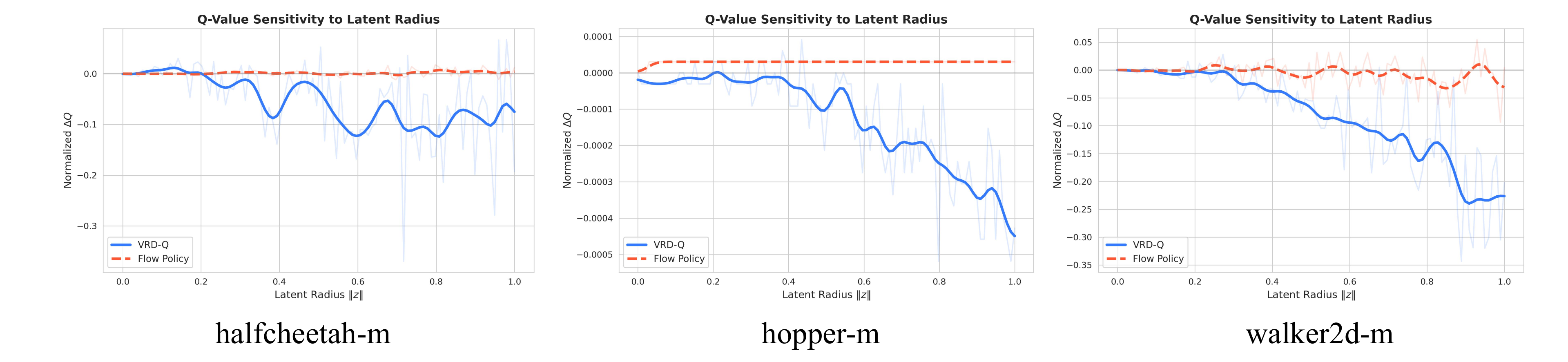}
    \caption{Validation for VRD-R: Q-value sensitivity to latent radius}
    \label{fig:validationr}
\end{figure}

\textbf{Results and Rationality Analysis.} As illustrated in Fig.~\ref{fig:validationr}, the experimental results provide strong empirical evidence for the geometric reorganization capabilities of VRD-R:
\begin{itemize}
    \item \textbf{Intentional Value Decay Gradient:} Across all tested environments, the VRD-R curves exhibit a significant and consistent downward trend as the latent radius $\|z\|$ increases. This sharp decline in Normalized $\Delta Q$ is a key indicator of successful latent reshaping.
    \item \textbf{Spatial Isolation of Modes:} The observed gradient confirms that VRD-R has successfully decoupled the intertwined modes found in the original multi-modal datasets. Specifically, it forces the \textit{optimal modes} (high-value behaviors) to consolidate within the central region of the latent space (near the origin, $\|z\| \approx 0$), while simultaneously pushing \textit{sub-optimal modes} (low-value behaviors) toward the peripheral, high-variance regions (larger $\|z\|$).
    \item \textbf{Strategic Alignment with Gaussian Prior:} Given that the standard Gaussian prior utilized during inference naturally assigns the highest probability density to the region around the origin, this "center-optimal, edge-suboptimal" topology is highly rational. It ensures that standard sampling is highly likely to hit the high-value action manifold.
\end{itemize}

In conclusion, the sensitivity results demonstrate that VRD-R effectively rectifies the flow dynamics by geometrically isolating high-value trajectories. This mechanism fundamentally prevents the \textit{Mode Averaging} issue by ensuring that sub-optimal behaviors are spatially excluded from the primary inference region, thereby significantly enhancing the policy's performance and robustness.

\subsection{Hyperparameters}
\label{sec:hyperparameters}

In this section, we detail the hyperparameters used for evaluating our proposed Value-Rectified Distillation (VRD) framework, including both the VRD-Q and VRD-R instantiations. To ensure a fair and rigorous comparison, the foundational hyperparameters (e.g., network architectures, learning rates, and flow matching configurations) are kept consistently identical across all evaluated benchmarks, namely D4RL, OGBench, and Minari. Task-specific tuning is strictly limited to the core hyperparameters of the VRD framework ($\tau$ for VRD-Q and $\sigma$ for VRD-R) to adapt to the varying degrees of multi-modality, dataset quality, and environment complexity.

\paragraph{General Network and Training Configurations}
Both the value networks (critics) and the flow-based policy networks (actors) are parameterized using Multi-Layer Perceptrons (MLPs). For the flow matching process, we maintain a fixed number of integration steps (NFE = 10) using the Euler ODE solver. The general architectural choices and optimization hyperparameters are summarized in Table~\ref{tab:general_hyperparameters}.

\paragraph{VRD-Specific Hyperparameters and Task Dependency}
The two algorithmic variants of our framework require specific hyperparameter configurations to effectively decouple sub-optimal and optimal behavioral modes. Empirically, we observe a strong dependency between the dataset profile (i.e., multi-modality and sub-optimality) and the optimal hyperparameter selection:

\begin{itemize}
    \item \textbf{VRD-Q (Quantile-based Filtering):} The primary hyperparameter is the quantile threshold $\tau \in (0.5, 1)$, which assigns a smaller distillation weight $1-\tau$ to suboptimal trajectories and a larger weight $\tau$ to optimal ones. For highly complex, exploratory, or noisy multi-modal datasets (e.g., OGBench mixed datasets and D4RL antmaze) where Q-value estimation is noisy, a moderate threshold ($\tau \in \{0.6, 0.7\}$) is strictly required to balance suboptimal mode suppression with sufficient behavioral anchoring ($1-\tau \in \{0.3, 0.4\}$), preventing policy collapse from Q-overestimation. Conversely, for datasets with cleaner distributions or expert-level demonstrations (e.g., D4RL expert or high-quality Minari tasks) with reliable critic gradients, a higher threshold ($\tau \in \{0.8, 0.9\}$) is preferred to maximize peak mode extraction efficiency.
    
    \item \textbf{VRD-R (Latent Reorganization):} The geometric isolation of high-value generative trajectories in the latent space is governed by the margin parameter $\sigma$, searched over a logarithmic grid $\sigma \in \{10^{-2}, 10^{-3}, 10^{-4}, 10^{-5}\}$. In tasks exhibiting severe Out-of-Distribution (OOD) mode averaging risks, such as OGBench humanoid control, a relatively larger margin ($\sigma = 10^{-2}$) is enforced to guarantee a wider geometric decoupling. For simpler, structured tasks (e.g., D4RL locomotion expert datasets), a strictly smaller margin ($\sigma = 10^{-5}$) prevents over-regularization and maintains the natural topological structure of the generative manifold.
\end{itemize}

A detailed breakdown of the exact parameters for VRD-Q and VRD-R across specific benchmark tasks is provided in Table~\ref{tab:vrd_task_hyperparameters}.

\begin{table}[htbp]
    \centering
    \caption{General network architectures and optimization hyperparameters shared across all benchmarks.}
    \label{tab:general_hyperparameters}
    \resizebox{0.6\columnwidth}{!}{%
    \begin{tabular}{llc}
        \toprule
        \textbf{Category} & \textbf{Hyperparameter} & \textbf{Value} \\
        \midrule
        \multirow{3}{*}{Architecture} & Actor / Critic Hidden Layers & $[256, 256, 256]$ \\
        & Activation Function & ReLU / SiLU \\
        & Time Embedding Dimension & $64$ \\
        \midrule
        \multirow{5}{*}{Optimization} & Actor / Critic Learning Rate & $3 \times 10^{-4}$ \\
        & Batch Size & $256$ \\
        & Discount Factor ($\gamma$) & $0.99$ \\
        & Target Soft-Update Rate ($\rho$) & $0.005$ \\
        \midrule
        \multirow{2}{*}{Flow Matching} & Integration Steps (NFE) / Solver & $10$ / Euler \\
        & Noise Schedule & Linear \\
        \bottomrule
    \end{tabular}%
    }
\end{table}

\begin{table}[htbp]
    \centering
    \caption{Task-specific hyperparameter selections for VRD-Q ($\tau$) and VRD-R ($\sigma$). Highly multi-modal tasks with sparse expert data (e.g., OGBench and AntMaze) utilize a smaller $\tau$ and a larger $\sigma$ (e.g., $10^{-2}$), whereas near-expert, structured datasets utilize a larger $\tau$ and a strictly smaller $\sigma$ (e.g., $10^{-5}$).}
    \label{tab:vrd_task_hyperparameters}
    \resizebox{\columnwidth}{!}{%
    \begin{tabular}{llcc}
        \toprule
        \textbf{Benchmark} & \textbf{Specific Datasets / Environments} & \textbf{VRD-Q Threshold ($\tau$)} & \textbf{VRD-R Margin ($\sigma$)} \\
        \midrule
        \multirow{2}{*}{\textbf{OGBench}} 
        & \texttt{humanoid-medium}, \texttt{ant-large}, \texttt{cube-single}, \texttt{puzzle} & $0.7$ & $10^{-3}$ \\
        & \texttt{humanoid-large}, \texttt{ant-giant}, \texttt{cube-double}, \texttt{scene} & $0.6$ & $10^{-2}$ \\
        \midrule
        \multirow{3}{*}{\textbf{D4RL}} 
        & \texttt{antmaze-umaze/medium/large} & $0.6 \sim 0.7$ & $10^{-2} \sim 10^{-3}$ \\
        & \texttt{halfcheetah/hopper/walker2d-medium/medium-replay} & $0.7 \sim 0.8$ & $10^{-4}$ \\
        & \texttt{halfcheetah/hopper/walker2d-expert/medium-expert} & $0.8 \sim 0.9$ & $10^{-5}$ \\
        \midrule
        \multirow{2}{*}{\textbf{Minari}} 
        & Mixed / Sub-optimal datasets (e.g., \texttt{pointmaze-medium-v1}) & $0.7 \sim 0.8$ & $10^{-4}$ \\
        & Expert datasets (e.g., \texttt{pointmaze-expert-v1}, \texttt{door-human}) & $0.8 \sim 0.9$ & $10^{-5}$ \\
        \bottomrule
    \end{tabular}%
    }
\end{table}

\subsection{Network structures}

The architectural framework incorporates the latent variable $z$ across different policy paradigms to enable stochasticity and generative modeling:

\textbf{One-Step Policy (Actor):} Following FQL~\cite{park2025flow}, the latent variable $z \sim \mathcal{N}(0, I)$ (or $z' \sim \mathcal{P}_{\text{test}}$) is concatenated directly with the state observation $s$ as an explicit input feature to the primary MLP. This allows the feed-forward network $\pi(s, z)$ to nonlinearly warp the latent prior into complex, multi-modal action distributions in a single forward pass without iterative ODE integration.

\textbf{Flow Policy (ActorVectorField):} The latent variable $z_t$ is treated as an explicit input feature alongside observations $s$ and time $t$. Because the network must model the local velocity of a continuous transformation, $z_t$ represents the intermediate state of the action as it "flows" from the base distribution to the target manifold. The MLP computes the vector field $\dot{z}_t = v_\theta(s, z_t, t)$ based on this concatenated representation.

\textbf{Critic Network (Value):} The critic acts as a deterministic evaluator for a given state-action pair $(s, a)$. While it does not directly process the latent variable $z$, it estimates the expected return of the final action $a$ (the endpoint of the flow) to guide the policy's optimization.

\begin{table}[htbp]
\centering
\caption{Refined Architectural Specifications Including Latent Variable Roles}
\label{tab:flow_q_comprehensive}
\resizebox{\linewidth}{!}{
\begin{tabular}{@{}lp{3.5cm}p{4cm}p{3.5cm}@{}}
\toprule
\textbf{Specification} & \textbf{One-Step Policy} (\texttt{Actor}) & \textbf{Flow Policy} (\texttt{ActorVectorField}) & \textbf{Critic} (\texttt{Value}) \\ \midrule
\textbf{Primary Inputs} & $s \oplus z$ (Obs, Latent z) & $s \oplus z_t \oplus t$ \newline (Obs, Latent $z$, Time $t$) & $s \oplus a$ \newline (Obs, Actions) \\ \midrule
\textbf{Latent Variable ($z$) Role} & \textbf{Input Feature:} Nonlinearly mapped via MLP to generate one-step action $a = \pi(s, z)$ & \textbf{Input Feature}: Actively processed to compute flow velocity ($\dot{z}_t$) & \textbf{N/A}: Evaluates the final outcome action $a$ \\ \midrule
\textbf{Backbone} & MLP \newline (Optional Layer Norm) & Time-conditioned MLP \newline (Sinusoidal Embedding) & Ensemblized MLP \newline ($N \ge 1$ via \texttt{vmap}) \\ \midrule
\textbf{Output Type} & Deterministic Action $a$ (conditioned on $s$ and $z$) & Vector Field \newline $v_\theta(s, z_t, t)$ & Scalar Value \newline $Q(s, a)$ or $V(s)$ \\ \bottomrule
\end{tabular}}
\end{table}

\subsection{Computational Resources}
\label{subsec:computational_resources}

All experiments, including the training of the proposed VRD frameworks and the reproduction of baseline methods, were conducted on a high-performance computing workstation to ensure reproducibility and efficiency. Specifically, the deep neural networks and flow-matching policies were accelerated using NVIDIA GeForce RTX 3090 GPUs (24GB VRAM). The algorithms were implemented in Python utilizing the PyTorch/JAX deep learning frameworks. Furthermore, the environment simulations and data sampling for the D4RL and OGBench benchmarks were processed on a Intel Xeon W-2295 CPU equipped with 128GB of RAM.

\section{Discussion}\label{appendix:discussion}

While the proposed Value-Rectified Distillation (VRD) framework, including its two instantiations (VRD-Q and VRD-R), successfully addresses the fundamental \textit{Latent Fusion} and Out-of-Distribution (OOD) Mode Averaging problems in flow-based offline RL, there are several limitations and open challenges that warrant further investigation.

\paragraph{1. Reliance on Q-value Accuracy.} 
The core mechanism of VRD inherently relies on the accuracy of the underlying value network to guide the trajectory alignment. Specifically, the quantile weight $\omega_\tau(z)$ in VRD-Q and the geometric reshaping function $\Phi_Q(s, z)$ in VRD-R are directly derived from the estimated Q-values. If the Q-network suffers from severe OOD overestimation—a notoriously common issue in offline RL—the VRD framework might inadvertently align the generative flow toward falsely optimistic, low-density regions. Although our empirical results demonstrate robustness, future work could explore integrating more rigorous conservative value estimation methods (e.g., CQL-style penalties or uncertainty-aware ensembles) directly into the value-rectification functional to mitigate this risk.

\paragraph{2. Restriction to Continuous Action Spaces.} 
The theoretical foundation of VRD (as proven in Theorem \ref{thm:ood_avg} and \ref{thm:optimalsupport}) and the Flow Matching parameterization are built strictly upon continuous vector fields and continuous topological mappings. Consequently, the current framework is restricted to continuous control tasks. Extending this value-guided flow rectification paradigm to discrete action spaces (e.g., Atari games or discrete routing problems) is non-trivial. Future research could explore incorporating Categorical Flow Matching or Discrete Diffusion processes, adapting the latent distribution alignment functional to discrete categorical probabilities.

\paragraph{3. Potential Diversity Collapse in VRD-R for Multi-Modal Experts.} 
By design, the geometric rectification in VRD-R utilizes a centripetal gravitational field, as Eq.(\ref{eq:vrdr_reshaping}), that maps all high-value behaviors towards the origin of the latent space ($\|z\| \approx 0$). While this strictly isolates optimal modes from sub-optimal noise, it may lead to an unintended \textit{diversity collapse} if the underlying optimal behavior itself is inherently multi-modal (i.e., there are multiple distinct but equally optimal strategies). By forcing all optimal trajectories into a single highly compact hypersphere, VRD-R might destroy the topological diversity of expert modes. VRD-Q avoids this by employing an \textit{in-place} statistical filtering, which explains why VRD-Q generally exhibits broader stability across diverse environments. Future iterations of VRD-R could explore multi-centric gravitational fields (e.g., using Gaussian Mixture priors) to preserve expert multi-modality.

\newpage
\section*{NeurIPS Paper Checklist}

\begin{enumerate}

\item {\bf Claims}
    \item[] Question: Do the main claims made in the abstract and introduction accurately reflect the paper's contributions and scope?
    \item[] Answer: \answerYes{} 
    \item[] Justification: The "abstract", "introduction" and "preliminaries and motivation" sections clearly state the problem of latent fusion, and the methodologies to solve this effectively. Theoretical and experimental evidences are provided to support our claims.
    \item[] Guidelines:
    \begin{itemize}
        \item The answer \answerNA{} means that the abstract and introduction do not include the claims made in the paper.
        \item The abstract and/or introduction should clearly state the claims made, including the contributions made in the paper and important assumptions and limitations. A \answerNo{} or \answerNA{} answer to this question will not be perceived well by the reviewers. 
        \item The claims made should match theoretical and experimental results, and reflect how much the results can be expected to generalize to other settings. 
        \item It is fine to include aspirational goals as motivation as long as it is clear that these goals are not attained by the paper. 
    \end{itemize}

\item {\bf Limitations}
    \item[] Question: Does the paper discuss the limitations of the work performed by the authors?
    \item[] Answer: \answerYes{} 
    \item[] Justification: The paper includes a dedicated "Limitations" paragraph in the Conclusion section and Appendix.
    \item[] Guidelines:
    \begin{itemize}
        \item The answer \answerNA{} means that the paper has no limitation while the answer \answerNo{} means that the paper has limitations, but those are not discussed in the paper. 
        \item The authors are encouraged to create a separate ``Limitations'' section in their paper.
        \item The paper should point out any strong assumptions and how robust the results are to violations of these assumptions (e.g., independence assumptions, noiseless settings, model well-specification, asymptotic approximations only holding locally). The authors should reflect on how these assumptions might be violated in practice and what the implications would be.
        \item The authors should reflect on the scope of the claims made, e.g., if the approach was only tested on a few datasets or with a few runs. In general, empirical results often depend on implicit assumptions, which should be articulated.
        \item The authors should reflect on the factors that influence the performance of the approach. For example, a facial recognition algorithm may perform poorly when image resolution is low or images are taken in low lighting. Or a speech-to-text system might not be used reliably to provide closed captions for online lectures because it fails to handle technical jargon.
        \item The authors should discuss the computational efficiency of the proposed algorithms and how they scale with dataset size.
        \item If applicable, the authors should discuss possible limitations of their approach to address problems of privacy and fairness.
        \item While the authors might fear that complete honesty about limitations might be used by reviewers as grounds for rejection, a worse outcome might be that reviewers discover limitations that aren't acknowledged in the paper. The authors should use their best judgment and recognize that individual actions in favor of transparency play an important role in developing norms that preserve the integrity of the community. Reviewers will be specifically instructed to not penalize honesty concerning limitations.
    \end{itemize}

\item {\bf Theory assumptions and proofs}
    \item[] Question: For each theoretical result, does the paper provide the full set of assumptions and a complete (and correct) proof?
    \item[] Answer: \answerYes{} 
    \item[] Justification: The main paper explicitly states the required assumptions.
    \item[] Guidelines:
    \begin{itemize}
        \item The answer \answerNA{} means that the paper does not include theoretical results. 
        \item All the theorems, formulas, and proofs in the paper should be numbered and cross-referenced.
        \item All assumptions should be clearly stated or referenced in the statement of any theorems.
        \item The proofs can either appear in the main paper or the supplemental material, but if they appear in the supplemental material, the authors are encouraged to provide a short proof sketch to provide intuition. 
        \item Inversely, any informal proof provided in the core of the paper should be complemented by formal proofs provided in appendix or supplemental material.
        \item Theorems and Lemmas that the proof relies upon should be properly referenced. 
    \end{itemize}

    \item {\bf Experimental result reproducibility}
    \item[] Question: Does the paper fully disclose all the information needed to reproduce the main experimental results of the paper to the extent that it affects the main claims and/or conclusions of the paper (regardless of whether the code and data are provided or not)?
    \item[] Answer: \answerYes{} 
    \item[] Justification: Detailed experimental settings are referenced in the main text. These are extensively expanded upon in Appendices C, providing the necessary hyperparameters, optimizers, and architectural configurations.
    \item[] Guidelines:
    \begin{itemize}
        \item The answer \answerNA{} means that the paper does not include experiments.
        \item If the paper includes experiments, a \answerNo{} answer to this question will not be perceived well by the reviewers: Making the paper reproducible is important, regardless of whether the code and data are provided or not.
        \item If the contribution is a dataset and\slash or model, the authors should describe the steps taken to make their results reproducible or verifiable. 
        \item Depending on the contribution, reproducibility can be accomplished in various ways. For example, if the contribution is a novel architecture, describing the architecture fully might suffice, or if the contribution is a specific model and empirical evaluation, it may be necessary to either make it possible for others to replicate the model with the same dataset, or provide access to the model. In general. releasing code and data is often one good way to accomplish this, but reproducibility can also be provided via detailed instructions for how to replicate the results, access to a hosted model (e.g., in the case of a large language model), releasing of a model checkpoint, or other means that are appropriate to the research performed.
        \item While NeurIPS does not require releasing code, the conference does require all submissions to provide some reasonable avenue for reproducibility, which may depend on the nature of the contribution. For example
        \begin{enumerate}
            \item If the contribution is primarily a new algorithm, the paper should make it clear how to reproduce that algorithm.
            \item If the contribution is primarily a new model architecture, the paper should describe the architecture clearly and fully.
            \item If the contribution is a new model (e.g., a large language model), then there should either be a way to access this model for reproducing the results or a way to reproduce the model (e.g., with an open-source dataset or instructions for how to construct the dataset).
            \item We recognize that reproducibility may be tricky in some cases, in which case authors are welcome to describe the particular way they provide for reproducibility. In the case of closed-source models, it may be that access to the model is limited in some way (e.g., to registered users), but it should be possible for other researchers to have some path to reproducing or verifying the results.
        \end{enumerate}
    \end{itemize}

\item {\bf Open access to data and code}
    \item[] Question: Does the paper provide open access to the data and code, with sufficient instructions to faithfully reproduce the main experimental results, as described in supplemental material?
    \item[] Answer: \answerYes{} 
    \item[] Justification: The code were supplied to reviewers as an anonymous zip archive, and full
public links will be added in the camera-ready version.
    \item[] Guidelines:
    \begin{itemize}
        \item The answer \answerNA{} means that paper does not include experiments requiring code.
        \item Please see the NeurIPS code and data submission guidelines (\url{https://neurips.cc/public/guides/CodeSubmissionPolicy}) for more details.
        \item While we encourage the release of code and data, we understand that this might not be possible, so \answerNo{} is an acceptable answer. Papers cannot be rejected simply for not including code, unless this is central to the contribution (e.g., for a new open-source benchmark).
        \item The instructions should contain the exact command and environment needed to run to reproduce the results. See the NeurIPS code and data submission guidelines (\url{https://neurips.cc/public/guides/CodeSubmissionPolicy}) for more details.
        \item The authors should provide instructions on data access and preparation, including how to access the raw data, preprocessed data, intermediate data, and generated data, etc.
        \item The authors should provide scripts to reproduce all experimental results for the new proposed method and baselines. If only a subset of experiments are reproducible, they should state which ones are omitted from the script and why.
        \item At submission time, to preserve anonymity, the authors should release anonymized versions (if applicable).
        \item Providing as much information as possible in supplemental material (appended to the paper) is recommended, but including URLs to data and code is permitted.
    \end{itemize}

\item {\bf Experimental setting/details}
    \item[] Question: Does the paper specify all the training and test details (e.g., data splits, hyperparameters, how they were chosen, type of optimizer) necessary to understand the results?
    \item[] Answer: \answerYes{} 
    \item[] Justification: The paper specifies baseline architectures, optimization protocols, hardware details, and hyperparameters comprehensively in Experiment section and Appendix C.
    \item[] Guidelines:
    \begin{itemize}
        \item The answer \answerNA{} means that the paper does not include experiments.
        \item The experimental setting should be presented in the core of the paper to a level of detail that is necessary to appreciate the results and make sense of them.
        \item The full details can be provided either with the code, in appendix, or as supplemental material.
    \end{itemize}

\item {\bf Experiment statistical significance}
    \item[] Question: Does the paper report error bars suitably and correctly defined or other appropriate information about the statistical significance of the experiments?
    \item[] Answer: \answerYes{} 
    \item[] Justification: Our paper reports appropriate information about statistical significance of the experiments.
    \item[] Guidelines:
    \begin{itemize}
        \item The answer \answerNA{} means that the paper does not include experiments.
        \item The authors should answer \answerYes{} if the results are accompanied by error bars, confidence intervals, or statistical significance tests, at least for the experiments that support the main claims of the paper.
        \item The factors of variability that the error bars are capturing should be clearly stated (for example, train/test split, initialization, random drawing of some parameter, or overall run with given experimental conditions).
        \item The method for calculating the error bars should be explained (closed form formula, call to a library function, bootstrap, etc.)
        \item The assumptions made should be given (e.g., Normally distributed errors).
        \item It should be clear whether the error bar is the standard deviation or the standard error of the mean.
        \item It is OK to report 1-sigma error bars, but one should state it. The authors should preferably report a 2-sigma error bar than state that they have a 96\% CI, if the hypothesis of Normality of errors is not verified.
        \item For asymmetric distributions, the authors should be careful not to show in tables or figures symmetric error bars that would yield results that are out of range (e.g., negative error rates).
        \item If error bars are reported in tables or plots, the authors should explain in the text how they were calculated and reference the corresponding figures or tables in the text.
    \end{itemize}

\item {\bf Experiments compute resources}
    \item[] Question: For each experiment, does the paper provide sufficient information on the computer resources (type of compute workers, memory, time of execution) needed to reproduce the experiments?
    \item[] Answer: \answerYes{} 
    \item[] Justification: Appendix C details the specific compute hardware used (NVIDIA RTX-3090 GPUs).
    \item[] Guidelines:
    \begin{itemize}
        \item The answer \answerNA{} means that the paper does not include experiments.
        \item The paper should indicate the type of compute workers CPU or GPU, internal cluster, or cloud provider, including relevant memory and storage.
        \item The paper should provide the amount of compute required for each of the individual experimental runs as well as estimate the total compute. 
        \item The paper should disclose whether the full research project required more compute than the experiments reported in the paper (e.g., preliminary or failed experiments that didn't make it into the paper). 
    \end{itemize}
    
\item {\bf Code of ethics}
    \item[] Question: Does the research conducted in the paper conform, in every respect, with the NeurIPS Code of Ethics \url{https://neurips.cc/public/EthicsGuidelines}?
    \item[] Answer: \answerYes{} 
    \item[] Justification: The research involves fundamental algorithmic developments evaluated on standard, publicly available datasets (e.g., D4RL, Minari, OGBench) and does not violate any ethical guidelines.
    \item[] Guidelines:
    \begin{itemize}
        \item The answer \answerNA{} means that the authors have not reviewed the NeurIPS Code of Ethics.
        \item If the authors answer \answerNo, they should explain the special circumstances that require a deviation from the Code of Ethics.
        \item The authors should make sure to preserve anonymity (e.g., if there is a special consideration due to laws or regulations in their jurisdiction).
    \end{itemize}

\item {\bf Broader impacts}
    \item[] Question: Does the paper discuss both potential positive societal impacts and negative societal impacts of the work performed?
    \item[] Answer: \answerNA{} 
    \item[] Justification: This paper focuses on foundational theoretical improvements to flow-based offline RL algorithms. It does not present direct, foreseeable negative societal impacts or obvious ties to malicious applications.
    \item[] Guidelines:
    \begin{itemize}
        \item The answer \answerNA{} means that there is no societal impact of the work performed.
        \item If the authors answer \answerNA{} or \answerNo, they should explain why their work has no societal impact or why the paper does not address societal impact.
        \item Examples of negative societal impacts include potential malicious or unintended uses (e.g., disinformation, generating fake profiles, surveillance), fairness considerations (e.g., deployment of technologies that could make decisions that unfairly impact specific groups), privacy considerations, and security considerations.
        \item The conference expects that many papers will be foundational research and not tied to particular applications, let alone deployments. However, if there is a direct path to any negative applications, the authors should point it out. For example, it is legitimate to point out that an improvement in the quality of generative models could be used to generate Deepfakes for disinformation. On the other hand, it is not needed to point out that a generic algorithm for optimizing neural networks could enable people to train models that generate Deepfakes faster.
        \item The authors should consider possible harms that could arise when the technology is being used as intended and functioning correctly, harms that could arise when the technology is being used as intended but gives incorrect results, and harms following from (intentional or unintentional) misuse of the technology.
        \item If there are negative societal impacts, the authors could also discuss possible mitigation strategies (e.g., gated release of models, providing defenses in addition to attacks, mechanisms for monitoring misuse, mechanisms to monitor how a system learns from feedback over time, improving the efficiency and accessibility of ML).
    \end{itemize}
    
\item {\bf Safeguards}
    \item[] Question: Does the paper describe safeguards that have been put in place for responsible release of data or models that have a high risk for misuse (e.g., pre-trained language models, image generators, or scraped datasets)?
    \item[] Answer: \answerNA{} 
    \item[] Justification: The models and evaluation benchmarks used do not fall under the category of high-risk for dual-use or misuse. No scraped or sensitive data is released.
    \item[] Guidelines:
    \begin{itemize}
        \item The answer \answerNA{} means that the paper poses no such risks.
        \item Released models that have a high risk for misuse or dual-use should be released with necessary safeguards to allow for controlled use of the model, for example by requiring that users adhere to usage guidelines or restrictions to access the model or implementing safety filters. 
        \item Datasets that have been scraped from the Internet could pose safety risks. The authors should describe how they avoided releasing unsafe images.
        \item We recognize that providing effective safeguards is challenging, and many papers do not require this, but we encourage authors to take this into account and make a best faith effort.
    \end{itemize}

\item {\bf Licenses for existing assets}
    \item[] Question: Are the creators or original owners of assets (e.g., code, data, models), used in the paper, properly credited and are the license and terms of use explicitly mentioned and properly respected?
    \item[] Answer: \answerYes{} 
    \item[] Justification: Existing datasets, platforms, and baseline models are properly credited with citations in the text and the references section.
    \item[] Guidelines:
    \begin{itemize}
        \item The answer \answerNA{} means that the paper does not use existing assets.
        \item The authors should cite the original paper that produced the code package or dataset.
        \item The authors should state which version of the asset is used and, if possible, include a URL.
        \item The name of the license (e.g., CC-BY 4.0) should be included for each asset.
        \item For scraped data from a particular source (e.g., website), the copyright and terms of service of that source should be provided.
        \item If assets are released, the license, copyright information, and terms of use in the package should be provided. For popular datasets, \url{paperswithcode.com/datasets} has curated licenses for some datasets. Their licensing guide can help determine the license of a dataset.
        \item For existing datasets that are re-packaged, both the original license and the license of the derived asset (if it has changed) should be provided.
        \item If this information is not available online, the authors are encouraged to reach out to the asset's creators.
    \end{itemize}

\item {\bf New assets}
    \item[] Question: Are new assets introduced in the paper well documented and is the documentation provided alongside the assets?
    \item[] Answer: \answerNA{} 
    \item[] Justification: The paper does not introduce or release new datasets or proprietary assets.
    \item[] Guidelines:
    \begin{itemize}
        \item The answer \answerNA{} means that the paper does not release new assets.
        \item Researchers should communicate the details of the dataset\slash code\slash model as part of their submissions via structured templates. This includes details about training, license, limitations, etc. 
        \item The paper should discuss whether and how consent was obtained from people whose asset is used.
        \item At submission time, remember to anonymize your assets (if applicable). You can either create an anonymized URL or include an anonymized zip file.
    \end{itemize}

\item {\bf Crowdsourcing and research with human subjects}
    \item[] Question: For crowdsourcing experiments and research with human subjects, does the paper include the full text of instructions given to participants and screenshots, if applicable, as well as details about compensation (if any)? 
    \item[] Answer: \answerNA{} 
    \item[] Justification: This research does not involve any human subjects or crowdsourcing.
    \item[] Guidelines:
    \begin{itemize}
        \item The answer \answerNA{} means that the paper does not involve crowdsourcing nor research with human subjects.
        \item Including this information in the supplemental material is fine, but if the main contribution of the paper involves human subjects, then as much detail as possible should be included in the main paper. 
        \item According to the NeurIPS Code of Ethics, workers involved in data collection, curation, or other labor should be paid at least the minimum wage in the country of the data collector. 
    \end{itemize}

\item {\bf Institutional review board (IRB) approvals or equivalent for research with human subjects}
    \item[] Question: Does the paper describe potential risks incurred by study participants, whether such risks were disclosed to the subjects, and whether Institutional Review Board (IRB) approvals (or an equivalent approval/review based on the requirements of your country or institution) were obtained?
    \item[] Answer: \answerNA{} 
    \item[] Justification: There is no research involving human subjects, making IRB approval unnecessary.
    \item[] Guidelines:
    \begin{itemize}
        \item The answer \answerNA{} means that the paper does not involve crowdsourcing nor research with human subjects.
        \item Depending on the country in which research is conducted, IRB approval (or equivalent) may be required for any human subjects research. If you obtained IRB approval, you should clearly state this in the paper. 
        \item We recognize that the procedures for this may vary significantly between institutions and locations, and we expect authors to adhere to the NeurIPS Code of Ethics and the guidelines for their institution. 
        \item For initial submissions, do not include any information that would break anonymity (if applicable), such as the institution conducting the review.
    \end{itemize}

\item {\bf Declaration of LLM usage}
    \item[] Question: Does the paper describe the usage of LLMs if it is an important, original, or non-standard component of the core methods in this research? Note that if the LLM is used only for writing, editing, or formatting purposes and does \emph{not} impact the core methodology, scientific rigor, or originality of the research, declaration is not required.
    \item[] Answer: \answerNA{} 
    \item[] Justification: Large language models are not part of the paper’s core methods.
    \item[] Guidelines:
    \begin{itemize}
        \item The answer \answerNA{} means that the core method development in this research does not involve LLMs as any important, original, or non-standard components.
        \item Please refer to our LLM policy in the NeurIPS handbook for what should or should not be described.
    \end{itemize}

\end{enumerate}

\end{document}